\newif\ifdraft \draftfalse
\newif\iffull \fulltrue

\let\hide\iffalse

\draftfalse\fulltrue

\documentclass[11pt]{article}
\usepackage{amsmath}

\newcommand\co{\textbf{o}}

\newcommand\RR{\mathbb{R}}

\newcommand\cB{\mathcal{B}}
\newcommand\cE{\mathcal{E}}
\newcommand\cL{\mathcal{L}}
\newcommand\cF{\mathcal{F}}
\newcommand\cH{\mathcal{H}}

\newcommand\cM{\mathcal{M}}
\newcommand\cO{\mathcal{O}}
\newcommand\cA{\mathcal{A}}
\newcommand\cP{\mathcal{P}}
\newcommand\cps{\cA^{prsm}_\cO(\mathcal{S})}
\newcommand\cpsp{\cA^{prsm}_\cO(\mathcal{S'})}
\newcommand\cQ{\mathcal{Q}}
\newcommand\cQd{\mathcal{Q}_{\mathrm{dual}}}

\newcommand\cX{\mathcal{X}}
\newcommand\cD{\mathcal{D}}

\newcommand\cN{\mathcal{N}}

\newcommand\cV{\mathcal{V}}

\newcommand\E{\mathbb{E}}

\newcommand\CFTPL{\text{Context-FTPL }}

\newcommand{\Lap}[1]{\text{Lap}(#1)}

\newcommand\Reg{\mathrm{Reg}}

\newcommand\OracleQuery{\mathrm{OracleQuery}}

\usepackage{xspace}
\usepackage{cleveref}
\usepackage{thmtools, thm-restate}
\usepackage{setspace}
\usepackage[square,sort,comma,numbers]{natbib}

\DeclareMathOperator{\polylog}{polylog}

\DeclareMathOperator*{\Expectation}{\mathbb{E}}
\newcommand{\Ex}[2]{\Expectation_{#1}\left[#2\right]}

\newcommand{\prob}[1]{\Pr\left[#1\right]}

\newcommand{\eps}{\varepsilon}
\def\epsilon{\varepsilon}

\DeclareMathOperator*{\argmin}{\mathrm{argmin}}
\DeclareMathOperator*{\argmax}{\mathrm{argmax}}
\newcommand{\INDSTATE}[1][1]{\STATE\hspace{#1\algorithmicindent}}
\usepackage{subfigure}
\usepackage[algo2e]{algorithm2e}
\usepackage{fullpage}
\usepackage{amsmath, amssymb, amsthm}
\usepackage{bbm}
\usepackage{mathtools}
\usepackage{color}
\usepackage{kpfonts}
\usepackage{algorithm,algorithmic}
\usepackage{xcolor}
\usepackage{thmtools}
\usepackage{thm-restate}
\theoremstyle{plain}

\newcommand{\sn}[1]{\textcolor{purple}{[Seth: #1]}}
\newcommand{\ar}[1]{\ifdraft \textcolor{red}{[Aaron: #1]}\fi}

\newtheorem{definition}{Definition}
\newtheorem{corollary}{Corollary}
\newtheorem{lemma}{Lemma}
\newtheorem{theorem}{Theorem}
\newtheorem{remark}{Remark}
\newtheorem{claim}{Claim}
\newtheorem{fact}{Fact}

\title{How to Use Heuristics for Differential Privacy}
\author{Seth Neel\thanks{Wharton Statistics Department, University of Pennsylvania. Email: \texttt{sethneel@wharton.upenn.edu}. Supported in part by an NSF Graduate Research Fellowship} \and Aaron Roth\thanks{Department of Computer and Information Sciences, University of Pennsylvania. Email: \texttt{aaroth@cis.upenn.edu}. Supported in part by the Sloan foundation, the DARPA Brandeis project, and NSF awards 1253345 and 1513694.} \and Zhiwei Steven Wu\thanks{Computer Science and Engineering Department, University of Minnesota. Email: \texttt{zsw@umn.edu}}}
\begin{document}
\maketitle

\begin{abstract}
We develop theory for using \emph{heuristics} to solve computationally hard problems in differential privacy. Heuristic approaches have enjoyed tremendous success in machine learning, for which performance can be empirically evaluated. However, privacy guarantees cannot be evaluated empirically, and must be proven --- without making heuristic assumptions. We show that learning problems over broad classes of functions --- those that have polynomially sized universal identification sets --- can be solved privately and efficiently, assuming the existence of a non-private oracle for solving the same problem. Our first algorithm yields a privacy guarantee that is contingent on the correctness of the oracle. We then give a reduction which applies to a class of heuristics which we call \emph{certifiable}, which allows us to convert oracle-dependent privacy guarantees to worst-case privacy guarantee that hold even when the heuristic standing in for the oracle might fail in adversarial ways. Finally, we consider classes of functions for which both they and their dual classes have small universal identification sets. This includes most classes of simple boolean functions studied in the PAC learning literature, including conjunctions, disjunctions, parities, and discrete halfspaces. We show that there is an efficient algorithm for privately constructing synthetic data for any such class, given a non-private learning oracle. This in particular gives the first oracle-efficient algorithm for privately generating synthetic data for contingency tables. The most intriguing question left open by our work is whether or not \emph{every problem} that can be solved differentially privately can be privately solved with an oracle-efficient algorithm. While we do not resolve this, we give a barrier result that suggests that any generic oracle-efficient reduction must fall outside of a natural class of algorithms (which includes the algorithms given in this paper).

\end{abstract}
%%% Local Variables:
%%% mode: latex
%%% TeX-master: "heuristicDP"
%%% End:

\thispagestyle{empty} \setcounter{page}{0}
\clearpage
%\doublespacing
\tableofcontents
\singlespacing
\thispagestyle{empty} \setcounter{page}{0}
\clearpage

\section{Introduction}
Differential privacy is compatible with a tremendous number of powerful data analysis tasks, including essentially any statistical learning problem \cite{KLNRS08,CMS11,BST14} and the generation of synthetic data consistent with exponentially large families of statistics \cite{BLR08,RR10,PMW,GRU12,NTZ13}. Unfortunately, it is also beset with a comprehensive set of computational hardness results. Of course, it inherits all of the computational hardness results from the (non-private) agnostic learning literature: for example, even the simplest learning tasks --- like finding the best conjunction or linear separator to approximately minimize classification error --- are hard \cite{FGKP09,FGRW12,DRSW11}. In addition, tasks that are easy absent privacy constraints can become hard when these constraints are added. For example, although information theoretically, it is possible to privately construct synthetic data consistent with all $d$-way marginals for $d$-dimensional data, privately constructing synthetic data for even $2$-way marginals is computationally hard \cite{hardsynth}.  These hardness results extend even to providing numeric answers to more than quadratically many statistical queries \cite{Ull16}.
%And of course, differentially private learning is only harder than non-private learning, for which computational hardness plagues even the simplest of learning problems, like learning the best conjunction or linear separator to approximately minimize classification error \cite{FGRW12,DRSW11}.

How should we proceed in the face of pervasive computational hardness? We might take inspiration from machine learning, which has not been slowed, despite the fact that its most basic problems (e.g. learning linear separators) are already hard even to approximate. Instead, the field has employed heuristics with tremendous success --- including exact optimization of convex surrogate loss functions (as in the case of SVMs), decision tree heuristics, gradient based methods for differentiable but non-convex problems (as in back-propogation for training neural networks), and integer programming solvers (as in recent work on interpretable machine learning \cite{BR16}). Other fields such as operations research similarly have developed sophisticated heuristics including integer program solvers and SAT solvers that are able to routinely solve problems that are hard in the worst case.

The case of private data analysis is different, however. If we are only concerned with performance (as is the case for most machine learning and combinatorial optimization tasks), we have the freedom to try different heuristics, and evaluate our algorithms in practice. Thus the design of heuristics that perform well in practice can be undertaken as an empirical science. In contrast, differential privacy is an inherently worst-case guarantee that cannot be evaluated empirically  (see \cite{GM18} for lower bounds for black-box testing of privacy definitions).

In this paper, we build a theory for how to employ \emph{non-private} heuristics (of which there are many, benefitting from many years of intense optimization) to solve computationally hard problems in differential privacy. Our goal is to guide the design of practical algorithms about which we can still prove theorems:

\begin{enumerate}
\item We will aim to prove accuracy theorems \emph{under the assumption that our heuristics solve some non-private problem optimally}. We are happy to make this assumption when proving our accuracy theorems, because accuracy is something that can be empirically evaluated on the datasets that we are interested in. An assumption like this is also necessary, because we are designing algorithms for problems that are computationally hard in the worst case. However:
\item We aim to prove that our algorithms are differentially private in the worst case, even under the assumption that our heuristics might fail in an adversarial manner.
\end{enumerate}

\subsection{Overview of Our Results}
Informally, we give a collection of results showing the existence of \emph{oracle-efficient} algorithms for privately solving learning and synthetic data generation problems defined by discrete classes of functions $\cQ$ that have a special (but common) combinatorial structure. One might initially ask whether it is possible to give a direct reduction from a non-private but efficient algorithm for solving a learning problem to an efficient private algorithm for solving the same learning problem \emph{without requiring any special structure at all}. However, this is impossible, because there are classes of functions (namely those that have finite VC-dimension but infinite Littlestone dimension) that are known to be learnable absent the constraint of privacy, but are not privately learnable in an information-theoretic sense \cite{intervals, littlestoneprivacy}. The main question we leave open is whether \emph{being information theoretically learnable under the constraint of differential privacy} is sufficient for oracle-efficient private learning. We give a barrier result suggesting that it might not be.

Before we summarize our results in more detail, we give some informal definitions.

\subsubsection{Definitions}
We begin by defining the kinds of \emph{oracles} that we will work with, and end-goals that we will aim for. We will assume the existence of oracles for (non-privately) solving learning problems: for example, an oracle which can solve the empirical risk minimization problem for discrete linear threshold functions. Because ultimately oracles will be implemented using heuristics, we consider two types of oracles:
\begin{enumerate}
\item \emph{Certifiable} heuristic oracles might fail, but when they succeed, they come with a certificate of success. Many heuristics for solving integer programs are certifiable, including cutting planes methods and branch and bound methods. SAT Solvers (and any other heuristic for solving a decision problem in NP) are also certifiable.
\item On the other had, some heuristics are \emph{non-certifiable}. These heuristics might produce incorrect answers, without any indication that they have failed. Support vector machines and logistic regression are examples of non-certifiable heuristic oracles for learning linear threshold functions.
%\item Finally, we might assume we have access to a differentially private heuristic oracle. These oracles are assumed only to approximately solve the underlying learning problem, but are assumed to do so in a differentially private manner. For example, private convex optimization methods optimizing convex surrogates for classification accuracy can be viewed as differentially private heuristics for solving computationally hard learning problems. In addition to assuming the existence of such oracles, we will aim to derive them starting from the assumption that we have a non-private but certifiable heuristic oracle.
\end{enumerate}
%We will also discuss differentially private heuristic optimization oracles, in order to state all of the consequences of our results in Section \ref{sec:synth}.
We define an oracle-efficient \textit{non-robustly} differentially private algorithm to be an algorithm that runs in polynomial time in all relevant parameters given access to an oracle for some problem, and has an accuracy guarantee and a differential privacy guarantee which may both be \emph{contingent} on the guarantees of the oracle --- i.e. if the oracle is replaced with a heuristic, the algorithm may no longer be differentially private. Although in certain situations (e.g when we have very high confidence that our heuristics actually do succeed on all instances we will ever encounter) it might be acceptable to have a privacy guarantee that is contingent on having an infallible oracle, we would much prefer a privacy guarantee that held in the worst case. We say that an oracle-efficient algorithm is \emph{robustly} differentially private if its privacy guarantee is not contingent on the behavior of the oracle, and holds in the worst case, even if an adversary is in control of the heuristic that stands in for our oracle.

\subsubsection{Learning and Optimization}
Our first result is a reduction from efficient non-private learning to efficient private learning over any class of functions $\cQ$ that has a small universal identification set \cite{goldman1993exact}. A universal identification set of size $m$ is a set of $m$ examples such that the labelling of these examples by a function $q \in \cQ$ is enough to uniquely identify $q$. Equivalently, a universal identification set can be viewed as a \emph{separator set} \cite{oracle16}: for any pair of functions $q \neq q' \in \cQ$, there must be some example $x$ in the universal identification set such that $q(x) \neq q(x')$. We will use these terms interchangeably throughout the paper. We show that if $\cQ$ has a universal identification set of size $m$, then given an oracle which solves the empirical risk minimization problem (non-privately) over $\cQ$, there is an $\epsilon$-differentially private algorithm with additional running time scaling linearly with $m$ and error scaling linearly with $m^2/\epsilon$ that solves the private empirical risk minimization problem over $\cQ$. The error can be improved to $O(m^{1.5}\sqrt{\log 1/\delta}/\epsilon)$, while satisfying $(\epsilon,\delta)$-differential privacy.  Many well studied discrete concept classes $\cQ$ from the PAC learning literature have small universal identification sets. For example, in $d$ dimensions, boolean conjunctions, disjunctions, parities, and halfspaces defined over the hypercube have universal identification sets of size $d$.  This means that for these classes, our oracle-efficient algorithm has error that is larger than the generic optimal (and computationally inefficient) learner from \cite{KLNRS08} by a factor of $O(\sqrt{d})$. Other classes of functions also have small universal identification sets --- for example, decision lists have universal identification sets of size $d^2$.
%This means for these classes, our oracle-efficient algorithm matches the optimal performance of the (computationally inefficient) exponential-mechanism based learner from \cite{KNRS08}.

The reduction described above has the disadvantage that not only its accuracy guarantees --- but also its proof of privacy --- depend on the oracle correctly solving the empirical risk minimization problem it is given; it is \textit{non-robustly} differentially private. This shortcoming motivates our main technical result: a generic reduction that takes as input any oracle-efficient non-robustly differentially private algorithm (i.e. an algorithm whose privacy proof might depend on the proper functioning of the oracle) and produces an oracle-efficient \emph{robustly} differentially private algorithm, \emph{whenever the oracle is implemented with a certifiable heuristic}. As discussed above, this class of heuristics includes the integer programming algorithms used in most commercial solvers. In combination with our first result, we obtain robustly differentially private oracle-efficient learning algorithms for conjunctions, disjunctions, discrete halfspaces, and any other class of functions with a small universal identification set. %\ar{Write about the algorithm and its analysis here}
\subsubsection{Synthetic Data Generation}
We then proceed to the task of constructing synthetic data consistent with a class of queries $\cQ$. Following \cite{Hsu13,dualquery}, we view the task of synthetic data generation as the process of computing an equilibrium of a particular zero sum game played between a data player and a query player. In order to compute this equilibrium, we need to be able to instantiate two objects in an oracle-efficient manner:
\begin{enumerate}
\item a private \emph{learning} algorithm for $\cQ$ (this corresponds to solving the best response problem for the ``query player''), and
\item  a \emph{no-regret learning algorithm} for a dual class of functions $\cQd$ that results from swapping the role of the data element and the query function (this allows the ``data player'' to obtain a diminishing regret bound in simulated play of the game).
\end{enumerate}

The no-regret learning algorithm need not be differentially private. From our earlier results, we are able to construct an oracle-efficient robustly differentially private learning algorithm for $\cQ$ whenever it has a small universal identification set. On the other hand, Syrgkanis et al. \cite{oracle16} show how to obtain an oracle-efficient no regret learning algorithm for a class of functions under the same condition. Hence, we obtain an oracle-efficient robustly differentially private synthetic data generation algorithm for any class of functions $\cQ$ for which both $\cQ$ and $\cQd$ have small universal identification sets. Fortunately, this is the case for many interesting classes of functions, including boolean disjunctions, conjunctions, discrete halfspaces, and parity functions. The result is that we obtain oracle-efficient algorithms for generating private synthetic data for all of these classes. We note that the oracle used by the data player need not be certifiable.

\subsubsection{A Barrier Result}
Finally, we exhibit a barrier to giving oracle-efficient private learning algorithms for \emph{all} classes of functions $\cQ$ known to be privately learnable. We identify a class of private learning algorithms called \emph{perturbed empirical risk minimizers} (pERMs) which output the query that \emph{exactly} minimizes some perturbation of their empirical risk on the dataset. This class of algorithms includes the ones we give in this paper, as well as many other differentially private learning algorithms, including the exponential mechanism and report-noisy-min. We show that any private pERM can be efficiently used as a no-regret learning algorithm with regret guarantees that depend on the scale of the perturbations it uses. This allows us to reduce to a lower bound on the running time of oracle-efficient online learning algorithms due to Hazan and Koren \cite{HK16}. The result is that there exist finite classes of queries $\cQ$ such that any oracle-efficient differentially private pERM algorithm must introduce perturbations that are polynomially large in the size of $|Q|$, whereas any such class is information-theoretically privately learnable with error that scales only with $\log|\cQ|$.

The barrier implies that \emph{if} oracle-efficient differentially private learning algorithms are as powerful as inefficient differentially private learning algorithms, then these general oracle efficient private algorithms must not be perturbed empirical risk minimizers. We conjecture that the set of problems solvable by oracle-efficient differentially private learners is strictly smaller than the set of problems solvable information theoretically under the constraint of differential privacy, but leave this as our main open question.
\subsection{Additional Related Work}
Conceptually, the most closely related piece of work is the ``DualQuery'' algorithm of \cite{dualquery}, which in the terminology of our paper is a robustly private oracle-efficient algorithm for generating synthetic data for $k$-way marginals for constant $k$. The main idea in \cite{dualquery} is to formulate the private optimization problem that needs to be solved so that the only computationally hard task is one that does not depend on private data. There are other algorithms that can straightforwardly be put into this framework, like the projection algorithm from \cite{NTZ13}. This approach immediately makes the privacy guarantees independent of the correctness of the oracle, but significantly limits the algorithm design space. In particular, the DualQuery algorithm (and the oracle-efficient version of the projection algorithm from \cite{NTZ13}) has running time that is proportional to $|\cQ|$, and so can only handle polynomially sized classes of queries (which is why $k$ needs to be held constant). The main contribution of our paper is to be able to handle private optimization problems in which the hard computational step is \emph{not} independent of the private data. This is significantly more challenging, and is what allows us to give oracle-efficient robustly private algorithms for constructing synthetic data for exponentially large families $\cQ$. It is also what lets give oracle-efficient private \emph{learning} algorithms over exponentially large $\cQ$ for the first time.

A recent line of work starting with the ``PATE'' algorithm \cite{pate} together with more recent theoretical analyses of similar algorithms by Dwork and Feldman, and Bassily, Thakkar, and Thakurta \cite{DF18,BTT18} can be viewed as giving oracle-efficient algorithms for an easier learning task, in which the goal is to produce a finite number of private \emph{predictions} rather than privately output the model that makes the predictions. These can be turned into oracle efficient algorithms for outputting a private model \emph{under the assumption} that the mechanism has access to an additional source of unlabeled data drawn from the same distribution as the private data, but that does not need privacy protections. In this setting, there is no need to take advantage of any special structure of the hypothesis class $\cQ$, because the information theoretic lower bounds on private learning proven in \cite{intervals,littlestoneprivacy} do not apply. In contrast, our results apply without the need for an auxiliary source of non-private data. 

Privately producing \emph{contingency tables}, and synthetic data that encode them --- i.e. the answers to statistical queries defined by conjunctions of features --- has been a key challenge problem in differential privacy at least since \cite{contingency}. Since then, a number of algorithms and hardness results have been given \cite{hardsynth, conjunctions, contingency2,marginals,marginals2,conjunctions2,conjunctions3}. This paper gives the first oracle-efficient algorithm for generating synthetic data consistent with a full contingency table, and the first oracle-efficient algorithm for answering arbitrary conjunctions to near optimal error.

Technically, our work is inspired by Syrgkanis et al. \cite{oracle16} who show how a small separator set (equivalently a small universal identification set) can be used to derive oracle-efficient no-regret algorithms in the contextual bandit setting. The small separator property has found other uses in online learning, including in the oracle-efficient construction of nearly revenue optimal auctions \cite{oracleauction}. Hazan and Koren \cite{HK16} show lower bounds for oracle-efficient no-regret learning algorithms in the experts setting, which forms the basis of our barrier result. More generally, there is a rich literature studying oracle-efficient algorithms in machine learning \cite{red1,red2,red3} and optimization \cite{oracleopt} as a means of dealing with worst-case hardness, and more recently, for machine learning subject to fairness constraints \cite{Agarwal18,KNRW18,AIK18}.

We also make crucial use of a property of differentially private algorithms, first shown by \cite{CLNRW16}: That when differentially private algorithms are run on databases of size $n$ with privacy parameter $\epsilon \approx 1/\sqrt{n}$, then they have similar output distributions when run on datasets that are \emph{sampled from the same distribution}, rather than just on neighboring datasets. In \cite{CLNRW16}, this was used as a tool to show the existence of \emph{robustly generalizing} algorithms (also known as \emph{distributionally private} algorithms in \cite{BLR08}). We prove a new variant of this fact that holds when the datasets are not sampled i.i.d. and use it for the first time in an analysis to prove differential privacy. The technique might be of independent interest.  
\section{Preliminaries}
\subsection{Differential Privacy Tools}
\label{sec:DPprelims}

%Two \emph{data sets} $x, x' \in \cX^*$ are said to be \emph{neighbors} (written as $x \sim x'$) if $x$ can be derived from $x'$ by a single insertion or deletion.
Let $\cX$ denote a $d$-dimensional data domain (e.g. $\mathbb{R}^d$ or
$\{0, 1\}^d$).  We write $n$ to denote the size of a dataset $S$. We
call two \emph{data sets} $S, S' \in \cX^n$ \emph{neighbors} (written
as $S \sim S'$) if $S$ can be derived from $S'$ by replacing a single
data point with some other element of
$\cX$.% \bo{Our proofs use this def of
  % neighbors I believe.}

\begin{definition}[Differential Privacy \cite{DMNS06,delta}]
Fix $\eps,\delta \geq 0$. A randomized algorithm $A:\cX^*\rightarrow \mathcal{O}$ is $(\eps,\delta)$-differentially private if for every pair of neighboring data sets $S \sim S' \in \cX^*$, and for every event $\Omega \subseteq \mathcal{O}$:
$$\Pr[A(S) \in \Omega] \leq \exp(\eps)\Pr[A(S') \in \Omega] + \delta.$$
%We call $\exp(\eps)$ the \emph{privacy risk} factor.
\end{definition}
Differentially private computations enjoy two nice properties:
\begin{theorem}[Post Processing \cite{DMNS06,delta}]
Let $A:\cX^*\rightarrow \mathcal{O}$ be any $(\eps,\delta)$-differentially private algorithm, and let $f:\mathcal{O}\rightarrow \mathcal{O'}$ be any function. Then the algorithm $f \circ A: \cX^*\rightarrow \mathcal{O}'$ is also $(\eps,\delta)$-differentially private.
\end{theorem}
Post-processing implies that, for example, every \emph{decision} process based on the output of a differentially private algorithm is also differentially private.
%For example, if a company chooses its next marketing campaign based on an $\eps$-differentially private data analysis, then the chance of picking any one campaign changes by at most the privacy risk factor when one data point changes.

\begin{theorem}[Basic Composition \cite{DMNS06,delta}]\label{composition}
Let $A_1:\cX^*\rightarrow \mathcal{O}$, $A_2:\mathcal{O}\times \cX^*\rightarrow \mathcal{O}'$ be such that $A_1$ is $(\eps_1,\delta_1)$-differentially private, and $A_2(o,\cdot)$ is $(\eps_2,\delta_2)$-differentially private for every $o \in \mathcal{O}$. Then the algorithm $A:\cX^*\rightarrow \mathcal{O'}$ defined as $A(x) = A_2(A_1(x),x)$ is $(\eps_1+\eps_2,\delta_1+\delta_2)$-differentially private.
\end{theorem}

The Laplace distribution plays a fundamental role in differential privacy.
The Laplace Distribution centered at $0$ with scale $b$ is the
distribution with probability density function
$\Lap{z|b} = \frac{1}{2b}e^{-\frac{|z|}{b}}$.
% \end{definition}
We write $X \sim \Lap{b}$ when $X$ is a random variable drawn from a Laplace distribution with scale
$b$. Let $f\colon \cX^n \rightarrow \RR^k$ be an arbitrary function. The {\em $\ell_1$ sensitivity} of $f$ is defined
to be $\Delta_1(f) = \max_{S\sim S'} \|f(S) - f(S')\|_1$.
The {\em Laplace mechanism} with parameter $\eps$ simply adds noise drawn independently
from $\Lap{\frac{\Delta_1(f)}{\eps}}$ to each coordinate of $f(S)$.

\begin{theorem}[\citep{DMNS06}]
  The Laplace mechanism is $\eps$-differentially private.
\end{theorem}

\subsection{Statistical Queries and Separator Sets}
We study learning (optimization) and synthetic data generation problems for statistical queries defined over a data universe $\cX$. A statistical query over $\cX$ is a function $q:\cX \rightarrow \{0,1\}$. A statistical query can represent, e.g. any binary classification model or the binary loss function that it induces. Given a dataset $S \in \cX^n$, the value of a statistical query $q$ on $S$ is defined to be $q(S) = \frac{1}{n}\displaystyle\sum_{i=1}^n q(S_i)$. In this paper, we will generally think about query classes $\cQ$ that represent standard \emph{hypothesis classes} from learning theory -- like conjunctions, disjunctions, halfspaces, etc.
% But see Appendix \ref{app:separator} for how to transform these into query classes $\cQ$ representing the empirical loss of hypotheses in these classes, without losing the properties we need for our results to hold.

In this paper, we will make crucial use of \emph{universal identification sets} for classes of statistical queries.
%\begin{definition}[\cite{goldman1993exact}]
%A set of examples $(x_1,\ldots,x_m) \subset \cX$ is a universal identification set for a class of statistical queries $\cQ$ if for every pair of queries $q \neq q' \in \cQ$:
%$$(q(x_1),\ldots,q(x_m)) \neq (q'(x_1),\ldots, q'(x_m))$$
%\end{definition}
Universal identification sets are equivalent to \emph{separator sets}, defined (in a slightly more general form) in \cite{oracle16}.

\begin{definition}[\cite{goldman1993exact,oracle16}]\label{def:separator}
A set $U \subseteq \cX$ is a \emph{universal identification set} or \emph{separator set} for a class of statistical queries $\cQ$ if for every pair of distinct queries $q, q' \in \cQ$, there is an $x \in U$ such that:
$$q(x) \neq q(x')$$
If $|U| = m$, then we say that $\cQ$ has a separator set of size $m$.
\end{definition}

Many classes of statistical queries defined over the boolean hypercube have separator sets of size proportional to their VC-dimension. For example, boolean conjunctions, disjunctions, halfspaces defined over the hypercube, and parity functions in $d$ dimensions all have separator sets of size $d$. When we solve learning problems over these classes, we will be interested in the set of queries that define the 0/1 loss function over these classes: but as we observe in Appendix \ref{app:separator}, if a hypothesis class has a separator set of size $m$, then so does the class of queries representing the empirical loss for functions in that hypothesis class.

\subsection{Learning and Synthetic Data Generation}
We study private learning as empirical risk minimization (the connection between in-sample risk and out-of-sample risk is standard, and follows from e.g. VC-dimension bounds \cite{KV} or directly from differential privacy (see e.g. \cite{BST14,DFHPRR15})). Such problems can be cast as finding a function $q$ in a class $\cQ$ that minimizes $q(S)$, subject to differential privacy (observe that the empirical risk of a hypothesis is a statistical query --- see Appendix \ref{app:separator}). We will therefore study minimization problems over classes of statistical queries generally:
\begin{definition}
We say that a randomized algorithm $M:\cX^n\rightarrow \cQ$ is an $(\alpha,\beta)$-minimizer for $\cQ$ if for every dataset $S \in \cX^n$, with probability $1-\beta$, it outputs $M(S) = q$ such that:
$$q(S) \leq \arg\min_{q^* \in \cQ} q^*(S) + \alpha$$
\end{definition}

Synthetic data generation, on the other hand, is the problem of constructing a \emph{new} dataset $\hat S$ that approximately agrees with the original dataset with respect to a fixed set of statistical queries:

\begin{definition}
We say that a randomized algorithm $M:\cX^n\rightarrow \cX^*$ is an $(\alpha,\beta)$-accurate synthetic data generation algorithm for $\cQ$ if for every dataset $S \in \cX^n$, with probability $1-\beta$, it outputs $M(S) = \hat S$ such that for all $q \in \cQ$:
$$|q(S) - q(\hat S)| \leq \alpha$$
\end{definition}

\subsection{Oracles and Oracle Efficient Algorithms}
We discuss several kinds of oracle-efficient algorithms in this paper. It will be useful for us to study oracles that solve weighted generalizations of the minimization problem, in which each datapoint $x_i \in S$ is paired with a real-valued weight $w_i$. In the literature on oracle-efficiency in machine learning, these are widely employed, and are known as \emph{cost-sensitive classification oracles}. Via a simple translation and re-weighting argument, they are no more powerful than unweighted minimization oracles, but are more convenient to work with.
\begin{definition}
A weighted optimization oracle for a class of statistical queries $\cQ$ is a function $\cO^*:(\cX \times \mathbb{R})^*\rightarrow \cQ$ takes as input a weighted dataset $WD \in (\cX\times \mathbb{R})^*$ and outputs a query $q = \cO^*(WD)$ such that
$$q \in \argmin_{q^* \in \cQ} \displaystyle\sum_{(x_i,w_i) \in WD}w_i q^*(x_i).$$
\end{definition}

In this paper, we will study algorithms that have access to weighted optimization oracles for learning problems that are computationally hard. Since we do not believe that such oracles have worst-case polynomial time implementations, in practice, we will instantiate such oracles with heuristics that are not guaranteed to succeed. There are two failure modes for a heuristic: it can fail to produce an output at all, or it can output an incorrect query. The distinction can be important. We call a heuristic that might fail to produce an output, but never outputs an incorrect solution a certifiable heuristic optimization oracle:

\begin{definition}
A certifiable heuristic optimization oracle for a class of queries $\cQ$ is a polynomial time algorithm $\cO:(\cX \times \mathbb{R})^*\rightarrow (\cQ\cup \bot)$ that takes as input a weighted dataset $WD \in (\cX\times \mathbb{R})^*$ and either outputs $\cO(WD) = q \in \argmin_{q^* \in \cQ} \displaystyle\sum_{(x_i,w_i) \in WD}w_i q^*(x_i)$ or else outputs $\bot$ (``Fail''). If it outputs a statistical query $q$, we say the oracle has succeeded.
\end{definition}

In contrast, a heuristic optimization oracle (that is not certifiable) has no guarantees of correctness. Without loss of generality, such oracles never need to return ``Fail'' (since they can always instead output a default statistical query in this case).

\begin{definition}
A (non-certifiable) heuristic optimization oracle for a class of queries $\cQ$ is an arbitrary polynomial time algorithm $M:(\cX \times \mathbb{R})^*\rightarrow \cQ$. Given a call to the oracle defined by a weighted dataset $WD \in (\cX\times \mathbb{R})^*$ we say that the oracle has succeeded on this call up to error $\alpha$ if it outputs a query $q$ such that $\displaystyle\sum_{(x_i,w_i) \in WD}w_i q(x_i) \leq \min_{q^* \in \cQ}\displaystyle\sum_{(x_i,w_i) \in WD}w_i q^*(x_i) + \alpha$. If it succeeds up to error 0, we just say that the heuristic oracle has succeeded. Note that there may not be any efficient procedure to determine whether the oracle has succeeded up to error $\alpha$.
\end{definition}

We say an algorithm $\cA_\cO$ is (certifiable)-oracle dependent if throughout the course of its run it makes a series of (possibly adaptive) calls to a (certifiable) heuristic optimization oracle $\cO$. An oracle-dependent algorithm $\cA_\cO$ is \textit{oracle equivalent} to an algorithm $\cA$ if given access to a perfect optimization oracle $\cO^*$, $\cA_{\cO^*}$ induces the same distribution on outputs as $\cA$. We now state an intuitive lemma (that could also be taken as a more formal definition of \textit{oracle equivalence}). See the Appendix for a proof.

\begin{restatable}{lemma}{lemcoupling}
\label{lemcoupling}
Let $\cA_\cO$ be a certifiable-oracle dependent algorithm that is \textit{oracle equivalent} to $\cA$. Then for any fixed input dataset $S$, there exists a coupling between $\cA(S)$ and $\cA_\cO(S)$ such that $\Pr[\cA_\cO(S) = a| \cA_\cO(S) \neq \bot] = \Pr[\cA(S) = a| \cA_\cO(S) \neq \bot]$.
\end{restatable}

We will also discuss differentially private heuristic optimization oracles, in order to state additional consequences of our construction in Section~\ref{sec:synth}. Note that because differential privacy precludes exact computations, differentially private heuristic oracles are necessarily non-certifiable, and will never succeed up to error 0.

\begin{definition}
\label{privoracle}
A weighted $(\epsilon, \delta)$-differentially private $(\alpha,\beta)$-accurate learning oracle for a class of statistical queries $\cQ$ is an $(\epsilon, \delta)$ differentially private algorithm $\cO:(\cX \times \mathbb{R})^*\rightarrow C$ that takes as input a weighted dataset $WD \in (\cX\times \mathbb{R})^*$ and outputs a query $q_{priv} \in \cQ$ such that with probability $1-\beta$:
$$\displaystyle\sum_{(x_i,w_i) \in WD}w_i q_{priv}(x_i)- \argmin_{q^* \in C} \displaystyle\sum_{(x_i,w_i) \in WD}w_i q^*(x_i) \leq \alpha$$
\end{definition}
We say that an algorithm is \emph{oracle-efficient} if given access to an oracle (in this paper, always a weighted optimization oracle for a class of statistical queries) it runs in polynomial time in the length of its input, and makes a polynomial number of calls to the oracle. In practice, we will be interested in the performance of oracle-efficient algorithms when they are instantiated with heuristic oracles. Thus, we further require oracle-efficient algorithms to halt in polynomial time even when the oracle fails. When we design algorithms for optimization and synthetic data generation problems, their $(\alpha,\beta)$-accuracy guarantees will generally rely on all queries to the oracle succeeding (possibly up to error $O(\alpha)$). If our algorithms are merely \emph{oracle equivalent} to differentially private algorithms, then their privacy guarantees depend on the correctness of the oracle. However, we would prefer that the \emph{privacy} guarantee of the algorithm not depend on the success of the oracle. We call such algorithms \emph{robustly} differentially private.

\begin{definition}
An oracle-efficient algorithm $M$ is $(\epsilon,\delta)$-robustly differentially private if it satisfies $(\epsilon,\delta)$-differential privacy even under worst-case performance of a heuristic optimization oracle. In other words, it is differentially private for every heuristic oracle $\cO$ that it might be instantiated with.
\end{definition}

We write that an oracle efficient algorithm is non-robustly differentially private to mean that it is oracle equivalent to a differentially private algorithm.

\section{Oracle Efficient Optimization}
\label{rspm}
In this section, we show how weighted optimization oracles can be used
to give differentially private oracle-efficient optimization
algorithms for many classes of queries with performance that is worse only by a $\sqrt{d}$ factor compared to that of the (computationally inefficient) exponential
mechanism. The first algorithm we give is not robustly differentially
private --- that is, its differential privacy guarantee relies on having access to a perfect oracle. We then show how to make that algorithm (or any other algorithm that is oracle equivalent to a differentially private algorithm)
robustly differentially private when instantiated with a certifiable
heuristic optimization oracle.

\subsection{A (Non-Robustly) Private Oracle Efficient Algorithm}
In this section, we give an oracle-efficient (non-robustly) differentially private optimization algorithm that works for any class of statistical queries that has a small separator set. Intuitively, it is attempting to implement the ``Report-Noisy-Min'' algorithm (see e.g. \cite{DworkRoth}), which outputs the query $q$ that minimizes a (perturbed) estimate $\hat q(S) \equiv q(S) + Z_q$ where $Z_q \sim \Lap{1/\eps}$ for each $q \in \cQ$. Because Report-Noisy-Min samples an independent perturbation for each query $q \in \cQ$, it is inefficient: its run time is linear in $|\cQ|$. Our algorithm -- ``Report Separator-Perturbed Min'' (RSPM) -- instead augments the dataset $S$ in a way that implicitly induces perturbations of the query values $q(S)$. The perturbations are no longer independent across queries, and so to prove privacy, we need to use the structure of a separator set.

The algorithm is straightforward: it simply augments the
dataset with one copy of each element of the separator set, each with a weight drawn independently from the Laplace distribution. All original elements in the dataset are assigned weight 1. The algorithm then simply passes this weighted dataset to the weighted optimization oracle, and outputs the resulting query. The number of random variables that need to be sampled is therefore now equal to the size of the separator set, instead of the size of $\cQ$. The algorithm is closely related to a no-regret learning algorithm given in \cite{oracle16} --- the only difference is in the magnitude of the noise added, and in the analysis, since we need a substantially stronger form of stability. 

%\begin{figure}[h]
\begin{algorithm}
\label{alg:rspm}
\textbf{Report Separator-Perturbed Min (RSPM)}\newline
\textbf{Given}: A separator  set $U = \{e_1,\ldots,e_m\}$ for a class of statistical queries $\cQ$, a weighted optimization oracle $\cO^*$ for $\cQ$, and a privacy parameter $\epsilon$.\newline
\textbf{Input}: A dataset $S \in \cX^n$ of size $n$.\newline
\textbf{Output}: A statistical query $q \in \cQ$.
\begin{algorithmic}
\STATE Sample $\eta_i \sim Lap(m/\epsilon)$ for $i \in \{1,\ldots,m\}$
\STATE Construct a weighted dataset $WD$ of size $n + m$ as follows:
$$WD(S,\eta) = \{(x_i, 1) : x_i \in S\} \cup \{(e_i, \eta_i) : e_i \in U\}$$
\STATE Output $q = \cO^*(WD(S,\eta))$.
\end{algorithmic}
\end{algorithm}
%\caption{The Report Separator Perturbed Min (RSPM) Algorithm}
%\end{figure}

It is thus immediate that the Report Separator-Perturbed Min algorithm is oracle-efficient whenever the size of the separator set $m$ is polynomial: it simply augments the dataset with a single copy of each of $m$ separator elements, makes $m$ draws from the Laplace distribution, and then makes a single call to the oracle:
\begin{theorem}
The Report Separator-Perturbed Min algorithm is oracle-efficient.
\end{theorem}

The accuracy analysis for the Report Separator-Perturbed Min algorithm is also straightforward, and follows by bounding the weighted sum of the additional entries added to the original data set.

\begin{theorem}
\label{thm:rspm}
The Report Separator-Perturbed Min algorithm is an $(\alpha,\beta)$-minimizer for $\cQ$ for:
$$\alpha = \frac{2m^2\log(m/\beta)}{ \epsilon n}$$
\end{theorem}
\begin{proof}
%We analyze un-normalized statistical queries $q$; e.g. $q(S) = \sum_{x \in S}q(x)$ as opposed to $q(S) = \frac{1}{n}\sum_{x \in S}q(x)$.
Let $q'$ be the query returned by RSPM, and let $q^*$ be the true minimizer $q^* = \arg\min_{q \in \cQ}q^*(S)$. Then we show that with probability $1-\beta, q'(S) \leq q^*(S) + \alpha$. By the CDF of the Laplace distribution and a union bound over the $m$ random variables $\eta_i$, we have that with probability $1-\beta$:
$$\forall i,\ |\eta_i| \leq \frac{m \log(m/\beta)}{  \epsilon}.$$ Since for every query $q$, $q(e_i) \in [0,1]$, this means that with probability $1-\beta$, $q'(WD) \geq q'(S) - m \cdot \frac{m \log(m/\beta)}{\epsilon n}$. Similarly $q^*(WD) \leq q^*(S) + m \cdot \frac{m \log(m/\beta)}{\epsilon n}$. Combining these bounds gives:
$$q'(S) \leq q'(WD) + m^2 \frac{\log(m/\beta)}{\epsilon n} \leq q^*(WD) + m^2\frac{\log(m/\beta)}{\epsilon n} \leq q^*(S) + \frac{2m^2\log(m/\beta)}{\epsilon n}$$ as desired, where the second inequality follows because by definition, $q'$ is the true minimizer on the weighted dataset $WD$.
%The factor of $\frac{1}{n}$ in the theorem is because the error is stated for normalized queries.
\end{proof}
\begin{remark}
We can bound the expected error of RSPM using Theorem~\ref{thm:rspm} as well. If we denote the error of RSPM by $E$, we've shown that for all $\beta$, $\prob{E \geq \frac{2m^2\log(m/\beta)}{ \epsilon n}} \leq \beta$. Thus $\prob{\frac{\epsilon n E}{2m^2}-\log m \geq \log(1/\beta)} \leq \beta$ for all $\beta$. Let $\tilde{E} = \max(0, \frac{\epsilon n E}{2m^2}-\log m)$. Since $\tilde{E}$ is non-negative:
$$\Ex{}{\tilde{E}} = \int_{0}^{\infty}\prob{\tilde{E} \geq t} \leq \int_0^{\infty}e^{-t} = 1.$$ Hence $\frac{\epsilon n \Ex{}{E}}{2m^2}-\log m \leq \Ex{}{\tilde{E}} \leq 1$, and so $\Ex{}{E} \leq \frac{2m^2}{\epsilon n}(1+ \log m)$.
\end{remark}

The privacy analysis is more delicate, and relies on the correctness of the oracle.
\begin{theorem}
If $\cO^*$ is a weighted optimization oracle for $\cQ$, then the Report Separator-Perturbed Min algorithm is $\epsilon$-differentially private.
\end{theorem}
\begin{proof}
We begin by introducing some notation.  Given a weighted dataset $WD(S,\eta)$, and a query $q \in \cQ$, let $q(S, \eta) = q(S) + \sum_{e_i \in U}q(e_i)\eta_i$ be the value when $q$ is evaluated on the weighted dataset given the realization of the noise $\eta$. To allow us to distinguish queries that are output by the algorithm on different datasets and different realizations of the perturbations, write $\cQ(S,\eta) = \cO^*(WD(S,\eta))$. Fix any $q \in \cQ$, and define:
$$\cE(q,S) = \{\eta : \cQ(S,\eta) =  q\}$$
to be the event defined on the perturbations $\eta$ that the mechanism outputs query $q$. Given a fixed $q \in \cQ$ we define a mapping $f_q(\eta):\mathbb{R}^m\rightarrow \mathbb{R}^m$ on noise vectors as follows:
\begin{enumerate}
\item If $q(e_i) = 1, f_{q}(\eta)_i = \eta_i-1$
\item If $q(e_i) = 0, f_{q}(\eta)_i = \eta_i + 1$
\end{enumerate}
Equivalently, $f_{q}(\eta)_i = \eta_i + (1-2q (e_i))$.% where $q = \cQ(S, \eta)$.

 We now make a couple of observations about the function $f_q$.

\begin{lemma}
\label{lem:inclusion}
Fix any $\hat q \in \cQ$ and any pair of neighboring datasets $S, S'$. Let $\eta \in \cE(\hat q, S)$ be such that $\hat{q}$ is the unique minimizer $\hat q \in \inf_{q \in \cQ}q(S, \eta)$. Then $f_{\hat q}(\eta) \in \cE(\hat q, S')$.
In particular, this implies that for any such $\eta$:
$$\mathbbm{1}(\eta \in \cE(\hat q, S)) \leq \mathbbm{1}(f_{\hat q}(\eta) \in \cE(\hat q, S'))$$
\end{lemma}
%\ar{Rewrote proof. Check.}
\begin{proof}
For this argument, it will be convenient to work with \emph{un-normalized} versions of our queries, so that $q(S) = \sum_{x_i \in S} q(x_i)$ --- i.e. we do not divide by the dataset size $n$. Note that this change of normalization does not change the identity of the minimizer. Under this normalization, the queries $q$ are now $1$-sensitive, rather than $1/n$ sensitive.

Recall that $\cQ(S, \eta) = \hat{q}$. Suppose for point of contradiction that $\cQ(S', f_{\hat q}(\eta)) = \tilde{q} \neq \hat q$. This in particular implies that  $\tilde{q}(S', f_{\hat q}(\eta)) \leq \hat{q}(S', f_{\hat q}(\eta)).$% or

% $$\tilde{q}(S') + \sum_{j \neq i}^m\tilde{q}(e_j)\eta_j + \tilde{q}(e_i)f(\eta_i) \leq \hat{q}(D') + \sum_{j \neq i}^m\hat{q}(e_j)\eta_j + %\hat{q}(e_i)f(\eta_i)$$
% \ar{Still confused about what $i$ is and why it is excluded from these sums}

We first observe that $\hat{q}(S', \eta)-\tilde{q}(S', \eta) < 1$. This follows because:
\begin{equation}
\label{ineq:first}
\tilde{q}(S',\eta) \geq \tilde{q}(S,\eta) -1
> \hat{q}(S,\eta) - 1
\end{equation}
Here the first inequality follows because the un-normalized queries $q$ are 1-sensitive, and the second follows because $\hat q \in \arg\min_{q \in \cQ}q(S, \eta)$ is the unique minimizer.

%$\hat{q}(S', \eta)-\tilde{q}(S, \eta) = \hat{q}(S', \eta)-\hat{q}(S, \eta) + \hat{q}(S, \eta) - \tilde{q}(S', \eta) + \tilde{q}(S, \eta) - %\tilde{q}(S, \eta) = \hat{q}(S', \eta)-\hat{q}(S, \eta) + \tilde{q}(S, \eta)-\tilde{q}(S', \eta) + \hat{q}(S, \eta)-\tilde{q}(S, \eta) \leq
%\hat{q}(S', \eta)-\hat{q}(S, \eta) + \tilde{q}(S, \eta)-\tilde{q}(S', \eta) - 1 \leq 1 + 1 - 1 = 1$. Here the penultimate step comes from the %uniqueness of $\hat{q}$, and the ultimate step comes from the fact that for all $q \in \cQ, \eta, |q(S', \eta)-q(S, \eta)| \leq 1$.
Next, we write:
$$\tilde{q}(S', f_{\hat q}(\eta)) - \hat{q}(S', f_{\hat q}(\eta)) = \tilde{q}(S',\eta) - \hat{q}(S', \eta) + \sum_{i = 1}^{m}(\tilde{q}(e_i)-\hat{q}(e_i))(f_{\hat q}(\eta_i)-\eta_i)$$
Consider each term in the final sum: $(\tilde{q}(e_i)-\hat{q}(e_i))(f_{\hat q}(\eta_i)-\eta_i)$. Observe that by construction, each of these terms is non-negative: Clearly if $\tilde{q}(e_i) = \hat{q}(e_i)$, then the term is $0$. Further, if $\tilde{q}(e_i) \neq \hat q(e_i)$, then by construction, $(\tilde{q}(e_i)-\hat{q}(e_i))(f_{\hat q}(\eta_i)-\eta_i) = 1$. Finally, by the definition of a separator set, we know that there is at least one index $i$ such that $\tilde q(e_i) \neq \hat q (e_i)$. Thus, we can conclude:
$$\tilde{q}(S', f_{\hat q}(\eta)) - \hat{q}(S', f_{\hat q}(\eta)) \geq \tilde{q}(S',\eta) - \hat{q}(S', \eta) + 1 > 0$$
where the final inequality follows from applying inequality \ref{ineq:first}. But rearranging, this means that $\hat{q}(S', f_{\hat q}(\eta)) < \tilde{q}(S', f_{\hat q}(\eta))$, which contradicts the assumption that  $\cQ(S', f_{\hat q}(\eta)) = \tilde{q}$.
\end{proof}

 Let $p$ denote the probability density function of the joint distribution of the Laplace random variables $\eta$, and by abuse of notation also of each individual $\eta_i$.
 \begin{lemma}
 \label{lem:density}
 For any $r \in \mathbb{R}^m, q \in \cQ$:
 $$p(\eta = r) \leq e^{\epsilon} p(\eta = f_q(r))$$
 \end{lemma}
 \begin{proof}
For any index $i$ and $z \in \mathbb{R}$, we have  $p(\eta_i = z) =  \frac{\epsilon}{2m} e^{-|z|\epsilon/m}$. In particular, if $|x-y| \leq 1$, $p(\eta_i = y) \leq e^{\epsilon/m}p(\eta_i = x)$.
%Therefore we can compute:
%$$\frac{p(\eta_i = y)}{p(\eta_i = x)} = e^{(|x|-|y|)\epsilon/m} \leq e^{|x-y|\epsilon/m} \leq e^{\epsilon/m}$$ as desired.
Since for all $i$ and $r \in \mathbb{R}^m$ $|f_q(r)_i-r_i| \leq 1$, we have:
$$\frac{p(\eta = f_q(r))}{p(\eta = r)} = \prod_{i = 1}^{m} \frac{p(\eta_i = f_q(r)_i)}{p(\eta_i = r_i)} \leq \prod_{i=1}^{m} e^{\epsilon/m} = e^{\epsilon}.$$
 \end{proof}

\begin{lemma}
\label{lem:B}
Fix any class of queries $\cQ$ that has a finite separator set $U = \{e_1,\ldots,e_m\}$.
For every dataset $S$ there is a subset $B \subseteq \mathbb{R}^m$ such that:
\begin{enumerate}
\item $\Pr[\eta \in B] = 0$ and
\item On the restricted domain $\mathbb{R}^m\setminus B$,  there is a unique minimizer $q' \in \arg\min_{q \in \cQ}q(S, \eta)$
\end{enumerate}
\end{lemma}
\begin{proof}
Let:
$$B = \{\eta : \left|\arg\min_{q \in \cQ} (q(S) + \sum_{i=1}^{m}\eta_iq(e_i))\right| > 1\}$$
be the set of $\eta$ values that do \emph{not} result in unique minimizers $q'$. \\
Because $\cQ$ is a finite set\footnote{Any class of queries $\cQ$ with a separator set of size $m$ can be no larger than $2^m$.}, by a union bound it suffices to show that for any two distinct queries $q_1, q_2 \in \cQ$,
$$\Pr_{\eta}\left[q_1(S) + \sum_{i=1}^{m}\eta_iq_1(e_i) = q_2(S) + \sum_{i=1}^{m}\eta_iq_2(e_i)\right] = 0.$$
This follows from the continuity of the Laplace distribution. Let $i$ be any index such that $q_1(e_i) \neq q_2(e_i)$ (recall that by the definition of a separator set, such an index is guaranteed to exist). For any fixed realization of $\{\eta_j\}_{j \neq i}$, there is a single value of $\eta_i$ that equalizes $q_1(S,\eta)$ and $q_2(S,\eta)$. But any single value is realized with probability $0$.  \\
\iffalse
For part $2$: We show the stronger claim, that on $\mathbb{R}^m\setminus B$, $f$ is a bijection: we do this by constructing an inverse function $g$ such that for every $\eta \in \mathbb{R}^m\setminus B$, $g(f(\eta)) = \eta$. Recall that  $f(\eta)_i = \eta_i + (1-2\hat q (e_i))$ where $\hat q = \cQ(S, \eta)$. Define $g$ such that  $g(\eta)_i = \eta_i - (1-2\bar q (e_i))$ where $\bar q = \cQ(S, \eta)$. By Lemma \ref{lem:inclusion}, we know that if $\eta \in \mathbb{R}^m \setminus B$, then $\mathbbm{1}(\eta \in \cE(\hat q, S)) \leq \mathbbm{1}(f(\eta) \in \cE(\hat q, S'))$. In particular, this implies that for $\bar q = \cQ(S,f(\eta))$ and $\hat q = \cQ(S, \eta)$, we have $\bar q = \hat q$. Thus:
$$g(f(\eta)) = \eta_i + (1-2\hat q (e_i)) -  (1-2\bar q (e_i)) = \eta_i.$$
\fi
\end{proof}

We now have enough to complete the proof. We have for any query $\hat q$:
\begin{align*}
\Pr[RSPM(S) = \hat q] &= \Pr[\eta \in \cE(\hat q, S)] \\
 &= \int_{\mathbb{R}^m} p(\eta) \mathbbm{1}(\eta \in \cE(\hat q, S)) d\eta \\
&= \int_{\mathbb{R}^m \setminus B} p(\eta) \mathbbm{1}(\eta \in \cE(\hat q, S)) d\eta && B \text{ has 0 measure. (Lemma \ref{lem:B})}\\
&\leq \int_{\mathbb{R}^m \setminus B} p(\eta) \mathbbm{1}(f_{\hat q}(\eta) \in \cE(\hat q, S')) d\eta && \text{Lemma~\ref{lem:B} $\implies$ Lemma \ref{lem:inclusion}}\\
&\leq \int_{\mathbb{R}^m \setminus B} e^{\epsilon} p(f_{\hat q}(\eta)) \mathbbm{1}(f_{\hat q}(\eta) \in \cE(\hat q, S')) d\eta && \text{Lemma \ref{lem:density}}\\
&\leq \int_{\mathbb{R}^m \setminus f_{\hat q}(B)} e^{\epsilon} p(\eta) \mathbbm{1}(\eta \in \cE(\hat q, S')) \left|\frac{\partial f_{\hat q}}{\partial \eta}\right| d\eta &&  \text{Change of variables $\eta \to f_{\hat q}(\eta)$} \\
&= \int_{\mathbb{R}^m} e^{\epsilon} p(\eta) \mathbbm{1}(\eta \in \cE(\hat q, S')) d\eta && f_{\hat q}(B) \text{ has 0 measure,} \left|\frac{\partial f_{\hat q}}{\partial \eta}\right|=1\\
&= e^{\epsilon} \Pr[\eta \in \cE(\hat q, S')] \\
&= e^{\epsilon} \Pr[RSPM(S') = \hat q]
\end{align*}
\end{proof}

In Appendix \ref{sec:gauss}, we give a somewhat more complicated analysis to show that by using Gaussian perturbations rather than Laplace perturbations, it is possible to improve the accuracy of the RSPM algorithm by a factor of $\sqrt{m}$, at the cost of satisfying $(\epsilon,\delta)$-differential privacy:
\begin{theorem}
The Gaussian RSPM algorithm is $(\epsilon,\delta)$-differentially private, and is an oracle-efficient $(\alpha,\beta)$-minimizer for any class of functions $\cQ$ that has a universal identifications sequence of size $m$ for:
$$\alpha =O\left(\frac{m\sqrt{m\ln(m/\beta) \ln(1/\delta)} }{\eps n}\right)$$
\end{theorem}
See Appendix \ref{sec:gauss} for the algorithm and its analysis.

It is instructive to compare the accuracy that we can obtain with oracle-efficient algorithms to the accuracy that can be obtained via the (inefficient, and generally optimal) exponential mechanism based generic learner from \cite{KLNRS08}. The existence of a universal identification set for $\cQ$ of size $m$ implies $|\cQ| \leq 2^{m}$ (and for many interesting classes of queries, including conjunctions, disjunctions, parities, and discrete halfspaces over the hypercube, this is an equality --- see Appendix \ref{app:separator}). Thus, the exponential-mechanism based learner from \cite{KLNRS08} is $(\alpha,\beta)$-accurate for:

$$\alpha = O\left(\frac{m +  \log(1/\beta)}{\epsilon n}\right).$$
Comparing this bound to ours, we see that we can obtain oracle-efficiency at a cost of roughly a factor of $\sqrt{m}$ in our error bound. Whether or not this cost is necessary is an interesting open question.

We can conclude that for a wide range of hypothesis classes $\cQ$ including boolean conjunctions, disjunctions, decision lists, discrete halfspaces, and several families of circuits of logarithmic depth (see Appendix \ref{app:separator}) there is an oracle-efficient differentially private learning algorithm that obtains accuracy guarantees within small polynomial factors of the optimal guarantees of the (inefficient) exponential mechanism.

%%% Local Variables:
%%% mode: latex
%%% TeX-master: "heuristicDP"
%%% End:

\subsection{A Robustly Differentially Private Oracle-Efficient Algorithm}
The RSPM algorithm is not \emph{robustly} differentially private, because its privacy proof depends on the oracle succeeding. This is an undesirable property for RSPM and other algorithms like it, because we do not expect to have access to \emph{actual} oracles for hard problems even if we expect that there are certain families of problems for which we can reliably solve typical instances\footnote{There may be situations in which it is acceptable to use non robustly differentially private oracle-efficient algorithms --- for example, if the optimization oracle is so reliable that it has never been observed to fail on the domain of interest. But robust differential privacy provides a worst-case guarantee which is preferable.}. In this section, we show how to remedy this: we give a black box reduction, starting from a (non-robustly) differentially private algorithm $\cA_\cO$ that is implemented using a \emph{certifiable} heuristic\footnote{We recall that heuristics for solving integer programs (such as cutting planes methods, branch and bound, and branch and cut methods, as implemented in commercial solvers) and SAT solvers are certifiable.} oracle $\cO$,  and producing a robustly differentially private algorithm $\tilde{\cA}_\cO$ for solving the same problem. $\tilde{\cA}_\cO$ will be $(\epsilon,\delta)$-differentially private for a parameter $\delta$ that we may choose, and will have a factor of roughly $\tilde O(1/\delta)$ running time overhead on top of $\cA_\cO$. So if $\cA_\cO$ is oracle efficient, so is $\tilde{\cA}_\cO$ whenever the chosen value of $\delta \geq 1/\textrm{poly}(n)$.  If the oracle never fails, then we can prove utility guarantees for it when $\cA_{\cO}$ has such guarantees, since it just runs $\cA_{\cO}$ (using a smaller privacy parameter) on a random sub-sample of the original dataset. But the privacy guarantees hold even in the worst case of the behavior of the oracle. We call this reduction the \textit{Private Robust Subsampling Meta Algorithm} or \textbf{PRSMA}.

\begin{algorithm}
\label{meta}
\textbf{Private Robust Subsampling Meta Algorithm (PRSMA)}\\
\textbf{Given}: Privacy parameters $\epsilon,\delta \geq 0$ and an oracle-efficient differentially private algorithm $\cA_\cO^{\epsilon}: \cX^n \to \cM$, implemented with a certifiable heuristic oracle $\cO$. \newline
\textbf{Input}: A dataset $S \in \cX^n$ of size $n$. \newline
\textbf{Output}: An output $m \in \cM$ or $\bot$ (``Fail'').
\begin{algorithmic}[1]
\STATE Randomly partition $S$ into $K = \frac{1}{\epsilon}(1+ \log(\frac{2}{\delta}))$ equally sized datasets $\{S_i\}_{i=1}^{K}$. (If $n$ is not divisible by $K$, first discard $n\mod K$ elements at random.)
\FOR{ $i = 1 \ldots K$}
\STATE Set $o_i = PASS$
	\FOR{$t = 1 \ldots \frac{\log(K/\delta)}{\delta}$}
		\STATE Compute $\cA_\cO^{\epsilon'}(S_i)=a_{it}$, where $\epsilon' = \frac{1}{\sqrt{8\frac{n}{K} \log(2K/\delta)}}$
		\STATE If $a_{it} = \bot,$ set $o_i = \bot$
	\ENDFOR
\ENDFOR
\STATE Compute $T = \#\{ o_i \neq \bot \}$. Let $\tilde{T} = T + z,$ where  $z \sim \text{Lap}(\frac{1}{\epsilon})$.
\STATE Test if  $\tilde{T} > \frac{1}{\epsilon}(1+ \log(\frac{1}{\delta}))$, if no output $\bot$ and halt. \textbf{Else:}
\STATE Sample $a$ uniformly at random from $\{a_{it}: o_i \neq \bot\}$.
\STATE Output $a$.
\end{algorithmic}
\end{algorithm}

\subsubsection{Intuition and Proof Outline}

Before we describe the analysis of \textbf{PRSMA}, a couple of remarks are helpful in order to set the stage.
\begin{enumerate}
\item At first blush, one might be tempted to assert that if an oracle-efficient non-robustly differentially private algorithm is implemented using a certifiable heuristic oracle, then it will sample from a differentially private distribution \emph{conditioned on the event that the heuristic oracle doesn't fail}. But a moment's thought reveals that this isn't so: the possibility of failures both on the original dataset $S$ and on the (exponentially many) neighboring datasets $S'$ can substantially change the probabilities of arbitrary events $\Omega$, and how these probabilities differ between neighboring datasets.
\item Next, one might think of the following simple candidate solution: Run the algorithm $\cA_{\cO}(S)$ roughly $\tilde O(1/\delta)$ many times in order to check that the failure probability of the heuristic algorithm on $S$ is $\ll \delta$, and then output a sample of $\cA_{\cO}(S)$ only if this is so. But this doesn't work either: the failure probability itself will change if we replace $S$ with a neighboring dataset $S'$, and so this won't be differentially private. In fact, there is no reason to think that the failure probability of $\cA_{\cO}$ will be a low sensitivity function of $S$, so there is no way to privately estimate the failure probability to non-trivial error.
\end{enumerate}

It \emph{is} possible to use the \emph{subsample-and-aggregate} procedure of \cite{subsample} to randomly partition the dataset into $K$ pieces $S_i$, and privately estimate on \emph{how many} of these pieces $\cA_{\cO}(S_i)$ fails with probability $\ll \delta$. The algorithm can then then fail if this private count is not sufficiently large. In fact, this is the first thing that \textbf{PRSMA} does, in lines 1-10, setting $o_i = PASS$ for those pieces $S_i$ such that it seems that the probability of failure is $\ll \delta$, and setting $o_i = \bot$ for the others.

But the next step of the algorithm is to randomly select one of the partition elements $S_i$ amongst the set that passed the earlier test: i.e. amongst the set such that $o_i  \neq \bot$ --- and return one of the outputs $a$ that had been produced by running $\cA_{\cO}(S_i)$. It is not immediately clear why this should be private, because \emph{which} partition elements passed the test $\{i : o_i \neq \bot\}$ is not itself differentially private. Showing that this results in a differentially private output is the difficult part of the analysis.

%To show that \textbf{PRSMA} preserves differential privacy, we need to show that the distributions on the output of \textbf{PRSMA} under two adjacent databases are close. Ensuring that the probability of outputting $\bot$ is close under adjacent databases is simple - it can be achieved by a ``Subsample and Aggregate" style approach \cite{DworkRoth} \ar{Find the right citation -- its not this}. In Subsample and Aggregate the goal is to privately compute a function $f(S)$ which has bad worst-case sensitivity. Instead of computing $f(S)$ and adding noise that scales with the sensitivity, $S$ is split into smaller datasets $\{D_i\},$ and $f(D_i)$ is computed exactly, without noise. Finally the results $f(D_i)$ are privately aggregated.  In our case the function $f(S)$ outputs $\bot$ if the oracle fails on $S$, else it outputs $1$ Our aggregation function just privately sums up the number of failures using the Laplace Mechanism, and then compares the noisy sum to a threshold. If adjacent databses $S, S'$ differ in the first elements $x_1, x_1',$ then by splitting the dataset into disjoint partitions, and only outputting $\bot$ if $\tilde{T}$ is less than the threshold, we've limited the sensitivity of the failure event to the change in $x_1$. We can then ensure that the failure probabilities are close under $S, S'$ by adding a small amount of noise.

To get an idea of the problem that we need to overcome, consider the following situation which our analysis must rule out: Fix a partition of the dataset $S_1,\ldots,S_K$, and imagine that each partition element passes: we have $o_i \neq \bot$ for all $i$. Now suppose that there is some event $\Omega$ such that $\Pr[\cA_{\cO}(S_1) \in \Omega] \geq 1/2$, but $\Pr[\cA_{\cO}(S_i) \in \Omega]$ is close to 0 for all $i \neq 1$. Since $K \approx 1/\epsilon$, and the final output is drawn from a uniformly random partition element, this means that \textbf{PRSMA} outputs an element of $\Omega$ with probability $\Omega(\epsilon)$. Suppose that on a neighboring dataset $S'$, $S_1$ no longer passes the test and has $o_1 = \bot$. Since it is no longer a candidate to be selected at the last step, we now have that on $S'$, \textbf{PRSMA} outputs an element of $\Omega$ with probability close to $0$. This is a violation of $(\epsilon,\delta)$-differential privacy for any non-trivial value of $\delta$ (i.e. $\delta \leq O(\epsilon)$).

The problem is that (fixing a partition of $S$ into $S_1,\ldots,S_K$) moving to a neighboring dataset $S'$ can potentially arbitrarily change the probability that any single element $S_i$ survives to step 11 of the algorithm, which can in principle change the probability of arbitrary events $\Omega$ by an \emph{additive} $\pm O(\epsilon)$ term, rather than a \emph{multiplicative} $1 \pm O(\epsilon)$ \emph{factor}.

Since we are guaranteed that (with high probability) if we make it to step 11 without failing, then at least $\Omega(1/\epsilon)$ elements $S_i$ have survived with $o_i \neq \bot$, it would be sufficient for differential privacy if for every event $\Omega$, the probabilities $\Pr[\cA_\cO(S_i) \in \Omega]$ were within a constant factor of each other, for all $i$. Then a change of whether a single partition element $S_i$ survives with $o_i \neq \bot$ or not would only add or remove an $\epsilon$ \emph{fraction} of the total probability mass on event $\Omega$.  While this seems like a ``differential-privacy'' like property, but it is not clear that the fact that $\cA_{\cO^*}$ is differentially private can help us here, because the partition elements $S_i,S_j$ are not neighboring datasets --- in fact, they are disjoint. But as we show, it does in fact guarantee this property \emph{if} we set the privacy parameter $\epsilon'$ to be sufficiently small --- to roughly $O(1/\sqrt{n/K})$ in step 5.

With this intuition setting the stage, the roadmap of the proof is as follows. For notational simplicity, we write $\cA(\cdot)$ to denote $\cA_{\cO^*}(\cdot)$, the oracle-efficient algorithm when implemented with a perfect oracle.
\begin{enumerate}
\item We observe that $\epsilon$-differential privacy implies that the log-probability of any event $\Omega$ when $\cA(\cdot)$ is run on $S_i$ changes by less than an additive factor of $\epsilon$ when an element of $S_i$ is changed. We use a method of bounded differences argument to show that this implies that the log-probability density function concentrates around its expectation, where the randomness is over the subsampling of $S_i$  from $S$. A similar result is proven in \cite{CLNRW16} to show that differentially private algorithms achieve what they call ``perfect generalization." We need to prove a generalization of their result because in our case, the elements of $S_i$ are not selected independently of one another. This guides our choice of $\epsilon'$ in step 5 of the algorithm.
(Lemma~\ref{mcdiarmid})
\item We show that with high probability, for every $S_i$ such that $o_i \neq \bot$ after step 10 of the algorithm, $\cA_\cO(S_i)$ fails with probability at most $O(\delta)$. By Lemma~\ref{lemcoupling}, this implies that it is $\delta$-close in total variation  distance to $\cA(S_i)$.
\item We observe that fixing a partition, on a neighboring dataset, only one of the partition elements $S_i$ changes --- and hence changes its probability of having $o_i \neq \bot$. Since with high probability, conditioned on \textbf{PRSMA} not failing, $\Omega(1/\epsilon)$ partition elements survive with $o_i \neq \bot$, parts 1 and 2 imply that changing a single partition element $S_i$ only changes the probability of realizing any outcome event by a \emph{multiplicative} factor of $\approx 1+\epsilon$.
\end{enumerate}

\subsubsection{The Main Theorem}
 \begin{theorem}
 \label{thm:prsma}
\textbf{PRSMA} is $(\epsilon,\delta)$ differentially private when given as input:
\begin{enumerate}
\item An oracle-efficient non-robustly differentially private algorithm $\cA_\cO$ implemented with a certifiable heuristic oracle $\cO$, and
\item Privacy parameters $(\epsilon^*,\delta^*)$ where $\epsilon^* = \frac{\epsilon}{62} \leq \frac{1}{2}$ and  $\delta^* = \frac{\delta}{11} \leq \frac{1}{2}$.
\end{enumerate}
\end{theorem}

\begin{proof}
We analyze \textbf{PRSMA} with privacy parameters $\epsilon$ and  $\delta$, optimizing the constants at the end. Fix an input dataset $S$ with $|S| = n$, and an adjacent dataset $S' \sim S$, such that without loss of generality $S, S'$ differ in the element $x_1 \neq x_1'$. We denote the \textbf{PRSMA} routine
with input $\cA_\cO$ and dataset $S$ by $\cps$. We first observe that:
$$
\Pr[\cps = \bot] \leq e^{\epsilon}\Pr[\cpsp = \bot]
$$
This is immediate since the indicator for a failure is
a post-processing of the Laplace mechanism. Since $x_1$ can affect at most one oracle failure, $T$ is $1$-sensitive, and so publishing $\tilde{T} = T + \Lap{\frac{1}{\epsilon}}$ satisfies $\epsilon$-differential privacy since it is an invocation of the Laplace Mechanism defined in Section \ref{sec:DPprelims}. (This can also be viewed as an instantiation of the ``sub-sample and aggregate procedure of \cite{subsample}).

We now proceed to the meat of the argument. To establish $(\epsilon, \delta)$ differential privacy we must
reason about the probability of arbitrary events $\Omega \subset \cM$, rather than just individual outputs $a$. We want to show:
$$
\Pr[\cps \in \Omega] \leq e^{\epsilon}\Pr[\cpsp \in \Omega] + \delta
$$

We first fix some notation and define a number of events that we will need to reason about. Let:
\begin{itemize}
\item $\cP_{split}^{S}$ be the uniform distribution over equal sized partitions of $S$ that the datasets $S_i$ are drawn from in line $1$; i.e. $\mathcal{P}(S) \sim \cP_{split}^{S},$ where $\mathcal{P}(S)$ is the partition of $S$ into $\{S_i\}$.
\item $\cA$ denote $\cA_{\cO^*}$, our oracle-efficient algorithm when instantiated with a perfect oracle $\cO^*$. i.e. $\cA(S)$ is the $\epsilon$-differentially private distribution that we ideally want to sample from.
\item $Z$ be the event that the Laplace noise $z$ in step $10$ of \textbf{PRSMA} has magnitude greater than  $\frac{1}{\epsilon}(\log(\frac{2}{\delta}))$.
\item $\cF = \big\{ \{o_i\}: | \{o_i : o_i \neq \bot \}| > \frac{1}{\epsilon} \big\}$. We will use $\co$ to denote a particular set $\{o_i\}$.  Let $I_{pass}^{\co}$ be the set $\{i:  o_i \neq \bot\}$. Given $\mathbf{o} \in \cF$, let $|\co|$ denote $|I_{pass}^{\co}|$.
\item $E$ be the event that for all $i = 1 \ldots K$: $\Pr[\cA_\cO(S_i) = \bot] \geq \delta \Rightarrow o_i = \bot$.
\item $i^*$ denote the index $i$ of the randomly chosen $a_{it}$ in step $11$ of \textbf{PRSMA}.
\item $Q$ be the event that the draw $\cP(S) \sim \cP_{split}^S$ is such that the probabilities of $\cA$ outputting $a \in \Omega$ when run on any two $S_i, S_j \in \cP(S)$ are within a multiplicative factor of $2$. Lemma~\ref{mcdiarmid} formally defines $Q$ and shows $\prob{Q} \geq 1-\delta$.  Let $S_Q$ denote the set of $\cP(S)$ on which event $Q$ holds.
\end{itemize}
 We now bound the probabilities of several of these events. By the CDF of the Laplace distribution,  we have $\Pr[Z] = \prob{|z| > \frac{1}{\epsilon}\log(1/\delta)} = e^{-\epsilon \cdot \frac{1}{\epsilon} \log(1/\delta)} = \delta,$ and by a union bound:
 $$\Pr[E] \geq 1-K \cdot (1-\delta)^{(\log(K/\delta)/\delta)} \geq 1-K \cdot e^{- \delta \cdot \log(K/\delta)/\delta} = 1-\delta.$$
  Let $\cL$ be the event $Z^{c} \cap E $. By the above calculation and another union bound, $\prob{\cL} \geq 1-2\delta$. Our proof now proceeds via a sequence of lemmas. All missing proofs appear in Appendix \ref{app:prsma}. We first show that $Q$ occurs with high probability.
\begin{restatable}{lemma}{mcdiarmid}
\label{mcdiarmid}
%Let $\cP(S) = \{S_i \}_{i = 1}^{K}$ be the random partition of $S$ into $K = \frac{1}{\epsilon}(1 + \log(2/\delta))$ equally sized disjoint subsets, where $|S| = n$, $S_i \subset \cX^{l}, l = \frac{n}{K}$. In our notation:
Let $\cP(S) \sim \cP_{split}^{S}$. Let $\cA: \cX^{l} \to \cM$ be an $(\epsilon', 0)$ differentially private algorithm, where:
$\epsilon' = \frac{1}{\sqrt{8\frac{n}{K} \log(2K/\delta)}}.$ Fix $\Omega \subset M$, and let $q_\Omega(S_{i}) = \log \Pr[\cA(S_{i})\in \Omega]$.
%So $q_\Omega( \cdot,)$ is the
%log-likelihood function of $\cA$ given fixed $S_i$, summed over all $a \in \Omega$.
Define $Q$ to be the event $$Q = \{\cP(S) : \max_{i, j \in 1 \ldots K} |q_\Omega(S_i)- q_\Omega(S_j)| \leq 2\}.$$ Then over the random draw of $\cP(S) \sim \cP_{split}^{S}, \; \prob{Q} \geq 1-\delta$.
\end{restatable}
The proof relies on the fact that $q_\Omega(\cdot)$ is $\epsilon$-Lipschitz. This result is similar to Theorem $5.4$ in \cite{CLNRW16}, although in our case sampling without replacement induces dependence among the elements in $S_i$, and thus we can't appeal to standard concentration inequalities for independent random variables. Instead we prove that elements sampled without replacement from a fixed set satisfy a type of negative dependence called \textit{stochastic covering}, and are $n/K$-homogenous (supported on a set of size $n/K$), which are used to prove exponential concentration of Lipschitz functions in \cite{negative}. We defer the details to the Appendix.

To establish the theorem we want to show that given an adjacent database $S'$ differing only in the first element from $S$, that
\begin{equation}
\label{dpg}
 \frac{\Pr[\cps \in \Omega]}{\Pr[\cpsp \in \Omega]} \leq e^{\epsilon^{*}} + \frac{\delta^{*}}{\Pr[\cpsp \in \Omega]}
 \end{equation}

 Our analysis will proceed by expanding the numerator by  first conditioning on $\cL$, and then on particular realizations of the partition $\cP(S)$, and on a fixed realization of $\co = \{o_i\}$. We can also restrict our attention to only summing over $\cP(S) \in S_Q$, by showing that the terms corresponding to $\cP(S) \in S_Q^{c}$ contribute at most an additive factor of $2\delta$ to the final probability. We will also only sum over $\co \in \cF$ since conditioned on $Z^{c}$, which is implied by $\cL$, these are the only $\co$ such that $\Pr[\cps \in \Omega | \co] \neq 0$.

 \begin{restatable}{lemma}{lemexpand}
 \label{lemexpand}
 $$
   \frac{\Pr[\cps \in \Omega]}{\Pr[\cpsp \in \Omega]} \leq \frac{\displaystyle\sum_{\cP(S) \in S_Q, \co \in \cF}\Pr[\cps \in \Omega | P(S), \co, \cL]\Pr[\co, P(S) | \cL] + 4\delta}{\Pr[\cpsp \in \Omega]}
   $$
   \end{restatable}

The rest of the proof will consist of upper bounding the individual terms $\frac{\Pr[\cps \in \Omega | P(S), \co, \cL]\Pr[\co, P(S) | \cL]}{\Pr[\cpsp \in \Omega]}$. Lemma~\ref{lemweight} is a tool used to prove Lemma~\ref{lemexpandtwo}, which upper bounds the numerator, and Lemma~\ref{lem:lb} lower bounds the denominator. The conclusion of the argument consists of manipulations to upper bound the ratio of these two bounds.

We first analyze the $\Pr[\cps \in \Omega |P(S), \co, \cL]$ term, for $\cP(S) \in S_Q, \co \in \cF$. Conditioned on $Z^{c}$, if $\co \in \cF$, then \textbf{PRSMA} passes the test in step $10$ and outputs a randomly chosen $a_{it}: i \in I_{pass}^{\co}$. Fixing a sampled value $i^* \in I_{pass}^{\co}$, $a_{i^{*}t}$ is distributed identically to $a \sim \cA_{\cO}(S_{i^{*}})| \cA_{\cO}(S_{i^{*}}) \neq \bot$, since after conditioning on $S_{i^{*}}$ each $a_{{i^{*}}t}$ is drawn $iid$ and the event $o_i \neq \bot$ does not depend on the sampled values $a_{it}$. In other words, $\Pr[\cps \in \Omega| {i^{*}}=i, P(S), \co] = \Pr[\cA_\cO(S_{i^{*}})\in \Omega|  \cA_\cO(S_{i^{*}}) \neq \bot]$, and so:

\begin{equation}
\label{decomp}
\Pr[\cps \in \Omega |P(S), \co, \cL] = \frac{1}{|\co|}\displaystyle\sum_{i \in I_{pass}^{\co}}\Pr[\cA_\cO(S_i)\in \Omega| \cA_\cO(S_i) \neq \bot]
\end{equation}

Substituting in \ref{decomp} we have:
$$
\displaystyle\sum_{\cP(S) \in S_Q, \co \in \cF}\Pr[\cps \in \Omega | P(S), \co, \cL]\Pr[\co, P(S) | \cL] = \displaystyle\sum_{\cP(S) \in S_Q, \co \in \cF}\frac{1}{|\co|}\displaystyle\sum_{i \in I_{pass}^{\co}}\Pr[\cA_\cO(S_i)\in \Omega| \cA_\cO(S_i) \neq \bot]\Pr[\co, P(S) | \cL]
$$

We now use the fact that $\cP(S) \in S_Q$, to show that none of the terms $\Pr[\cA_\cO(S_i)\in \Omega| \cA_\cO(S_i) \neq \bot]$ in the right hand side of Equation~\ref{decomp} individually represent a substantial fraction of the total probability mass. Conditioning on $\cL$ ensures that if $o_i \neq \bot$ then $\prob{\cA_\cO(S_i) = \bot} \leq \delta$,
which means $\cA_\cO(S_i)$ is $\delta$-close in total variation distance to $\cA(S_i)$. By Lemma~\ref{mcdiarmid}, with high probability any $\cA(S_i)$ is approximately equally likely to output an element in $\Omega$, which allows us to bound the effect that changing a single data point (and hence a single $S_i$) can have on the probability of outputting an element in $\Omega$.
\begin{restatable}{lemma}{lemweight}
\label{lemweight}
Fix any $\co$, any $\cP(S) \in S_Q$, and index $j \in I_{pass}^{\co}$, i.e. $o_j \neq \bot$. Then:
$$\Pr[\cA_\cO(S_j)\in \Omega| \cA_\cO(S_j) \neq \bot, \cL] \leq \frac{e^2}{(1-\delta)^2}\frac{1}{|\co|-1}\displaystyle\sum_{i \in I_{pass}^{\co}, i \neq j}\Pr[\cA_\cO(S_i)\in \Omega| \cA_\cO(S_i) \neq \bot] + \frac{\delta e^2}{(1-\delta)}$$
\end{restatable}

Without loss of generality (up to renaming of partition elements), assume that the element on which $S$ and $S'$ differ falls into $S_1$. Now we break the summation over $\cO$ into two pieces, depending on whether $o_1 = \bot$ or  $o_1 \neq \bot$. We will use Lemma~\ref{lemweight} to bound terms
involving $\Pr[\cA_\cO(S_1)\in \Omega| \cA_\cO(S_1) \neq \bot, \cL]$, since $S_1$ is the only partition where $S_i \neq S_i'$.
$$
\displaystyle\sum_{P(S) \in S_Q}\bigg(\displaystyle\sum_{\co \in \cF}\big(\frac{1}{|\co|}\displaystyle\sum_{i \in I_{pass}^{\co}}\Pr[\cA_\cO(S_i)\in \Omega| \cA_\cO(S_i) \neq \bot]\big)\Pr[\co| P(S)]\bigg)\Pr[P(S)] =$$
$$
\displaystyle\sum_{P(S) \in S_Q}\bigg(\displaystyle\sum_{\co \in \cF: o_1 \neq \bot}\big(\frac{1}{|\co|}\displaystyle\sum_{i \in I_{pass}^{\co}}\Pr[\cA_\cO(S_i)\in \Omega| \cA_\cO(S_i) \neq \bot]\big)\Pr[\co| P(S)] \;+\; $$
$$\displaystyle\sum_{\co \in \cF: o_1 = \bot}\big(\frac{1}{|\co|}\displaystyle\sum_{i \in I_{pass}^{\co}}\Pr[\cA_\cO(S_i)\in \Omega| \cA_\cO(S_i) \neq \bot]\big)\Pr[\co| P(S)\bigg)\Pr[P(S)]$$

\begin{restatable}{lemma}{lemexpandtwo}
\label{lemexpandtwo}
$$
\displaystyle\sum_{\cP(S) \in S_Q, \co \in \cF}\Pr[\cps \in \Omega | P(S), \co, \cL]\Pr[\co, P(S) | \cL] \leq
$$
$$
(1 + \frac{e^2}{(1-\delta)^2(\frac{1}{\epsilon}-1)}) \displaystyle\sum_{P(S) \in S_Q}\bigg(\displaystyle\sum_{\co \in \cF}\big(\frac{1}{|\co|}\displaystyle\sum_{i \in I_{pass}^{\co}, i \neq 1}\Pr[\cA_\cO(S_i)\in \Omega| \cA_\cO(S_i) \neq \bot]\big)\Pr[\co| P(S)]\bigg)\Pr[P(S)] +  \frac{\epsilon\delta e^2}{1-\delta}
$$
\end{restatable}

\iffalse
Before we proceed with this final step, we jot down a quick identity:

$$\forall j, \; \displaystyle\sum_{o \in \cF}\Pr[\cA_\cO(S_j)\in \Omega|o_j]\Pr[\co| P(S)] = \Pr[\cA_\cO(S_j)\in \Omega|o_j \neq \bot]\Pr[o_j \neq \bot| P(S)],$$

since $\displaystyle\sum_{o \in \cF: o_j = 1}\Pr[\co| P(S)] = \Pr[o_j =1 |P(S)]$, and $\Pr[\cA_\cO(S_j)\in \Omega|o_j = \bot] = 0$, by construction.\\
\fi

We now condition on a fixed partition $P(S')$ in the denominator as well.  Given a fixed partition $P(S)$ of $S$, define the adjacent partition $\cP(S') \sim \cP(S)$ as the partition of an adjacent database $S'$ such that for all $i \neq 1$, $S_i = S_i'$, and $S_1, S_1'$ differ only in $x_1 \neq x_1'$, where $x_1$ is the differing element between $S, S'$. Let $S_Q'$ be the set of $\cP(S')$ adjacent to $\cP(S) \in S_Q$, i.e. $S_Q' = \{ \cP(S'): \exists \; \cP(S) \in S_Q, \cP(S') \sim \cP(S)\}$. We now lower bound the denominator in Lemma~\ref{lemexpand}, which follows by conditioning on $\co, \cP(S')$, and then dropping some (non-negative) terms.
\begin{restatable}{lemma}{lem:lb}
\label{lem:lb}
$$
\Pr[\cpsp \in \Omega] \geq {\displaystyle\sum_{P(S') \in S'_Q}(\displaystyle\sum_{\co \in \cF}(\frac{1}{|\co|}\displaystyle\sum_{i \in I_{pass}^{\co}, i \neq 1}\Pr[\cA_\cO(S_i')\in \Omega| \cA_\cO(S_i') \neq \bot])\Pr[\co| P(S')])\Pr[ P(S')]}
$$
\end{restatable}
Thus by Lemmas~\ref{lemexpandtwo} and \ref{lem:lb},
$$
\frac{\Pr[\cps \in \Omega | P(S), \co, \cL]\Pr[\co, P(S) | \cL]}{\Pr[\cpsp \in \Omega]} \leq
$$
\begin{equation}
\label{doodoomax}
\frac{(1 + \frac{e^2}{(1-\delta)^2(\frac{1}{\epsilon}-1)})\displaystyle\sum_{P(S) \in S_Q}\big(\displaystyle\sum_{\co \in \cF}\big(\frac{1}{|\co|}\displaystyle\sum_{i \in I_{pass}^{\co}, i \neq 1}\Pr[\cA_\cO(S_i)\in \Omega| \cA_\cO(S_i) \neq \bot]\Pr[\co| P(S)]\big)\Pr[ P(S)]\big)}{\displaystyle\sum_{P(S') \in S'_Q}(\displaystyle\sum_{\co \in \cF}(\frac{1}{|\co|}\displaystyle\sum_{i \in I_{pass}^{\co}}\Pr[\cA_\cO(S_i')\in \Omega| \cA_\cO(S_i') \neq \bot])\Pr[\co| P(S')])\Pr[ P(S')]} + \frac{\frac{\epsilon\delta e^2}{1-\delta}}{\Pr[\cpsp \in \Omega]}
\end{equation}

For all $\cP(S), \cP(S'), \Pr[\cP(S)] = \Pr[\cP(S')]$, and we can bound the ratio of the summations over $S_Q, S_Q'$  by the supremum of the ratio. Hence (\ref{doodoomax}) $\leq$
\begin{equation}
\label{doodoo3}
 \sup_{P(S) \sim P(S')} \frac{(1 + \frac{e^2}{(1-\delta)^2(\frac{1}{\epsilon}-1)})\displaystyle\sum_{\co \in \cF}\big(\frac{1}{|\co|}\displaystyle\sum_{i \in I_{pass}^{\co}, i \neq 1}\Pr[\cA_\cO(S_i)\in \Omega| \cA_\cO(S_i) \neq \bot]\Pr[\co| P(S)]\big)}{\displaystyle\sum_{\co \in \cF}(\frac{1}{|\co|}\displaystyle\sum_{i \in I_{pass}^{\co}}\Pr[\cA_\cO(S_i')\in \Omega| \cA_\cO(S_i') \neq \bot]\Pr[\co| P(S')])} + \frac{\frac{\epsilon\delta e^2}{1-\delta}}{\Pr[\cpsp \in \Omega]}
\end{equation}

Now since for all $i \neq 1, S_i = S_i'$, if we could control the ratio $\Pr[\co | \cP(S)]/\Pr[\co|\cP(S')]$ we would be done. But this ratio could potentially be unbounded, as $\Pr[\co_1|P(S)]$ could be nonzero, and the substitution of $x_1'$ for $x_1$ could force failure on the first partition $S_1'$, and so $\Pr[\co_1|P(S')] = 0$.

Given $\co$ let $\co_{-1}$ denote $\{o_2, \ldots o_{K}\}$. The remainder of the argument circumvents this obstacle by decomposing the outer summation over $\co \in \cF$ into a summation over the indicators of all but the first failure event ($\co_{-1})$, and integrating out the probability of the first failure event $(o_1)$ from the joint probability $\Pr[\co|P(S)]$. This trick will be applied in the numerator and denominator, with a slight difference corresponding to an upper and a lower bound respectively. See the end of Section \textbf{C} of the Appendix for details. Following the chain of inequalities, we finally obtain:
$$
\frac{\Pr[\cps \in \Omega | P(S), \co, \cL]\Pr[\co, P(S) | \cL]}{\Pr[\cpsp \in \Omega]} \leq (1 + \frac{e^2}{(1-\delta)^2(\frac{1}{\epsilon}-1)})\frac{1}{1-\epsilon} + \frac{\frac{\epsilon\delta e^2}{1-\delta}}{\Pr[\cpsp \in \Omega]},
$$
which substituting into Lemma~\ref{lemexpand} gives:
$$
   {\Pr[\cps \in \Omega]} \leq (1 + \frac{e^2}{(1-\delta)^2(\frac{1}{\epsilon}-1)})\frac{1}{1-\epsilon}{\Pr[\cpsp \in \Omega]} + 4\delta + \frac{\epsilon\delta e^2}{1-\delta}
$$

For $\epsilon, \delta  \leq 1/2$, $(1 + \frac{e^2}{(1-\delta)^2(\frac{1}{\epsilon}-1)})\frac{1}{1-\epsilon} \leq  e^{8e^2\epsilon + \epsilon + \epsilon^2}$,
which establishes that  \textbf{PRSMA} is $(8e^2\epsilon + \epsilon + \epsilon^2, 4\delta + \frac{\epsilon\delta e^2}{1-\delta})$ differentially private. Setting $\epsilon = \epsilon^*, \delta = \delta^*$ completes the proof.

\end{proof}

We now turn to \textbf{PRSMA}'s accuracy guarantees. Note that when \textbf{PRSMA} starts with an algorithm $\cA_{\cO^*}$ instantiated with a perfect oracle $\cO^*$, it with high probability outputs the result of running $\cA_{\cO^*}$ on a subsampled dataset $S_i$ of size $n/K \approx \epsilon n$, with privacy parameter $\epsilon' = \frac{1}{\sqrt{8\frac{n}{K}\log(2K/\delta)}}$. In general, therefore, the accuracy guarantees of \textbf{PRSMA} depend on how robust the guarantees of $\cA$ are to subsampling, which is typical of ``Subsample and Aggregate" approaches, and also to its specific privacy-accuracy tradeoff. Learning algorithms are robust to sub-sampling however: below we derive an accuracy theorem for \textbf{PRSMA} when instantiated with our oracle-efficient RSPM algorithm.

\begin{restatable}{theorem}{prsmacc}
\label{prsmacc}
Let $\cQ$ a class of statistical queries with a separator set of size $m$. Let $\cA_{\cO^*}$ denote the RSPM algorithm with access to $\cO^*$, a perfect weighted optimization oracle for $\cQ$. Then \textbf{PRSMA} instantiated with $\cA_{\cO^*}$, run on a dataset $S$ of size $n$, with input parameters $\epsilon$ and $\delta$ is an $(\alpha, \beta)$-minimizer
for any $\beta > \delta$ and
$$
\alpha \leq \tilde{O}\left( \frac{m^2 \log\left(\frac{m}{\beta - \delta}\right)
\log(1/\delta) + \sqrt{\log(1/\delta) \log\left(\frac{|\cQ|}{\beta - \delta} \right)}
 }{\sqrt{n \epsilon}}  \right),
$$
% $$
% \alpha \leq \tilde{O}\left(\frac{m^2\log\left({\frac{m}{\beta-\delta}}\right)}{\sqrt{n\epsilon}} + \frac{\sqrt{\log \frac{|\cQ|}{\beta-\delta}}}{\sqrt{n \epsilon}}\right),
% $$
%O(\frac{2m^2(1+\log m)\sqrt{(1+ \log(2/\delta) \log(\frac{1}{\epsilon}(1 + \log(2/\delta)))}}{\sqrt{n\epsilon}} + \frac{\sqrt{1 + \log(2/\delta)(1 + \log 2 |\cQ|)}}{\sqrt{n \epsilon}}. %
where the $\tilde{O}$ hides logarithmic factors in $\frac{1}{\epsilon}, \log(\frac{1}{\delta})$.\end{restatable}
\begin{proof}
  With probability at least $1-\delta$, \textbf{PRSMA} outputs the
  result of RSPM run on an $n/K$ fraction of the dataset, with privacy parameter
  $\epsilon' = \frac{1}{\sqrt{8\frac{n}{K}\log(2K/\delta)}}$. We will
  condition on this event for the remainder of the proof, which occurs
  except with probability $\delta$.

  Let $q^{*}$ denote the true minimizer on $S$. Let $S_K$ denote the
  random subsample, and let $q_K$ denote the true minimizer on
  $S_K$. By Theorem~\ref{thm:rspm}, we know that for any $\eta > 0$,
  with probability $1-\eta$, the error on $S_K$ is bounded as follows:
  $$
  \hat q(S_K)-q_K(S_K) \leq \frac{2m^2 \log(m/\eta)}{\epsilon'
    \frac{n}{K}}
  $$
  We next bound $ {\max_{q \in \cQ}q(S_K)-q(S)}$, the maximum
  difference between the value that any query takes on $S_K$ compared to the value that it takes on $S$.  By
  a Chernoff bound for subsampled random variables (see e.g. Theorem
  1.2 of \cite{ChernoffBound}), for any $q \in \cQ, t > 0,$
  $$\prob{q(S_K)-q(S) \geq t} \leq
  \exp\left({-2\frac{n}{K}t^2}\right).$$ By a union bound over $\cQ$,
  this means that with probability $1-\eta$,
$$ \max_{q \in \cQ}q(S_K)-q(S) \leq \sqrt{\frac{K}{2n}\log\left(\frac{2|\cQ|}{\eta}\right)}$$
\iffalse
By the same trick we used to compute the expected error of RSPM, this implies that
$$
 \Ex{}{\max_{q \in \cQ}q(S_K)-q(S)} \leq \sqrt{\frac{K}{2n}(1+\log2|\cQ|)}
$$
\fi
We now have all the ingredients to complete the bound:
\begin{align*}
  \hat q(S)-q^*(S) &= [\hat q(S)-\hat q(S_K)] +
                     [\hat q(S_K)-q^*(S_K)] + [q^*(S_K)-q^*(S)] \\
                   &\leq 2 \max_{q \in
                     \cQ}|q(S_K)-q(S)| + \hat q(S_K)-q^*(S_K) \\
                   &= 2 \max_{q \in \cQ}|q(S_K)-q(S)| + \hat q(S_K)-q_K(S_K) + q_K(S_K)-q^*(S_K)\\
  &\leq 2 \max_{q \in \cQ}|q(S_K)-q(S)| + \hat q(S_K)-q_K(S_K).
\end{align*}
% $$2 \max_{q \in \cQ}|q(S_K)-q(S)| + q(S_K)-q_K(S_K) + q_K(S_K)-q^*(S_K) \leq 2 \max_{q \in \cQ}|q(S_K)-q(S)| + q(S_K)-q_K(S_K), $$
By a union bound and the results above, we know that the righthand
side is less than
$$
2\sqrt{\frac{K}{2n}\log(\frac{2|\cQ|}{\eta})}\;+ \frac{2m^2
  \log(m/\eta)}{\epsilon' \frac{n}{K}},
$$
 with probability at least
$1-\eta$.  Substituting $\eta = \beta -\delta$,
$\epsilon' = \frac{1}{\sqrt{8\frac{n}{K}\log(2K/\delta)}}$,
$K = O(\frac{1}{\epsilon}(1 + \log(2/\delta)))$ gives the desired
result. % Note that the big $O$ notation in the definition of $K$ hides
% only the relatively civilized constants defining
% $\epsilon^*, \delta^*$ in the statement of Theorem~\ref{thm:prsma}.
\end{proof}
We remark that we can convert this $(\alpha,\beta)$-accuracy bound into a bound on the expected error using the same technique we used to compute the expected error of RSPM. The expected error of \textbf{PRSMA} with the above inputs is $\tilde{O}(\frac{2m^2(\log m + 1)}{\sqrt{n\epsilon}} + \frac{\sqrt{(\log |\cQ| + 1)}}{\sqrt{n \epsilon}})$.

\section{$\OracleQuery$:  {Oracle-Efficient} Private Synthetic Data Generation}
\label{sec:synth}

We now apply the oracle-efficient optimization methods we have
developed to the problem of generating \emph{private synthetic data}.
In particular, given a private dataset $S$ and a query class $\cQ$, we
would like to compute a synthetic dataset $\hat S$ subject to
differential privacy such that the error
$\max_{q\in\cQ}|q(\hat S) - q(S)|$ is bounded by some target parameter
$\alpha$.  We provide a general algorithmic framework called
$\OracleQuery$ for designing oracle-efficient algorithms.  The
crucial property of the query class we rely on to obtain oracle
efficiency is \textit{dual separability}, which requires both the
query class and its dual class have separator sets
(\Cref{def:separator}). Informally, the dual of a query class $\cQ$ is the query class $\cQd$ that results from swapping the role of the functions $q \in \cQ$ and the data elements $x \in \cX$. More formally:

  \begin{definition}[Dual class and dual separability]
 \label{dualsep}
 Fix a class of queries $\cQ$. For every element $x$ in $\cX$, let
 $h_x\colon \cQ\rightarrow \{0, 1\}$ be defined such that
 $h_x(q) = q(x)$. The \emph{dual class} $\cQd$
 of $\cQ$ is the set of all such functions defined by elements in
 $\cX$:
 $$\cQd = \{h_x \mid x\in \cX\}.$$  We say that the class $\cQ$ is
 $(m_1, m_2)$-\textit{dually separable} if there exists a separator
 set of size $m_1$ for $\cQ$, and there exists a separator set of size
 $m_2$ for $\cQd$.
\end{definition}

As we will show (see Appendix \ref{app:separator}), many widely studied query classes, including discrete
halfspaces, conjunctions, disjunctions, and parities are dually
separable, often with $m_1 = m_2 = d$ (in fact, many of these classes are \emph{self-dual}, meaning $\cQ = \cQd$). For any $q\in \cQ$, define its negation $\neg q$ to be
$\neg q(x) = 1 - q(x)$. Let
$\neg \cQ = \{\neg q \mid q\in \cQ\}$ be the \emph{negation of $\cQ$}. It will simplify several aspects of our exposition to deal with classes that are closed under negation. For any class $\cQ$, define
$\overline \cQ = \cQ \cup \neg\cQ$ to be the closure of $\cQ$ under negation. Note that whenever we have a weighted minimization oracle for $\cQ$, we have one for $\neg \cQ$ as well --- simply by negating the weights. Further, if $U$ is a separator set for $\cQ$, it is also a separator set for $\neg \cQ$. This implies that we also have oracle efficient learners for $\overline \cQ$, since we can separately learn over $\cQ$ and $\neg \cQ$, and then privately take the minimum value query that results from the two procedures (using e.g. report-noisy-min \cite{DworkRoth}).

Before we give our algorithm and analysis, we state several consequences of our main theorem (that follow from instantiating it with different oracle-efficient learners).

\begin{theorem}% {oraclequerythm}
\label{oraclequerythm}
Let $\cQ$ be an $(m_1, m_2)$-dually separable query class. Then
given access to a weighted minimization oracle $\cO$ over the class
$\cQd$ and a  differentially private weighted minimization algorithm 
$\cO_{\epsilon_0, \delta_0}$ for the class $\cQ$ (with appropriately
chosen privacy parameters $\epsilon_0$ and $\delta_0$), the algorithm
$\OracleQuery$ is oracle-efficient,
$(\epsilon, \delta)$-differentially private, and
$(\alpha, \beta)$-accurate with $\alpha$ depending on the
instantiation of $\cO_{\epsilon_0, \delta_0}$. If $\cO_{\epsilon_0, \delta_0}$ is robustly differentially private, then so is $\OracleQuery$.
\begin{enumerate}

\item If $\cO_{\epsilon_0, \delta_0}$ is instantiated with the Gaussian
  RSPM algorithm, then
  \[
    \alpha \leq \tilde O \left( \frac{m_1^{3/2} m_2^{3/4}
        \sqrt{\log(m_1/\beta)\log|\cX|}\log(1/\delta)}{n \epsilon}\right)^{1/2}
    % \alpha \leq \tilde O \left( \frac{m_1^{3/2} m_2^{3/4}
    %     \sqrt{\log|\cX|}}{n \epsilon}\right)^{1/2}
  \]
  In this case, $\OracleQuery$ is oracle equivalent to a differentially private algorithm, but is not robustly differentially private.

\item If $\cO_{\epsilon_0, \delta_0}$ is instantiated with the \textbf{PRSMA}
  algorithm (using the Laplace RSPM as $\cA_{\cO^*}$),
  then
  \[
    \alpha \leq \tilde O \left(\frac{(m_1^{4/3} + \log^{1/3}(|\cQ|))
        m_2^{1/4} \log^{1/6}(|\cX|)}{(n\eps)^{1/3}}\right)\cdot \polylog\left(\frac{1}{\beta - \delta} \right)
  \]
as long as $\beta > \delta$.  In this case, $\OracleQuery$ is robustly differentially private.

  \item If $\cO_{\epsilon_0, \delta_0}$ is an
  $(\alpha_0, \beta_0)$-accurate differentially private oracle with
  $\alpha_0 = O\left(\log(|\cQ|/(\epsilon_0 n)) \right)$, then
  \[
\alpha \leq   \tilde O \left( \frac{m_2^{3/4} \sqrt{\log |\cX|
          \log(1/\delta)} \log(|\cQ|/\beta)}{n \epsilon} \right)^{1/2}
  \]
 In this case, $\OracleQuery$ is robustly differentially private.
\end{enumerate}
where the $\tilde O$ hides logarithmic factors in
$\frac{1}{\delta}, \frac{1}{\beta},m_1, m_2, n$ and $\log(|\cX|)$.
\end{theorem}

A couple of remarks are in order.

\begin{remark}
The first two bounds quoted in Theorem \ref{oraclequerythm} result from plugging in constructions of oracle-efficient differentially private learners that we gave in Section \ref{rspm}. These constructions start with a \emph{non-private} optimization oracle. The third bound quoted in Theorem \ref{oraclequerythm} assumes the existence of a \emph{differentially private} oracle with error bounds comparable to the (inefficient) exponential mechanism based learner of \cite{KLNRS08}. We don't know if such oracles can be constructed from non-private (exact) optimization oracles. But this bound is analgous to the bounds given in the non-private oracle-efficient learning literature. This literature gives constructions assuming the existence of perfect learning oracles, but in practice, these oracles are instantiated with heuristics like regression or support vector machines, which exactly optimize some convex surrogate loss function. This is often reasonable, because although these heuristics don't have strong worst-case guarantees, they often perform very well in practice. The same exercise makes sense for private problems: we can use a differentially private convex minimization algorithm to optimize a surrogate loss function (e.g. \cite{CMS11,BST14}), and hope that it does a good job minimizing classification error in practice. It no longer makes sense to assume that the heuristic exactly solves the learning problem (since this is impossible subject to differential privacy) --- instead, the analogous assumption is that it does as well as the best inefficient private learner.
\end{remark}

\begin{remark}
It is useful to compare the bounds we obtain to the best bounds that can be obtained with inefficient algorithms. To be concrete, consider the class of boolean conjunctions defined over the boolean hypercube $\cX = \{0,1\}^d$ (see Appendix \ref{app:separator}), which are dually-separable with $m_1 = m_2 = d$. The best (inefficient) bounds for constructing synthetic data useful for conjunctions \cite{PMW,GRU12} obtain error: $\alpha = O\left(\frac{\sqrt{\log |\cQ|}(\log|\cX|)^{1/4}}{\sqrt{\epsilon n}}\right)$. In the case of boolean conjunctions, $\log |\cX| = \log |\cQ| = d$, and so this bound becomes: $\alpha = O\left(\frac{d^{3/4}}{\sqrt{\epsilon n}}\right)$. In contrast, the three oracle efficient bounds given in Theorem \ref{oraclequerythm}, when instantiated for boolean conjunctions are:
\begin{enumerate}
\item $\alpha = O\left(\frac{d^{11/8}}{\sqrt{\epsilon n}}\right)$,
\item $\alpha = O\left(\frac{d^{7/4}}{(\epsilon n)^{1/3}}\right)$, and
\item $\alpha = O\left(\frac{d^{9/8}}{\sqrt{\epsilon n}}\right)$
\end{enumerate}
respectively.
Therefore the costs in terms of error that we pay, in exchange for oracle efficiency are $d^{5/8}$, $\frac{d}{(\epsilon n)^{1/6}}$, and $d^{3/8}$ respectively.
\end{remark}

We now give a brief overview of our construction before diving into the technical details.
\paragraph{Proof overview:}{ We  present our solution in three
  main steps.
\begin{enumerate}

\item We first revisit the formulation by \cite{Hsu13} that views the synthetic data generation problem as a zero-sum game between a \emph{Data player} and a \emph{Query player}. We  leverage the
    fact that at any approximate equilibrium, the data player's mixed strategy (over $\cX$) represents a good synthetic dataset $S'$ with respect to $\cQ$.

\item Using the seminal result of~\cite{FS96}, we will compute the equilibrium for the zero-sum game by simulating \emph{no-regret dynamics} between the two players: in rounds, the Data player plays according to an oracle-efficient online learning algorithm due to \cite{oracle16}, and the Query player best responds to the Data player by using a differentially private oracle efficient optimization algorithm. At the end of the dynamics, the average play of the Data player is an approximate minimax strategy for the game, and hence a good synthetic dataset.

\item We instantiate the private best response procedure of the Query player using different oracle-efficient methods, which we have derived in this paper, each of which gives different accuracy guarantees. Finally, we apply our result to several query classes of interest.

\end{enumerate}

}

\subsection{The Query Release Game}
The query release game defined in \cite{Hsu13} involves a \emph{Data player} and \emph{Query
  player}. The data player has action set equal to the data universe
$\cX$ (or equivalently the dual class $\cQd$), while the query player
has action set equal to the query class $\overline\cQ$. Given a pair
of actions $x \in \cX$ and $q\in \overline\cQ$, the payoff is defined
to be:
\[
  A(x, q) = q(S) - q(x),
\]
where $S$ is the input private dataset. In the zero-sum game, the Data
player will try minimize the payoff and the Query player will try to
maximize the payoff. To play the game, each player chooses a
\emph{mixed strategy}, which is defined by a probability distribution
over their action set.  Let $\Delta(\cX)$ and $\Delta(\overline \cQ)$
denote the sets of \emph{mixed strategies} of the Data player and
Query player respectively. To simplify notation, we will write
$A(\hat S, \cdot) = \Ex{x\sim \hat S}{A(x, \cdot)}$ and
$A(\cdot, W) = \Ex{q\sim W}{A(\cdot, q)}$ for any
$\hat S\in \Delta(\cX)$ and $W\in \Delta(\overline \cQ)$.  By von
Neumann's minimax theorem, there exists a value $V$ such that
\[
  V = \min_{\hat S\in \Delta(\cX)}\max_{q \in \overline \cQ}
  A(\hat S, W) = \max_{W \in \Delta(\overline \cQ)} \min_{x\in
    \cX} A(x, W)
\]

If both players are playing strategies that can guarantee a payoff
value close to $V$, then we say that the pair of strategies form an
approximate equilibrium.

\begin{definition}[Approximate Equilibrium]
  For any $\alpha > 0$, a pair of strategies $\hat S\in \Delta(\cX)$
  and $W \in \Delta(\overline\cQ)$ form an $\alpha$-\emph{approximate
    minimax equilibrium} if
  \[
    \max_{q\in \overline \cQ} A(\hat S, q) \leq V + \alpha \qquad
    \mbox{and} \qquad \min_{x\in \cX} A(W, x) \geq V - \alpha.
  \]
\end{definition}

Hsu et al. \cite{Hsu13} show that the query release game has value $V=0$ and that at any approximate equilibrium, the mixed strategy of the
Data player provides accurate answers for all queries in $\overline \cQ$.

\begin{lemma}[Accuracy at equilibrium~\cite{Hsu13}]
  Let $(\hat S, W)$ be an $\alpha$-approximate equilibrium of the
  query release game. Then for any $q\in \overline \cQ$,
  $|q(S) - q(\hat S)| \leq \alpha$.
\end{lemma}

Therefore, the dataset represented by the distribution $\hat S$ (or that could be obtained by sampling from $\hat S$) is
exactly the synthetic dataset we would like to compute, and hence the
problem of privately computing synthetic data is reduced to the problem of differentially private
equilibrium computation in the query release game.

\subsection{Solving the Game with No-Regret Dynamics}
To privately compute an approximate equilibrium of the game, we will
simulate the following \emph{no-regret dynamics} between the Data
Player and the Query Player in rounds: In each round $t$, the Data
player plays a distribution $S^t$ according to a \emph{no-regret}
learning algorithm, and the Query player plays an approximate
best-response to $S^t$. The following classical theorem of Freund and Schapire \cite{FS96} (instantiated in our setting)
shows that the average play of both players in this dynamic forms an
approximate equilibrium.

\begin{theorem}[\cite{FS96}]\label{fs}
  Let $S^1, S^2 , \ldots , S^T\in \Delta(\cX)$ be a sequence of
  distributions played by the Data Player, and let
  $q^1, q^2 , \ldots , q^T \in \overline\cQ$ be the Query player's
  sequence of approximate best-responses against these
  distributions. Suppose that the \emph{regret} of the two players satisfy:
  \begin{align*}
    \Reg_D(T) = \sum_{t= 1}^T A(S^t, q^t) - \min_{x \in \cX}\displaystyle\sum_{t=1}^T A(x, q^t)\leq \gamma_D T \\
    \Reg_Q(T) = \max_{q \in \overline\cQ} \sum_{t= 1}^T A(S^t , q) - \sum_{t=1}^T A(S^t, q^t) \leq \gamma_Q T.
  \end{align*}
  Let $\overline S$ be uniform mixture of the distributions
  $\{S^1, \ldots, S^T\}$ and $\overline W$ be the uniform distribution
  over $\{q^1, \ldots ,q^T\}$. Then $(\overline S, \overline W)$ is a
  $(\gamma_D + \gamma_Q)$-approximate minimax equilibrium of the game.
\end{theorem}

\iffalse
Query Player's best-response will be a function of the data $S$, and
so will need to be computed subject to differential privacy.  By the
result of \cite{FS96}, the average plays of both players over time
converge to an approximate equilibrium of the game, as long as the
Data Player has low regret \ar{Have we defined regret?}, and the Query
Player's private best-response is sufficiently accurate.

\ar{Somewhere, maybe earlier, we need to write down the connection stating how accurate the synthetic data corresponding to a $(\gamma_L + \gamma_A)$-approximate equilibrium of the game is...}
We now discuss and state regret bounds for the algorithms used by the Data Player and the Query Player. \\
\\
\fi

Now we will detail the no-regret algorithm for the Data player and the
best-response method for the Query player, and provide the regret
bounds $\gamma_D$ and $\gamma_Q$.

\paragraph{No-Regret Algorithm for the Data Player.}
We start with the observation that the regret of the Data player is
independent of the private data $S$, because
\[
  \Reg_D(T) = \sum_{t= 1}^T \left(q^t(S) - q^t(S^t) \right) - \min_{x
    \in \cX}\displaystyle\sum_{t=1}^T \left(q^t(S) - q^t(x) \right) =
  \sum_{t= 1}^T \neg q^t(S^t) - \min_{x \in
    \cX}\displaystyle\sum_{t=1}^T \neg q^t(x).
\]
Therefore, it suffices to minimize regret with respect to the sequence
of loss functions $\{\neg q^t(\cdot)\}$, while ignoring the private dataset $S$. We crucially rely on the fact
that each $q^t$ is computed by the Query player subject to
differential privacy, and so the Data player's learning algorithm need not be differentially private: differential privacy for the overall procedure will follow from the post-processing guarantee of differential privacy. In particular, we will run
an oracle-efficient algorithm \CFTPL ~due to \cite{oracle16}, which is
a variant of the ``Follow-the-Perturbed-Leader'' of \cite{KV05} algorithm that
performs perturbations using a separator set. We state its
regret guarantee below. Because \CFTPL need not be differentially private, it can be instantiated with an arbitrary heuristic oracle, that need not be either differentially private or certifiable.

\begin{algorithm}
\label{ftpl}
\textbf{$\CFTPL(\cQd, \mu)$ Algorithm \cite{oracle16}}\\
\textbf{Given}: parameter $\mu$, hypothesis class $\cQd$ (or
equivalently $\cX$), separator set $U \subset \cQ$ for $\cQd$,
weighted optimization oracle $\cO$ for $\cQd$ \\
\textbf{Input:} A sequence of queries $\{q^1, \ldots , q^T\}$ selected by the Query player.

\begin{algorithmic}[1]
% \STATE initialize $Q_0 = \{\}$
\FOR{ $t = 1 \ldots T$}
        \STATE Data player plays the distribution $S^t$ such that each draw $x$ generated as follows:
	\INDSTATE  Draw a sequence $(s, \eta_s)$, for $s \in U$, where $\eta_s \sim \text{Lap}(\mu)$
	\INDSTATE  Let $x = \argmin_{x \in \cX}\sum_{\tau = 1}^{t-1} h_x(\neg q^\tau) + \sum_{s \in S} \eta_s h_x(s)\qquad$ \COMMENT{Use non-private oracle $\cO$}
	% \STATE  Query Player plays $q_t$
	% \STATE Update: $Q_t = Q_{t-1} \cup q_t$
\ENDFOR
\end{algorithmic}
\end{algorithm}

\begin{theorem}[Follows from~\cite{oracle16}]
\label{regretdata}
Suppose that $U$ is a separator set for $\cQd$ of cardinality
$m_2$. Then the Data player running \CFTPL$(\cX, \mu)$ with
appropriately chosen $\mu$ has regret:
$${\Reg_D(T)} \leq  O(m_2^{3/4}\sqrt{T \log |\cX|})$$
\end{theorem}

Note that the algorithm \CFTPL only provides sample access to each
distribution $S^t$, but each draw from $S^t$ can be computed using a single
call to the oracle $\cO$.

\iffalse
CONTEXT-FTPL$(\cX, \mu)$ is oracle-efficient in the sense that it can be implemented (e.g. we can draw $x_t$ in line $4$ above) by making a single call to an optimization oracle $\cO$ for hyperplanes
on the weighted dataset $WD = \{(q,-1)\}_{q \in Q_{t-1}} \cup \{ (s, \eta_s)\}_{s \in S}$. \ar{See comment in algorithm --- why weight queries with $-1$?}\sn{doesn't this propagate from the -1 in the defintion of the fake data game}
\fi

\paragraph{Approximate Best  Response by the Query Player.}

At each round $t$, after the Data player chooses $S^t$, the Query
player needs to approximately solve the following best-response problem:
\[
  \argmax_{q\in \overline \cQ} A(S^t, q) = \argmax_{q\in \overline
    \cQ} \left(q(S) - q(S^t)\right)
\]
Unlike the problem faced by the Data player, this optimization problem
directly depends on the private data $S$, so the best response needs
to be computed privately. Since we only have sample access to the
distribution $S^t$, the Query player will first draw $N$ random
examples from the distribution $S^t$, and we will the empirical
distribution $\hat S^t$ over the sample as a proxy for $S^t$. Recall that
$\overline\cQ = \cQ \bigcup \neg\cQ$, so we will first approximately and differentially privately solve both of the
following two problems separately:
\begin{equation}\label{syrup}
  \argmax_{q\in \cQ} \left(q(S) - q(\hat S^t)\right) \qquad \mbox{and}
  \qquad \argmax_{q\in \;\neg\cQ} \left(q(S) - q(\hat S^t)\right)
\end{equation}
Note that the two problems are equivalent to the following problems
respectively:
\begin{equation}\label{canada}
  \argmin_{q\in \cQ} \left(q(\hat S^t) - q(S) \right) \qquad
  \mbox{and} \qquad \argmin_{q\in \cQ} \left(q(S) - q(\hat S^t)\right),
\end{equation}
both of which are weighted optimization problems:
\begin{equation}\label{maple}
  \argmin_{q\in \cQ} \frac{1}{n}\sum_{x_i\in S} w_i q(x_i) +
  \frac{1}{N}\sum_{x'_j \in \hat S^t} w'_j q(x'_j)
\end{equation}
with weights $w_i, w'_j$ taking values in $\{1, -1\}$. We will rely on
a private weighted optimization algorithm
$\cO_{\epsilon_0, \delta_0}$ to compute two solutions $q_1^t$ and
$q_2^t$ for the two problems in \Cref{canada} respectively. Finally,
the Query player privately selects one of the queries using
\emph{report noisy max}---i.e. it first perturb the values of
$A(\hat S^t, q^t_1)$ and $A(\hat S^t, q^t_2)$ with Laplace noise, and
then select the query with higher noisy value. By bounding the errors
from the sampling of $\hat S^t$, the private optimization oracle
$\cO_{\epsilon_0, \delta_0}$, and report noisy max, we can derive
the following regret guarantee for the Query player.

\begin{algorithm}[h]
\label{bestresponse}
\textbf{Private Best-Response (PBR)}\\
\textbf{Given}: privacy parameters $(\epsilon_0, \delta_0)$, accuracy
parameters $(\alpha_0, \beta_0)$, a private
weighted optimization algorithm $\cO_{\epsilon_0, \delta_0}$.\\
\textbf{Input:} A private dataset $S$ and the Data
player's sequence of distributions $\{S^1, \ldots , S^T\}$.
\begin{algorithmic}[1]
\FOR{ $t = 1 \ldots T$}
\STATE Query player plays a query $q^t$ as follows:

	\INDSTATE  Draw $N$ samples $x'_1, \ldots x'_N \sim_{i.i.d.} S^t$ with  $N = \frac{2\log(2|\cQ|/\beta_0)}{\alpha_0^2}$
	\INDSTATE  Form the weighted dataset $WD^1 = \{(x_i, \frac{1}{n})\}_{x_i \in S} \cup \{x_j, \frac{-1}{N}\}_{j = 1 \ldots N}$
	\INDSTATE  Form the weighted dataset $WD^2 = \{(x_i, \frac{-1}{n})\}_{x_i \in S} \cup \{x_j, \frac{1}{N}\}_{j = 1 \ldots N}$
	\INDSTATE  Let $q_1^t =  \cO_{\epsilon_0, \delta_0}(WD^1)$ and $q_2^t =  \neg \left(\cO_{\epsilon_0, \delta_0}(WD^2)\right)$

        \INDSTATE Perturb payoffs:
        $\tilde A_1 = A(\hat S^t, q_1^t) + \Lap{1/(\epsilon_0 n)}$ and
        $\tilde A_2 = A(\hat S^t, q_2^t) + \Lap{1/(\epsilon_0 n)}$
        \INDSTATE \textbf{If} $\tilde A_1 > \tilde A_2$ \textbf{then} $q^t = q_1^t$
        \textbf{else} $q^t = q^t_2$

\ENDFOR
\end{algorithmic}
\end{algorithm}

\begin{lemma}
  Suppose that the oracle $\cO_{\epsilon_0,\delta_0}$ succeeds in solving all problems it is presented with up to error at most $\alpha_0$ except with
  probability $\beta_0$. Then with probability at least
  $1 - 3\beta_0 T$, the Query player has regret:
  \[
    \Reg_Q(T) \leq T \, O\left( \alpha_0 + \frac{\log(1/\beta_0)}{n\epsilon_0} \right).
  \]
  % as long as $\alpha_0 \geq \log(1/\beta_0)/(n\epsilon_0)$.
\end{lemma}

\begin{proof}
  There are three potential sources of error at each round $t$. The first is the error introduced by solving our optimization problem over the proxy distribution $\hat S^t$ instead of $S^t$.  By
  applying a Chernoff bound and a union bound over all queries in $\cQ$,
  we have with probability $1 - \beta_0$ that,
  \begin{equation}\label{oj}
    \mbox{for all } q\in \overline\cQ,\qquad \left| q(S^t) - q(\hat
      S^t) \right| \leq \sqrt{\frac{2\log(2|\cQ|/\beta_0)}{N}} \leq
    \alpha_0.
  \end{equation}
  % This implies that the output query $q^t$ satisfies
  % \begin{equation}
  %   \left| q(S^t) - q^t(\hat
  %     S^t) \right|  \leq
  %   \alpha_0.    \label{oj}
  % \end{equation}
  Next is the error introduced by the oracle. By our assumption on the oracle $\cO_{\epsilon_0, \delta_0}$,
  we have except with probability $2\beta_0$ that
  \[
    \left(q^t_1(S) - q^t_1(\hat S^t)\right) \geq  \max_{q\in \cQ} \left(q(S)
      - q(\hat S^t)\right) - \alpha_0 \qquad \mbox{and} \qquad \left(q^t_2(S) -
      q^t_2(\hat S^t)\right) \geq  \max_{q\in \;\neg\cQ} \left(q(S) - q(\hat
      S^t)\right) - \alpha_0
  \]

  The two inequalities together imply that
  \begin{equation}\label{simpson}
    \max_{q \in \{q_1^t, q_2^t\}} \left(q(S) - q(\hat S^t)\right) \geq
    \max_{q\in \;\neg\cQ} \left(q(S) - q(\hat S^t)\right) - \alpha_0
  \end{equation}
  Finally, the Laplace noise used to privately select the best query
  amongst $q_1^t$ and $q_2^t$ introduces additional error. But by the
  accuracy guarantee of report noisy max \cite{DworkRoth} (which
  follows from the CDF of the Laplace distribution and a union bound
  over two samples from it) we know that with probability
  $1 - \beta_0$,
  \begin{equation}\label{run}
    \left(q^t(S) - q^t(\hat S^t)\right) \geq \max_{q \in \{q_1^t,
      q_2^t\}} \left(q(S) - q(\hat S^t)\right) -
    \frac{2\log(2/\beta_0)}{n\epsilon_0} % \geq \max_{q \in \{q_1^t,
      % q_2^t\}} \left(q(S) - q(\hat S^t)\right) - 4\alpha_0
  \end{equation}
  Combining \Cref{oj,simpson,run} and applying a union bound, we
  have the following per-round guarantee: except with probability
  $3\beta_0$,
  \[
    \left(q^t(S) - q^t(S^t)\right) \geq \max_{q \in \overline \cQ}
    \left(q(S) - q(S^t)\right) - O\left(\alpha_0+
      \frac{\log(1/\beta_0)}{n\epsilon_0} \right).
  \]
  Finally, taking a union bound over all $T$ steps recovers the
  stated regret bound.
\end{proof}

\subsection{The Full Algorithm: $\OracleQuery$}

Our main algorithm $\OracleQuery$ first simulates the no-regret
dynamics described above, and then constructs a synthetic dataset from
the average distribution
$\overline S = \frac{1}{T}\sum_{t\in [T]} S^t$ played by the Data
player. Since we only have sampling access to each $S^t$, we will
approximate $\overline S$ by the empirical distribution of a set of
independent samples drawn from $\overline S$. As we show below, the
sampling error will be on the same order as the regret as long as we
take roughly $\log|\cQ|/\alpha^2$ samples.

\begin{lemma}\label{samplingerror}
  Suppose that $\overline S$ is an $\eta$-approximate minimax strategy
  for the query release game. Let $\{x'_1, \ldots , x'_{N}\}$ be a set
  of $N = \frac{2\log(2|\cQ|/\beta)}{\alpha_0^2}$ samples drawn
  i.i.d. from $\overline S$, and $\hat S$ be the empirical
  distributions over the drawn samples. Then with probability
  $1 - \beta$, $\hat S$ is an $(\eta + \alpha_0)$-approximate minimax
  strategy.% , as long as
  % $$N_\alpha \geq \frac{2\log(2|\cQ|/\beta)}{\alpha^2}$$
\end{lemma}

\begin{proof}
  By the definition of an $\eta$-approximate minimax strategy, we have
  \[
    \max_{q\in \overline \cQ} A(\overline S, q) \leq V + \eta = \eta.
  \]
  By applying the Chernoff bound, we know that except with probability
  $\beta/|\cQ|$, the following holds for each $q\in \cQ$:
  \[
    \left| A(\overline S, q) - A(\hat S, q) \right| \leq
    \sqrt{\frac{2\log(2|\cQ|/\beta)}{N}} = \alpha
  \]
  Note that this implies
  $ \left| A(\overline S, \neg q) - A(\hat S,\neg q) \right| \leq
  \alpha$ as well. Then by taking a union bound over $\cQ$, we know
  that $ \left| A(\overline S, q) - A(\hat S, q) \right| \leq \alpha$
  holds for all $q\in \overline \cQ$ with probability at least
  $1-\beta$. It follows that
  \[
    \max_{q\in \overline \cQ} A(\hat S, q) \leq \max_{q\in \overline
      \cQ} A(\overline S, q) + \alpha_0 \leq \eta + \alpha_0,
  \]
  which recovers the stated bound.
\end{proof}

\begin{algorithm}[h]
\label{oraclequery}

\textbf{Given}:  Target privacy parameters
$\eps, \delta\in (0, 1)$, a target failure probability $\beta$, a number
of rounds $T$, accuracy parameters $\alpha_0, \beta_0$, a weighted
optimization oracle $\cO$ for the class $\cQd$, a 
$(\epsilon_0,\delta_0)$-differentially private
$(\alpha_0, \beta_0)$-accurate minimization oracle
$\cO_{\epsilon_0,\delta_0}$ for class $\cQ$ with parameters
that satisfy
  \[
    \epsilon_0 = \frac{\eps}{\sqrt{24T \ln(2/\delta)}},\qquad \delta_0
    \leq \frac{\delta}{4T}, \qquad \beta_0 = \frac{\beta}{4T}
  \] 
\textbf{Input:} A dataset $S \in \cX^n$.
% noise level $t$, separator set $S$ for $\mathcal{X}$, $v$
\begin{algorithmic}[1]
  \STATE Initialize $q^0 \in \overline \cQ$ to be an arbitrary query
  \STATE Let $N_{\alpha_0} = \frac{2\log(8|\cQ|/\beta)}{\alpha_0^2}$
  \FOR{ $t = 1 \ldots T$}

  \STATE Let $S^t$ be a distribution defined by the sampling algorithm
  $\CFTPL(\cQd, \cO, \{q^0,\ldots, q^{(t-1)}\})$\ \ \  \COMMENT{Data player's no-regret algorithm}

  \STATE Let
  $q^t = \text{PBR}(\eps_0, \delta_0, \alpha_0, \beta_0, S,
  \cO_{\epsilon_0, \delta_0}, S^t)$ \ \ \  \COMMENT{Query player's best response}
\ENDFOR
\FOR{$j = 1, \ldots , N_{\alpha_0}$:}
\STATE Draw $\tau$ from Unif$([T])$ and then draw $x_j'$ from distribution $S^\tau$
\ENDFOR

\textbf{Output}: the dataset
$\hat S = \{x_1', \ldots, x'_{N_{\alpha_0}}\}$
\end{algorithmic}
  \caption{\textbf{Oracle-Efficient Synthetic Data Release: $\OracleQuery$}}
\label{halloween}
\end{algorithm}

The details of the algorithm are presented in \Cref{halloween}. To
analyze the algorithm, we will start with establishing its privacy
guarantee, which directly follows from the advanced composition
of~\cite{DRV10}, the fact that each call to PBR by the Query player
satisfies $(3\eps_0, 2\delta_0)$-differential privacy (with
$\epsilon_0$ and $\delta_0$ set according to \Cref{halloween}), and
the fact that the rest of the algorithm can be viewed as a
post-processing of these calls.

\begin{lemma}[Privacy of $\OracleQuery$]\label{per-round}
  $\OracleQuery$ is oracle equivalent to an
  $(\epsilon, \delta)$-differentially private algorithm. If
  $\cO_{\epsilon_0,\delta_0}$ is robustly differentially private, then
  $\OracleQuery$ is $(\epsilon,\delta)$-robustly differentially
  private.
\end{lemma}

Now to analyze the accuracy of $\OracleQuery$, we will show that the
average distribution $\overline S$ is part of an approximate minimax equilibrium, and
so is its approximation $\hat S$.

\begin{lemma}[Accuracy of $\OracleQuery$]\label{acc}
  Suppose that $\cO_{\epsilon_0, \delta_0}$ is a weighted
  $(\epsilon_0, \delta_0)$-differentially private
  $(\alpha_0, \beta_0)$ minimization oracle over the class $\cQ$,
  where the parameters $\epsilon_0$, $\delta_0$, and $\beta_0$ are set
  according to \Cref{halloween}. Then $\OracleQuery$ is an
  $(\alpha, \beta)$-accurate synthetic data generation algorithm for
  $\cQ$ with
  \[
    \alpha \leq O\left(\alpha_0 + m_2^{3/4}
      \sqrt{\frac{\log(|\cX|)}{T}} +
      \frac{\log(1/\beta_0)}{n\epsilon_0}\right)
  \]
\end{lemma}

\begin{proof}
  First, we will show that the average distribution $\overline S$ from
  the no-regret dynamics is an $\eta$-approximate minimax strategy,
  with
  \[
    \eta \leq O\left(\alpha_0 + m_2^{3/4} \sqrt{\frac{\log(|\cX|)}{T}}
      + \frac{\log(1/\beta_0)}{n\epsilon_0}\right).
  \]
  Recall that the average regret for the two players is bounded by
  \[
    \Reg_D(T)/T \leq  O(m_2^{3/4}\sqrt{ \log |\cX|/T}) \qquad \mbox{and,} \qquad
    \Reg_Q(T)/T \leq  O\left( \alpha_0 + \frac{\log(1/\beta_0)}{n\epsilon_0} \right),
  \]
  with probability at least $1 - 3T\beta_0$. Then by \Cref{fs}, we
  know that $\overline S$ is an $\eta$-approximate minimax strategy,
  with probability at least $1 - 3\beta/4$. Let us condition on this
  event. Lastly, by the setting of $N_{\alpha_0}$ in \Cref{halloween}
  and \Cref{samplingerror}, we know that $\hat S$ is a
  $(\eta + \alpha_0)$, except with probability $\beta/4$. Then the
  stated bound follows directly from a union bound.
\end{proof}

Finally, we will consider three different instantiations of
$\cO_{\epsilon_0, \delta_0}$. To optimize the error guarantee for each
instantiation, we set the number of rounds $T$ used in
\Cref{halloween} so that the regret of the Data player given by
\Cref{regretdata} is on the same order as the error of
$\cO_{\epsilon_0, \delta_0}$. We first consider a differentially
private oracle that matches the error guarantees of the generic
private learner from \cite{KLNRS08}.

\begin{corollary}
  Suppose that $\cO_{\epsilon_0, 0}$ is an
  $(\epsilon_0, 0)$-differentially private
  $(\alpha_0, \beta_0)$-accurate weighted minimization oracle, where
  $\alpha_0, \beta_0$ are such that
  $\alpha_0 \leq O\left(\frac{\log(|\cQ|/\beta_0)}{n\epsilon_0}
  \right)$.  Then $\OracleQuery$ with
  $T = \left\lceil \frac{n \epsilon m_2^{3/4}
      \sqrt{\log(|\cX|)}}{\log(|\cQ|/\beta) \sqrt{\log(1/\delta)}}
  \right\rceil$ is an $(\alpha, \beta)$-accurate synthetic data
  generation algorithm for $\cQ$ with
  \[
    \alpha \leq \tilde O \left( \frac{m_2^{3/4} \sqrt{\log |\cX|
          \log(1/\delta)} \log(|\cQ|/\beta)}{n \epsilon} \right)^{1/2}
  \]
  where the $\tilde O$ hides logarithmic factors in $m_2, n$ and
  $\log(|\cX|)$.
\end{corollary}

Next, we will instantiate $\cO_{\epsilon_0, \delta_0}$ with the RSPM
algorithm. Note that with this choice, although $\OracleQuery$ is oracle equivalent to a differentially private algorithm, it is not robustly differentially private.

\begin{corollary}
  When $\cO_{\epsilon_0, \delta_0}$ is instantiated with the Gaussian RSPM algorithm, $\OracleQuery$ with
  $T = \left\lceil \frac{m_2^{3/4} \sqrt{\log(|\cX|)}
      n\epsilon}{m_1^{3/2} \sqrt{\log(m_1/\beta)} \ln(1/\delta)}
  \right\rceil$ is an $(\alpha, \beta)$-accurate synthetic data
  generation algorithm for $\cQ$ with
  \[
    \alpha \leq \tilde O \left( \frac{m_1^{3/2} m_2^{3/4}
        \sqrt{\log(m_1/\beta)\log|\cX|}\log(1/\delta)}{n \epsilon}\right)^{1/2}
    % \alpha \leq \tilde O \left( \frac{m_1^{3/2} m_2^{3/4} \sqrt{\log|\cX|}}{n \epsilon}\right)^{1/2}
  \]
  where the $\tilde O$ hides logarithmic factors in $m_2, n$ and
  $\log(|\cX|)$.
\end{corollary}

Finally, we instantiate
the oracle $\cO_{\epsilon_0, \delta_0}$ with the \textbf{PRSMA} algorithm that
uses the Laplace RSPM algorithm with a certifiable heuristic oracle.

\begin{corollary}
  When $\cO_{\epsilon_0, \delta_0}$ is instantiated with PRSMA
  algorithm (that internally uses Laplace RSPM as $\cA_{\cO^*}$), then
  for any $\beta > \delta$, $\OracleQuery$ with
  $T = \left\lceil \left(\frac{m_2^{3/4} \sqrt{\log(|\cX|) n
          \epsilon}}{m_1^2 +
        \sqrt{\log(|\cQ|)}}\right)^{4/3}\right\rceil$
  % $$T = \left\lceil \left(\frac{m_2^{3/4} \sqrt{\log(|\cX|) n
  %         \epsilon}}{m_1^2 \log(m_1/(\beta - \delta))
  %       \log^{5/4}(1/\delta) +
  %       \sqrt{\log^{3/2}(1/\delta)\log(|\cQ|/(\beta -
  %         \delta))}}\right)^{4/3}\right\rceil$$
  is an
  $(\alpha, \beta)$-accurate synthetic data generation algorithm for
  $\cQ$ with
 \[
   \alpha \leq O \left(\frac{ m_2^{1/4} \log^{1/6}|\cX|\left(
         m_1^{4/3} + \log^{1/3}(|\cQ|) \right) }{(n\eps)^{1/3}}\right)
   \polylog\left(m_1 1/\delta, 1/\epsilon, \frac{1}{\beta - \delta}
   \right).
  \]
  % where the $\tilde O$ hides logarithmic factors in
  % $\frac{1}{\delta}, \frac{1}{\beta},m_1, m_2, n$ and
  % $\log(|\cX|)$.
\end{corollary}

\subsection{Example: Conjunctions}
Here, we instantiate our bounds for a particular query class of interest: boolean conjunctions over the hypercube $\{0,1\}^d$. Constructing synthetic data for conjunctions (or equivalently producing a full \emph{marginal table} for a dataset) has long been a challenge problem in differential privacy, subject to a long line of work \cite{contingency,hardsynth, conjunctions, contingency2,marginals,marginals2,conjunctions2,conjunctions3}, and is known to be computationally hard even for the special case of \emph{two-way} conjunctions \cite{hardsynth}. Our results in particular imply the first oracle efficient algorithm for generating synthetic data for all $2^d$ conjunctions. Other interesting classes satisfy all of the conditions needed for our synthetic data generation algorithm to apply, including disjunctions, parities, and discrete halfspaces --- see Appendix \ref{app:separator} for details, and for how to allow negated variables in the class of conjunctions while preserving seperability.

\begin{definition}
Given a subset of variables $S \subseteq [d]$, the boolean conjunction defined by $S$ is the statistical query $q_S(x) = \wedge_{j \in S} x_j$.
The set of boolean conjunctions $\cQ_C$ defined over the hypercube $\cX = \{0,1\}^d$ is:
$$\cQ_C = \{q_S \mid S\subseteq [d] \}$$
\end{definition}

Boolean conjunctions are $(d, d)$ dually-separable (see Appendix \ref{app:separator} for the separator set).

Thus, we can instantiate Theorem~\ref{oraclequerythm} with (e.g.) the Gaussian RSPM algorithm, and obtain an oracle-efficient algorithm for generating synthetic data for all $2^d$ conjunctions that outputs a synthetic dataset $S'$ that satisfies:
$$\max_{q \in \cQ_C}|q(S)-q(S')| \leq   \tilde O\left(\frac{d^{11/8}}{\sqrt{\epsilon n}}\right)$$

%We remark that halfspace queries contain many fundamental subclasses of queries,  including $k$-way marginals. Halfspace queries are a particularly interesting example in the study of \textit{dually separable} query classes, because they have separator sets, and are \textit{self dual}. By self dual we mean that for halfspace queries, one can view a query $q$ applied to a data point $x$, as a halfspace query $x$ applied to a datapoint $q$, e.g. $q(x) = x(q)$. In fact, it was this initial nearly tautological observation about halfspaces, that motivated all of the subsequent work in this paper. As such we can implement a fully oracle efficient version of $\OracleQuery$ assuming access to an oracle for solving weighted optimization problems for halfspaces.

%By Lemma~\ref{halfsep}, $\cQ$ over $\cX$ is $(|\cX|d, |\cV|d)$ dually separable. Hence by point $2.$ in Theorem~\ref{oraclequerythm}:
%\begin{corollary}
%\label{halfspacecor}
%\item If the Query Player plays Gaussian RSPM then given access to a perfect optimization oracle for halfspaces, for any $\epsilon \leq \log 2, \delta' > 0,$  $\OracleQuery$  satisfies $(\epsilon, \delta')$  differential privacy and
%$$\sup_{q \in \cQ}|\Ex{x\sim p_T}{q(x)}-q(D)| \leq  \tilde{O}(\frac{d^{7/8}(|\cX|-1)^{1/8}\log |\cX|^{1/4}(|\cV|-1)^{3/4}}{\sqrt{\epsilon n }})$$
%\end{corollary}

%%% Local Variables:
%%% mode: latex
%%% TeX-master: "heuristicDP"
%%% End:

\section{A Barrier}
\label{sec:barrier}
In this paper, we give \emph{oracle-efficient} private algorithms for learning and synthetic data generation for classes of queries $\cQ$ that exhibit special structure: small universal identification sets. Because of information theoretic lower bounds for differentially private learning \cite{intervals,littlestoneprivacy}, we know that these results cannot be extended to \emph{all} learnable classes of queries $\cQ$. But can they be extended to all classes of queries that are information theoretically learnable subject to differential privacy? Maybe --- this is the most interesting question left open by our work. But here, we present a ``barrier'' illustrating a difficulty that one would have to overcome in trying to prove this result. Our argument has three parts:
\begin{enumerate}
\item First, we observe a folklore connection between differentially private learning and online learning: any differentially private empirical risk minimization algorithm $\cA$ for a class $\cQ$ that \emph{always outputs the exact minimizer of a data-independent  perturbation of the empirical risks} can also be used as a no-regret learning algorithm, using the ``follow the perturbed leader'' analysis of Kalai and Vempala \cite{KV05}. The per-round run-time of this algorithm  is exactly equal to the run-time of $\cA$.
\item Oracle-efficient no-regret learning algorithms are subject to a lower bound of Hazan and Koren \cite{HK16}, that states that even given access to an oracle which solves optimization problems over a set of experts $\cQ$ in unit time, there exist finite classes $\cQ$ such that obtaining non-trivial regret guarantees requires total running time larger than $\mathrm{poly}(|Q|)$. This implies a lower bound on the magnitude of the perturbations that an algorithm of the type described in (1) must use.
\item Finally, we observe for any finite class of hypotheses $\cQ$, information theoretically, it is possible to solve the empirical risk minimization problem on a dataset of size $T$ up to error $O(\frac{\log |\cQ|}{\epsilon T})$ using the generic learner from \cite{KLNRS08}. This implies a separation between the kinds of algorithms described in 1), and the (non-efficiently) achievable information theoretic bounds consistent with differential privacy.
\end{enumerate}

%We leverage the lower bound result in \cite{HK16} to show a barrier on
%obtaining a more general class of oracle-efficient algorithms.
%Consider the following generic private ERM algorithm $\cM$ that takes
%a dataset $D\in \cX^n$ as input and outputs a query (or hypothesis)
%$q$ from some finite class $\cQ$. Given any privacy parameters $\eps$
%and $\delta$, the algorithm first draws a perturbation vector
%$Z\in \mathbb{R}^{|\cQ|}$ (with coordinates indexed by $q\in \cQ$)
%from some distribution $\cD_{\eps, \delta}$. Then, it output the query
%$\hat q$ with the smallest ``noisy'' loss:
%\begin{equation}\label{hotness}
%   \hat q \in \argmin_{q\in \cQ} \left[n\, q(D) + Z_q\right].
%\end{equation}

We emphasize that oracle efficient algorithms for learning over $\cQ$ have access to a non-private oracle which exactly solves the learning problem over $\cQ$ --- not an NP oracle, for which the situation is different (see the discussion in Section \ref{sec:conc}).

First we define the class of mechanisms our barrier result applies to:
\begin{definition}
We say that an $(\epsilon,\delta)$-differentially private learning algorithm $\cA:\cX^n\rightarrow \cQ$ for $\cQ$ is a perturbed Empirical Risk Minimizer (pERM) if there is some distribution $\cD_{\epsilon,\delta}$ (defined independently of the data $S$) over perturbations $Z \in \mathbb{R}^{|Q|}$ such that on input $S \in \cX^n$, $\cA$ outputs:
$$\cA(S) = \arg\min_{q \in \cQ} \left( n\cdot q(S) + Z_q \right)$$
where $Z \sim \cD_{\epsilon,\delta}$.
\end{definition}

We note that many algorithms are pERM algorithms. The most obvious example is \emph{report-noisy-min}, in which each $Z_q \sim \Lap{1/\epsilon}$ independently. The exponential mechanism instantiated with empirical loss as its quality score (i.e. the generic learner of \cite{KLNRS08}) is also a pERM algorithm, in which each $Z_q$ is drawn independently from a Gumbel distribution \cite{DworkRoth}. But note that the coordinates $Z_q$ need not be drawn independently: The oracle-efficient RSPM algorithm we give in Section \ref{rspm} is also a pERM algorithm, in which the perturbations $Z_q$ are introduced in an implicit (correlated) way by perturbing the dataset itself. And it is natural to imagine that many algorithms that employ weighted optimization oracles --- which after all solve an exact minimization problem --- will fall into this class. The expected error guarantees of these algorithms are proven by bounding $\mathbb{E}[||Z||_\infty]$, which is typically a tight bound.

 %This generic algorithm captures many existing private algorithms. In
 %particular, RSPM can be viewed as a special case where the
% perturbation is induced by assigning random weights on the separator
% set $S$: $Z_q = \sum_{i=1}^m \eta_i q(e_i)$. When each $Z_q$ is drawn
% independently from the Laplace and Gumbel distribution, this recovers
% the ``report noisy min'' and the exponential mechanism respectively.
% Note that when the perturbation is induced by randomly weighted
% examples from the separator set, the step in \Cref{hotness} can be
% solved by a single call to $\cO$.

 % \sw{might include the monotone transformation}\ar{What monotone transformation?}\sw{something like the generalized exp mech?}
 We now briefly recall the online learning setting. Let $\cQ$ be an arbitrary class of functions $q:\cX \rightarrow [0,1]$. In rounds $t = 1, \ldots, T$, the learner
 selects a function $q^t\in \cQ$, and an (adaptive) adversary selects an example
 $x^t\in \cX$, as a function of the sequence $(q^1,x^1,\ldots,q^{t-1},x^{t-1})$. The learner incurs a loss of $\ell^t = q^t(x^t)$. A standard
 objective is to minimize the \emph{expected average regret}:
 \[
   R(T) = \Ex{}{\frac{1}{T}\sum_{t=1}^T q^t(x^t)} - \min_{q\in \cQ}
     \frac{1}{T} \sum_{t=1}^T q(x^t)
 \]
 where the expectation is taken over the randomness of the learner. A weighted optimization oracle in the online learning setting is exactly the same thing as it is in our setting: Given a weighted dataset $(S, w)$, it returns $\arg\min_{q \in \cQ} \sum_{x_i \in S} w_i\cdot q(x_i)$.

 A natural way to try to use a private learning algorithm in the online learning setting is just to run it at each round $t$ on the dataset defined on the set of data points observed so far: $S_t = \{x^1,\ldots,x^{t-1}\}$.
 \begin{definition}
 Follow the Private Leader, instantiated with $\cA$, is the online learning algorithm that at every round $t$ selects $q^t = \cA(S_t)$.
 \end{definition}
The follow the private leader algorithm instantiated with $\cA$ has a controllable regret bound whenever $\cA$ is a differentially private pERM algorithm. The following theorem is folklore, but follows essentially from the original analysis of ``follow the perturbed leader'' by Kalai and Vempala \cite{KV05}. See e.g. the lecture notes from \cite{RS18} or \cite{Abernethy} for an example of this analysis cast in the language of differential privacy. We include a proof in Appendix \ref{app:barrier} for completeness.

\begin{restatable}{theorem}{tokyo}
\label{tokyo}
  Let $\eps, \delta\in (0, 1)$ and let $\cA$ be an $(\epsilon,\delta)$ differentially private pERM algorithm for query class $\cQ$, with perturbation distribution $\cD_{(\eps, \delta)}$. Then Follow the Private Leader instantiated with $\cA$ has expected regret bounded by:
  \[
    R(T) \leq  O\left(\epsilon + \delta  + \frac{\Ex{Z \sim \cD_{\eps,
            \delta}}{\|Z\|_\infty}}{T} \right)
  \]
\end{restatable}

Note that the regret is controlled by $\Ex{Z \sim \cD_{\eps,\delta}}{\|Z\|_\infty}$, which also controls the error of $\cA$ as a learning algorithm.

%Now we will focus on the special case of the ERM algorithm that is
%oracle-efficient. In particular, each perturbation $Z_q$ is given by
%$Z_q = \sum_{s_i\in S} \eta_i \, q(s_i)$ with a set of examples $S$ of
%polynomial size and each $\eta_i$ is a random weight drawn from
%distribution. Note that the computation in \Cref{hotness} can then be
%solved by a single call to the oracle $\cO$, and if the algorithm
%satisfies differential privacy, it can then turn into an
%oracle-efficient online learner.

We wish to exploit a lower bound on the running time of oracle efficient online learners over arbitrary sets $\cQ$ due to Hazan and Koren \cite{HK16}:
%We now quote a lower bound on the running time of oracle-efficient online learners, over arbitrary sets $\cQ$, due to Hazan and Koren \cite{HK16}:

\begin{theorem}[\cite{HK16}]
 For every algorithm with access to a weighted optimization oracle $\cO$, there exists a class of functions $\cQ$  such that the algorithm cannot guarantee that its
  expected average regret will be smaller than $1/16$  in total time less than
  $O\left( \sqrt{|\cQ|} / \log^3(|\cQ|)\right)$.
\end{theorem}
Here, it is assumed that calls to the oracle $\cO$ can be carried out in unit time, and \emph{total} time refers to the cumulative time over all $T$ rounds of interaction. Hence, if $\cA$ is oracle-efficient --- i.e. it runs in time $f(t) = \textrm{poly}(t,\log|\cQ|)$ when given as input a dataset of size $t$, the \emph{total} run time of follow the private leader instantiated with $\cA$ is: $\sum_{t=1}^T f(t) \leq T \cdot f(T) = \textrm{poly}(T,\log|\cQ|)$.

This theorem is almost what we want --- except that the order of quantifiers is reversed. It in principle leaves open the possibility that for every class $\cQ$, there is a different oracle efficient algorithm (tailored to the class) that can efficiently obtain low regret. After all, our RSPM algorithm is non-uniform in this way --- for each new class of functions $\cQ$, it must be instantiated with a separator set for that class.

Via a min-max argument together with an equilibrium sparsification technique, we can give a version of the lower bound of \cite{HK16} that has the order of quantifiers we want --- see Appendix \ref{app:barrier} for the proof.

\begin{restatable}{theorem}{hongkong}
\label{hongkong}
For any $d$, there is a fixed finite class of statistical queries $\cQ$ of size $|\cQ| = N = 2^d$ defined over a data universe of size $|\cX| =  O(N^5\log^2 N)$  such that for every online learning algorithm with access to a weighted optimization oracle for $\cQ$, it cannot guarantee that its expected average regret will be $o(1)$ in total time less than $\Omega(\sqrt{N}/\log^3(N))$.
\end{restatable}

Theorem \ref{hongkong} therefore implies that follow the private leader, when instantiated with any oracle-efficient differentially private pERM algorithm $\cA$ cannot obtain diminishing regret $R(T) = o(1)$ unless the number of rounds $T = \Omega(|Q|^c)$ for some $c > 0$. In combination with Theorem \ref{tokyo}, this implies our barrier result:
\begin{restatable}{theorem}{barrier}
\label{barrier}
Any oracle efficient (i.e. running in time $\textrm{poly}(n, \log |\cQ|)$) $(\epsilon,\delta)$-differentially private pERM algorithm instantiated with a weighted optimization oracle for the query class $\cQ$ defined in Theorem \ref{hongkong}, with perturbation distribution $\cD_{(\eps, \delta)}$ must be such that for every $(\epsilon + \delta ) = o(1)$:
$$\mathbb{E}_{Z \sim \cD_{(\eps, \delta)}}[||Z||_\infty] \geq  \Omega(|\cQ|^c)$$
for some constant $c > 0$. 
\end{restatable}

If the accuracy guarantee of $\cA$ is proportional to $\mathbb{E}_{Z \sim \cD_{(\eps, \delta)}}[||Z||_\infty]$ (as it is for all pERM algorithms that we know of), this means that there exist finite classes of statistical queries $\cQ$ such that no oracle-efficient algorithm can obtain non-trivial error unless the dataset size $n \geq \textrm{poly}(|\cQ|)$. Of course, if $n \geq \textrm{poly}(|\cQ|)$, then algorithms such as report-noisy-min and the exponential mechanism can be run in polynomial time.

This is in contrast with what we can obtain via the generic (inefficient) private learner of \cite{KLNRS08}, which obtains expected error $O\left(\frac{\log|Q|}{\epsilon n}\right)$, which is non-trivial whenever $n = \Omega\left(\frac{\log |\cQ|}{\epsilon}\right)$. Similarly, because we show in Theorem \ref{hongkong} that the hard class $\cQ$ can be taken to have universe size $\cX = \mathrm{poly}(\cQ)$, this means that information theoretically, it is even possible to privately solve the (harder) problem of $\alpha$-accurate synthetic data for $\cQ$ for $\alpha = O\left(\left(\frac{\log^2|\cQ|}{\epsilon n}\right)^{1/3}\right)$ using the (inefficient) synthetic data generation algorithm of \cite{BLR08}. This is non-trivial whenever $n = \Omega\left(\frac{\log^2 |\cQ|}{\epsilon}\right)$. In contrast, our barrier result states is that \emph{if} there exists an oracle-efficient learner $\cA$ for this class $\cQ$ that has polynomially related sample complexity to what is obtainable absent a guarantee of oracle efficiency, then $\cA$ must either:
\begin{enumerate}
\item Not be a pERM algorithm, or:
\item Have expected error that is $O\left(\frac{\mathrm{poly}(\log \mathbb{E}_{Z \sim \cD_{(\eps, \delta)}}[||Z||_\infty] )}{n}\right)$.
\end{enumerate}

Condition 2. seems especially implausible, as for every pERM we are aware of, $\mathbb{E}_{Z \sim \cD_{(\eps, \delta)}}[||Z||_\infty] $ is a tight bound (up to log factors) on its expected error. In particular, this barrier implies that there is no oracle efficient algorithm for \emph{sampling} from the exponential mechanism distribution used in the generic learner of \cite{KLNRS08} for arbitrary query classes $\cQ$.

\section{Conclusion and Open Questions}
\label{sec:conc}
In this paper, we have initiated the systematic study of the power of \emph{oracle-efficient} differentially private algorithms, and have made the distinction between oracle-dependent non-robust differential privacy and robust differential privacy. This is a new direction that suggests a number of fascinating open questions. In our opinion, the most interesting of these is:
\begin{quote}
{\textit ``Can every learning and synthetic data generation problem that is solvable subject to differential privacy be solved with an oracle-efficient (robustly) differentially private algorithm, with only a polynomial blow-up in sample complexity?''}
\end{quote}
 It remains an open question whether or not finite Littlestone dimension characterizes private learnability (it is known that infinite Littlestone dimension precludes private learnability \cite{littlestoneprivacy}) --- and so one avenue towards resolving both open questions in the affirmative simultaneously would be to show that finite Littlestone dimension can be leveraged to obtain oracle-efficient differentially private learning algorithms.

 However, because of our barrier result, we conjecture that the set of query classes that are privately learnable in an oracle-efficient manner is a \emph{strict subset} of the set that are privately learnable. If this is so, can we precisely characterize this set? What is the right structural property, and is it more general than the sufficient condition of having small universal identification sets that we have discovered?

 Even restricting attention to query classes with universal identification sets of size $m$, there are interesting quantitative questions. The Gaussian version of our RSPM algorithm efficiently obtains error that scales as $m^{3/2}$, but information-theoretically, it is possible to obtain error scaling only linearly with $m$. Is this optimal error rate possible to obtain in an oracle-efficient manner, or is the $\sqrt{m}$ error overhead that comes with our approach necessary for oracle efficiency?

 Our \textbf{PRSMA} algorithm shows how to generically reduce from an oracle-dependent guarantee of differential privacy to a guarantee of robust differential privacy --- \emph{but at a cost}, both in terms of running time, and in terms of error. Are these costs necessary? Without further assumptions on the construction of the oracle, it seems difficult to avoid the $O(1/\delta)$-overhead in running time, but perhaps there are natural assumptions that can be placed on the failure-mode of the oracle that can avoid this. It is less clear whether the error overhead that we introduce --- by running the original algorithm on an $\epsilon$ fraction of the dataset, with a privacy parameter $\epsilon' \approx 1/\sqrt{\epsilon n}$ --- is necessary. Doing this is a key feature of our algorithm and analysis, because we take advantage of the fact that differentially private algorithms are actually \emph{distributionally private} when $\epsilon'$ is set this small --- but perhaps it can be avoided entirely with a different approach.

 Our barrier result takes advantage of a connection between differentially private learnability and online learnability. Because private pERM algorithms can be used efficiently as no-regret learning algorithms, they are subject to the lower bounds on oracle-efficient online learning proven in \cite{HK16}. But perhaps the connection between differentially private learnability and online learnability runs deeper. Can \emph{every} differentially private learning algorithm be used in a black box manner to efficiently obtain a no-regret learning algorithm? Note that it is already known that private learnability implies finite Littlestone dimension, so the open question here concerns whether there is an \emph{efficient blackbox} reduction from private ERM algorithms to online learning algorithms. If true, this would convert our barrier for pERM algorithms into a full lower-bound for oracle-efficient private learning algorithms generally.

Finally, a more open ended question --- that applies both to our work and to work on oracle efficiency in machine learning more generally --- concerns how to refine the model of oracle efficiency. Ideally, the learning problems fed to the oracle should be ``natural'' --- e.g. a small perturbation or re-weighting of the original (non-private) learning problem, as is the case for the algorithms we present in our paper. This is desirable because presumably we believe that the heuristics which can solve hard learning problems in practice work for ``natural'' instances, rather than arbitrary problems. However, the definition for oracle efficiency that we use in this paper allows for un-natural algorithms. For example, it is possible to show that the problem of sampling from the exponential mechanism of \cite{MT07} defined by rational valued quality scores that are efficiently computable lies in $\mathbf{BPP}^{\mathbf{NP}}$  --- in other words, the sampling can be done in polynomial time given access to an oracle for solving circuit-satisfiability problems\footnote{This construction is due to Jonathan Ullman and Salil Vadhan (personal communication). It starts from the ability to sample uniformly at random amongst the set of satisfying assignments of an arbitrary polynomially sized boolean circuit given an NP oracle, using the algorithm of \cite{npwitness}. For any distribution $\cP$ such that there is a polynomially sized circuit $C$ for which the relative probability mass on any discrete input $x$ can be computed by $C(x)$, we can construct a boolean circuit $C'$ that computes for bounded bit-length rational numbers $w$: $C'(x,w) = 1$ if $C(x) \geq w$. The marginal distribution on elements $x$ when sampling uniformly at random from the  satisfying assignments of this circuit is $\cP$.}. This implies in particular, that there exists an oracle efficient algorithm (as we have defined them) for any NP hard learning problem --- because the learning oracle can be used as an arbitrary $\mathbf{NP}$ oracle via gadget reductions\footnote{Note that this procedure is not \emph{robustly} differentially private, since sampling from the correct distribution occurs only if the oracle does not fail. But it could be fed into our \textbf{PRSMA} algorithm to obtain robust privacy. It also does not solve synthetic data generation oracle efficiently because the quality score used for synthetic data generation in \cite{BLR08} is not computable by a polynomially sized circuit generally.}. The same logic implies that there are oracle efficient no-regret learning algorithms for any class of experts for which offline optimization is NP hard --- because an NP oracle can be used to sample from the multiplicative weights distribution. But these kinds of gadget reductions seem to be an abuse of the model of oracle efficiency, which currently reduces to all of  $\mathbf{BPP}^{\mathbf{NP}}$ when the given oracle is solving an NP hard problem\footnote{This does not contradict the lower bound of \cite{HK16} for oracle efficient online learning, or our barrier result/conjectured separation in the case of private learning algorithms. This is because oracles solving problems that don't have polynomial time algorithms, but \emph{are not NP hard} cannot be used to encode the arbitrary circuit-SAT instances needed to implement an NP oracle.}. Ambitiously, might there be a refinement of the model of oracle efficiency that requires one to prove a utility theorem along the following lines: assuming an oracle which can with high probability solve learning problems drawn from the actual data distribution, the oracle efficient algorithm will (with slightly lower probability) solve the private learning problem when the underlying instance is drawn from the same distribution. Theorems of this sort would be of great interest, and would (presumably) rule out ``unnatural'' algorithms relying on gadget reductions.  

\paragraph*{Acknowledgements}
We thank Michael Kearns, Adam Smith, Jon Ullman and Salil Vadhan for insightful conversations about this work. 

 \bibliographystyle{alpha}
\bibliography{./refs}
\appendix
\section{Examples of Separator Sets}
\label{app:separator}

\subsection{Separator Sets for Empirical Loss Queries}
In this section we show how to construct a separator set for a family of empirical loss queries defined over a hypothesis class,
given a separator set for the corresponding hypothesis
class. Formally, let us consider the data domain $\cX$ to be the set of
labelled examples $\cX_A \times\{0, 1\}$, where $\cX_A$ is the domain
of \emph{attribute vectors}. Let $\cH$ be a hypothesis class, with
each hypothesis $h: \cX_a \rightarrow \{0, 1\}$ mapping attribute
vectors to binary labels. Let $U_\cH$ be a separator set for $\cH$
such that for any pair of distinct hypotheses $h, h'\in \cH$, there is
a $u\in U_\cH$ such that $h(u) \neq h'(u)$.

For every $h\in \cH$, let $q_h\colon \cX \rightarrow \{0, 1\}$ be the
\emph{loss query} corresponding to $h$ such that $q_h((x, y)) = \mathbbm{1}[h(x)\neq y
]$. Define the query class $\cQ_\cH$ to be $\{q_h \mid h\in \cH\}$,
and let $U = \{(u, 0) \mid u\in U_\cH\}$. Solving the learning problem over $\cH$ corresponds to solving the minimization problem over $\cQ_{\cH}$.

\begin{claim}
  The set $U$ is a separator set for the query class $\cQ_\cH$.
\end{claim}

\begin{proof}
  Let $q_h, q_{h'}\in \cQ_\cH$ be a pair of distinct queries. Since
  $U_\cH$ is a separator set for $\cH$, there exists an element
  $u\in U_\cH$ such that $h(u) \neq h(u')$. As a result,
  \[
    q_h((u, 0)) = h(u) \neq h(u') = q_{h'}((u, 0)).
  \]
  Therefore, $U$ is a separator set for $\cQ_\cH$.
\end{proof}

Thus, if a hypothesis class $\cH$ has a separator set of size $m$, so does the set of queries $\cQ_{\cH}$ representing the empirical loss of the hypotheses $h \in \cH$.

\subsection{Separator Sets for Common Hypothesis Classes}
In this section we provide some examples of hypothesis classes with small
separator sets. We say that a class of queries $\cQ$ is self-dual of $\cQ = \cQd$.

\subsubsection{Conjunctions, Disjunctions, and Parities}

We begin with some easy but important cases that can be verified by inspection:
\begin{fact}
  Let $\cX_A = \{0, 1\}^d$. For every $j\in [d]$, let $e_j\in\cX_A$ be
  a boolean vector that has 1 in the $j$-th coordinate and 0 in all
  others, and let $\overline e_j\in \cX_A$ be the vector that has 0 in
  the $j$-th coordinate and 1 in all others. Let
  $U = \{e_j \mid j\in [d]\}$ and
  $\overline U = \{\overline e_j \mid j\in [d]\}$. Then for the
  following hypothesis classes:
\begin{itemize}
\item $\overline U$ is a separator set for conjunctions over $\cX_A$:
$\{\wedge_{j \in S} x_j \mid S\subseteq [d] \}$;
\item $U$ is a separator set for disjunctions over $\cX_A$:
  $\{\vee_{j \in S} x_j \mid S\subseteq [d] \}$;
\item $U$ is a separator set for parities over $\cX_A$:
  $\{\oplus_{j \in S} x_j \mid S\subseteq [d] \}$.
\end{itemize}
Hence, each of these classes has a separator set of size $d$, equal to the data dimension. Moreover, each of these classes is self-dual.
\end{fact}

\begin{remark}
Note here we have defined \emph{monotone} conjunctions, disjunctions, and parities --- i.e. in which the literals cannot appear negated. Up to a factor of 2 in the dimension, this is without loss of generality, since we can add $d$ extra coordinates to each example which by convention will encode the negation of each of the values in the first $d$ coordinates. This allows us to handle non-monotone conjunctions, disjunctions, and parities as well.
\end{remark}

\subsubsection{Discrete Halfspaces}

Discrete halfspaces are a richer set of hypotheses that generalize both conjunctions and disjunctions. Let $\cX_A = B^d$ for some set $B \subseteq [-1,1]$. For example we could allow real valued features by letting $B = [-1,1]$, or we could take some discretization. We will assume that $0 \in B$. Halfspaces themselves will be defined with respect to vectors of weights $w$ that are discretized to lie in some finite set $w \in \cV^d$, for $\cV \subseteq [-1,1]$. Here we could take $|\cV| = 2$ by requiring that the \emph{weights} be defined over the hypercube ($\cV = \{-1,1\}$), or we could allow finer discretization. It is important that $\cV$ be finite.

\begin{definition}[Halfspace Query]
Given a weight vector $w \in \cV^d$, the halfspace query parameterized by $w$ is defined to be $q_w(x) = \textbf{1}\{w \cdot x \geq 1\}$. Let $\cQ_{\cV} = \{q_w : w \in \cV^d\}$.

 The value of ``1'' used as the intercept is arbitrary, and can is set without loss of generality at the cost of 1 extra data dimension.
\end{definition}

% We assume the class $\cQ = \cV^{d}$ for a discrete set $\cV \subset [-1,1], |\cV| < \infty$, and that $\cQ$ is closed under negation; e.g. $\forall q \in \cQ, -q \in \cQ$.
%}
%We remark that halfspace queries contain many fundamental subclasses of queries,  including $k$-way marginals. Halfspace queries are a particularly interesting example in the study of \textit{dually separable} query classes, because they have separator sets, and are \textit{self dual}. By self dual we mean that for halfspace queries, one can view a query $q$ applied to a data point $x$, as a halfspace query $x$ applied to a datapoint $q$, e.g. $q(x) = x(q)$. In fact, it was this initial nearly tautological observation about halfspaces, that motivated all of the subsequent work in this paper. As such we can implement a fully oracle efficient version of $\OracleQuery$ assuming access to an oracle for solving weighted optimization problems for halfspaces.

\begin{lemma}
$\cQ_{\cV}$ has a separator set of size $(|\cV|-1)d$. In particular, if weights are defined over the hypercube $(\cV = \{-1,1\})$, then the separator set is of size $d$.
\end{lemma}

\begin{proof}
Suppose the elements of $\cV$ are $x_1 < x_2 < \ldots <  x_{|\cV|}$, which without loss of generality are distinct. We construct a separator set of size $(|\cV|-1) d$ as follows. Let $c_1, \ldots c_{|\cV|-1}$ be a
sequence such that $c_v$ lies in  $[\frac{1}{x_{v+1}}, \frac{1}{x_v})$.  Define the vector $s_{jv} \in \cX$ to take value $c_v$ in coordinate $j$, and $0$ elsewhere. We claim that $U = \{s_{jv}\}_{j = 1, \ldots d, v = 1 \ldots |\cV|-1}$ is a separator set for $\cQ_{\cV}$. Let $q_{w_1} \neq q_{w_2} \in \cQ_{\cV}$. Since $w_1 \neq w_2$ they must differ in some coordinate, call it $k$. Let $w_{1,k} = x_l, q_{2,k} = x_m$, and without loss of generality assume $x_m > x_l$. Then by construction of $\{c_v\}$ there exists $c_v$ such that $c_v \geq \frac{1}{x_m}$, but $c_v < \frac{1}{x_l}$.
We therefore have that $w_1 \cdot s_{kv} = x_l\cdot c_v < 1$, whereas $w_2\cdot s_{kv} = x_m\cdot c_v \geq 1$, and hence $q_{w_1}(s_{kv}) = 0 \neq 1 = q_{w_2}(s_{kv})$.
\end{proof}

Finally, we note that the dual of $\cQ_{\cV}$ is $\cQ_{\cB}$ --- and so if  $B = \cV$, the set of halfspace queries $\cQ_{\cV}$ is self-dual.

\subsubsection{Decision Lists}
For simplicity, in this section we discuss \emph{monotone} decision lists, in which variables cannot be negated --- but as we have already remarked, this is without loss of generality up to a factor of $2$ in the dimension. Here we define the general class of $k$-decision lists: $1$-decision lists (often just referred to as decision lists) are a restricted class of binary decision tree in which one child of every internal vertex must be a leaf. $k$-decision lists are a generalization in which each branching decision can depend on a conjunction of $k$ variables.

\begin{definition}
A monotone $k$-decision list over $\cX_A = \{0,1\}^{d}$ is defined by an ordered sequence $L = (c_1, b_1), \ldots , (c_l, b_l)$ and a bit $b$, in which each $c_i$ is a monotone conjunction of at most $k$ literals, and each $b_i \in \{0,1\}$. Given a pair $(L, b)$, the decision list $q_{L,b}(x)$ computes as follows: it outputs $q_{L,b}(x)=b_j$ where $j$ is the minimum index in $L$ satisfying $c_j(a) =1$. If $c_j$ is the first conjunction that $x$ satisfies in the definition of $q_{L,b}(x)$ we say that $x$ binds at $c_j$. If no such index exists then $q_{L,b}(x) = b$, and we say $q_{L,b}(x)$ does not bind. $k$-decision lists are strict generalizations of $k$-DNF and $k$-CNF formulae.
\end{definition}

\begin{lemma}
The class of $k$-decision lists has a separator set of size $\leq \sum_{j = 0}^{2k}{{d} \choose{j}}$. In particular, $1$-decision lists have a separator set of size $O(d^2)$.
\end{lemma}
\begin{proof}
 Let $MC_k$ denote the set of monotone conjunctions of $ \leq k$ literals over $\cX_\cA$. For any two $s, l$ in $MC_k$, let $e_{sl}$ denote the element $a \in \cX_A$ such that all literals appearing in either $s$ or $l$ are set to $1$, and all others are set to $0$. Define $U = \{e_{sl}: s,l \in MC_k\}$. Since each $e_{sl}$ corresponds to setting between $0$ and $2k$ of the $d$ variables to $1$, there are $\sum_{j = 0}^{2k}{{d} \choose{j}}$ elements in $U$. \\
Now let $(L_1, b_1) \neq (L_2, b_2)$ be two distinct $k$-decision lists. Since the two decision lists are distinct there must exist $x \in \cX_A$ such that $q_{(L_1, b_1)}(x) \neq q_{(L_2, b_2)}(x)$. Let $c_1,c_2$ be the conjunctions on which $(L_1, b_1), (L_2, b_2)$ bind on $x$ respectively. (If $L_i$ does not bind on $x$, set $c_i$ to be the empty conjunction). Define $e_{c_1,c_2} \in U$ as above. We claim that $q_{(L_1, b_1)}(e_{c_1,c_2}) = q_{(L_1, b_1)}(x) \neq q_{(L_2, b_2)}(x) = q_{(L_2, b_2)}(e_{c_1,c_2})$, and hence $e_{c_1,c_2} \in U$ distinguishes $(L_1, b_1)$ from $(L_2, b_2)$, which proves the claim. The key fact that $q_{(L_i, b_i)}(e_{c_1,c_2}) = q_{(L_i, b_i)}(x)$ follows from the fact that $(L_i, b_i)$ still binds at $c_i$ (or does not bind at all) on input $x$, and it can't bind earlier since any monotone conjunction satisfied by $e_{c_1,c_2}$ is satisfied by $x$.
\end{proof}

\subsubsection{Other Classes of Functions}
In this section, we have exhibited simple constructions of small universal identification sets for conjunctions, disjunctions, parities, discrete halfspaces, and $k$-decision lists. This is not an exhaustive enumeration of such classes --- we give these as examples of the most frequently studied classes of boolean functions in the PAC learning literature. However, short universal identification sets for other classes of functions are known. For example:
\begin{theorem}[\cite{goldman1993exact}]
There exist polynomially sized universal identification sets for the following two classes of circuits:
\begin{enumerate}
\item Logarithmic depth read-once majority formulas, and
\item Logarithmic depth read-once positive NAND formulas.
\end{enumerate}
\end{theorem}

\section{RSPM with Gaussian Perturbations}
\label{sec:gauss}

We now present a Gaussian variant of the RSPM algorithm.

\begin{algorithm}
\label{alg:Grspm}
\textbf{Gaussian RSPM}\\
\textbf{Given}: A separator set $U = \{e_1,\ldots,e_m\}$ for a class of statistical queries $\cQ$ and a weighted optimization oracle $\cO^*$ for $\cQ$, privacy parameters $\epsilon$ and $\delta\in (0, 1/e)$, and $\sigma = \frac{3.5\sqrt{m \ln(1/\delta)}}{  \eps}$.  \newline
\textbf{Input}: A dataset $S \in \cX^n$ of size $n$.\newline
\textbf{Output}: A statistical query $q \in \cQ$.
\begin{algorithmic}
\STATE Sample independently $\eta_i \sim \cN(0, \sigma^2)$ for $i \in \{1,\ldots,m\}$
\STATE Construct a weighted dataset $WD$ of size $n + m$ as follows:
$$WD(S,\eta) = \{(x_i, 1) : x_i \in S\} \cup \{(e_i, \eta_i) : e_i \in U\}$$
\STATE Output $q = \cO(WD(S,\eta))$.
\end{algorithmic}
\end{algorithm}

\begin{theorem}[Utility]
  The Gaussian RSPM algorithm is an oracle-efficient $(\alpha,\beta)$-minimizer for $\cQ$ for:
$$\alpha =O\left(\frac{m\sqrt{m\ln(2m/\beta) \ln(1/\delta)} }{\eps n}\right)$$
\end{theorem}

\begin{proof}
  Note that for each noise variable $\eta_i$, we have the following
  tail bound:
  \[
    \Pr[|\eta_i| \geq t] \leq 2 \exp\left(-\frac{t^2}{2\sigma^2} \right)
  \]
  Taking a union bound, we have with probability at least $1 - \beta$
  that
  \[
    \max_i |\eta_i| \leq \sqrt{2 \ln(2m/\beta)}\sigma
  \]
  Then the proof follows from the same reasoning in the proof for
  \Cref{thm:rspm}.
\end{proof}

\begin{theorem}[Privacy]
If $\cO^*$ is a weighted optimization oracle for $\cQ$, then the Gaussian Report Separator-Perturbed Min algorithm is $(\epsilon,\delta)$-differentially private.
\end{theorem}

\begin{proof}

  In the following, we will inherit the notation from
  \Cref{rspm}. We will denote the output by the algorithm on input
  dataset $S$ under realizations of the perturbations $\eta$ as
  $\cQ(S,\eta) = \cO^*(WD(S,\eta))$. For any $q \in \cQ$, let
  $\cE(q,S) = \{\eta : \cQ(S,\eta) = q\}$.  We will use the same
  mapping $f_q(\eta):\mathbb{R}^m\rightarrow \mathbb{R}^m$ as defined
  in \Cref{rspm}.  We will write $P_G$ to denote the pdf of the
  distribution $\cN(0, \sigma^2)$. We will also again define the set
  $B$ as the set of $\eta$ for which there are multiple minimizers
  $\hat q$. For every $\eta \in \RR^m \setminus B$, let $\hat q_\eta$
  be the unique minimizer. Note that Lemmas~\ref{lem:inclusion} and
  \ref{lem:B} both hold in our setting, and in
  particular, Lemma~\ref{lem:B} holds because of the continuity of the
  Gaussian distribution.

  Similar to the standard analysis for the Gaussian
  mechanism~\cite{DworkRoth}, we will leverage the fact that the
  distribution $\cN(0, \sigma^2I)$ is independent of the orthonormal
  basis from which its constituent normals are drawn, so we have the
  freedom to choose the underlying basis without changing the
  distribution. For any $r\in \mathbb{R}^m$ and any $q\in \cQ$, fix
  such a basis $b_1, \ldots, b_m$ such that $b_1$ is parallel to
  $\nu = f_q(r) - r$.  A random draw $\eta$ from $\cN(0, \sigma^2 I)$
  can be realized by the following process: first draw signed lengths
  $\lambda_i \sim \cN (0, \sigma^2 )$, for $i \in [m]$, then define
  $\eta^{[i]}= \lambda_i b_i$ , and finally let
  $\eta = \sum^m_{i=1} \eta^{[i]}$.  For each $i\in [m]$, let
  $r^{[i]}$ be the projection of $r$ onto the direction of $b_i$,
  which gives $r = \sum_{i=1}^m r^{[i]}$.

 \begin{lemma}
 \label{lem:gaussian_density}
 Suppose that $\sigma > 1$. For any $r \in \mathbb{R}^m$, $q \in \cQ$,
 $$P_G(r) \leq \exp\left(\frac{1}{2\sigma^2}\left(m
     + 2\sqrt{m}\|r^{[1]}\|_2 \right)\right) P_G(f_q(r)).$$
\end{lemma}

\begin{proof}
  Note that for any $r\in \mathbb{R}^m$, we have
  \[
    P_G(r) = \frac{1}{(2\pi)^{m/2}
      \sigma^m}\exp\left(-\frac{\|r\|^2_2}{2\sigma^2} \right).
  \]
  We will write $\nu = f_q(r) - r$. It follows that
  \begin{align*}
    \frac{P_G(r)}{P_G(f_q(r))} &= \exp\left(\frac{1}{2\sigma^2}\left(\|r + \nu\|_2^2 - {\|r\|_2^2} \right)\right)
  \end{align*}

 Now we can write
  $$
  \|r + \nu\|_2^2 =  \|\nu + r^{[1]}\|_2^2 + \sum_{i=1}^m
\|r^{[i]}\|_2^2, \qquad\|r\|^2_2 = \sum_{i=1}^m     \|r^{[i]} \|_2^2.
  $$
  It follows that
  \[
    \|r^{[1]} + \nu\|_2^2 - \|r^{[1]}\|_2^2 = \|\nu\|_2^2 + 2 \|\nu\|_2\,
    \|r^{[1]}\|_2 \leq m + 2\sqrt{m} \|r^{[1]}\|_2
  \]
  This means
  \[
    \frac{P_G(r)}{P_G(f_q(r))} \leq \exp\left(\frac{1}{2\sigma^2}\left(m
        + 2\sqrt{m}\|r^{[1]}\|_2 \right)\right),
  \]
  which completes the proof.
\end{proof}

To finish up the privacy analysis, note that by
\Cref{lem:gaussian_density}, the ratio $P_G(\eta)/P_G(f_q(\eta))$ is
bounded by $\exp(\eps)$, as long as
$\|\eta^{[1]}\|_2 < \sigma^2 \eps/\sqrt{m} - \sqrt{m} /2$. Now we will
bound the probability that the random vector $\eta^{[1]}$ has norm
exceeding this bound. First, observe that
$\|\eta^{[1]}\|_2 = |\lambda_1|$, where $\lambda_1$ is a random draw
from the distribution $\cN(0, \sigma^2)$. Since $\lambda_1^2$ is a
$\chi^2$ random variable with degree of freedom 1, we can apply the
following tail bound \cite{laurent2000}: for any $t > 0$,
  \[
    \Pr[\lambda_1^2 \geq \sigma^2\left( \sqrt{2t} + 1 \right)^2] \leq
    \exp(-t)
  \]
  which can be further simplied to
  \[
    \Pr[\|\eta^{[1]}\|_2 \geq \sigma\left( \sqrt{2t} + 1 \right)] \leq
    \exp(-t)
  \]
  In other words, for any $\delta \in (0, 1/e)$, with probability at
  least $1 - \delta$, we have
  \[
    \|\eta^{[1]}\|_2 < \sigma(\sqrt{2\ln(1/\delta)} + 1) \equiv
    \Lambda
  \]
  It follows that $\Lambda \leq \sigma^2 \eps/\sqrt{m} - \sqrt{m} /2$,
  as long as $\sigma = \frac{c\sqrt{m \ln(1/\delta)}}{\eps}$ for any
  $c \geq 3.5$. We will use this value of $\sigma$ for the remainder
  of the analysis. Now let
  $L = \{\eta \in \mathbb{R}^m \mid \|\eta^{[1]}\|_2 > \Lambda\}$,
  then we know that $\Pr[\eta \in L] < \delta$.  Let $S\subset \cQ$ be
  a subset of queries. It follows that:
\begin{align*}
  \Pr\left[\eta \in \bigcup_{\hat q \in S}\cE(\hat q, S)\right] &= \int_{\mathbb{R}^m} P_G(\eta) \mathbbm{1}\left(\eta \in \bigcup_{\hat q \in S} \cE(\hat q, S)\right) d\eta \\
                                                                &= \int_{(\mathbb{R}^m \setminus B) \setminus L} P_G(\eta) \mathbbm{1}\left(\eta \in \bigcup_{\hat q \in S}\cE(\hat q, S) \right) d\eta + \int_{L} P_G(\eta) \mathbbm{1}\left(\eta \in \bigcup_{\hat q \in S} \cE(\hat q, S)\right) d\eta  \\
                                                                &
 \leq \int_{(\mathbb{R}^m \setminus B) \setminus L} P_G(\eta) \mathbbm{1}\left(\eta \in \bigcup_{\hat q \in S}\cE(\hat q, S) \right) d\eta + \delta \\
 &=
\sum_{\hat q \in S}  \int_{ \RR^m \setminus (B \cup L)}
  P_G(\eta) \mathbbm{1}\left(\eta \in \cE(\hat q, S)\right) d\eta  + \delta\\
 &\leq
\sum_{\hat q \in S}  \int_{ \RR^m \setminus (B \cup L)}
  P_G(\eta) \mathbbm{1}\left(f_{\hat q}(\eta) \in \cE(\hat q, S')\right) d\eta  + \delta && (\Cref{lem:inclusion})\\
 &\leq
\sum_{\hat q \in S}  \int_{ \RR^m \setminus (B \cup L)}
 \exp(\eps) P_G(f_{\hat q}(\eta)) \mathbbm{1}\left(f_{\hat q}(\eta) \in \cE(\hat q, S')\right) d\eta  + \delta &&(\Cref{lem:gaussian_density})\\
 & =
\sum_{\hat q \in S}  \int_{ \RR^m \setminus (f_{\hat q}(B) \cup f_{\hat q}(L))}
 \exp(\eps) P_G(\eta) \mathbbm{1}\left(\eta \in \cE(\hat q, S')\right) \left|\frac{\partial f_{\hat q}}{\partial \eta}\right| d\eta  + \delta \\
&\leq \exp(\eps)\sum_{\hat q\in S}\int_{\RR^m} P_G(\eta) \mathbbm{1}\left(\eta \in \cE(\hat q, S') \right) d\eta + \delta &&\left(\forall \hat q, \left|\frac{\partial f_{\hat q}}{\partial \eta}\right| = 1 \right)\\
&= \exp(\eps)   \Pr\left[\eta \in \bigcup_{\hat q \in S}\cE(\hat q, S')\right] + \delta
\end{align*}
This completes the proof.
\end{proof}

\section{Proofs and Details for Theorem~\ref{thm:prsma}}

\lemcoupling*
\begin{proof}
We can assume without loss of generality that $\cA_\cO$ and $\cA$ draw all their randomness up front in the form of a random seed $\eta$, and are then a deterministic function of the random seed and the input dataset. By definition, during the run of $\cA_\cO$, the algorithm generates a (possibly randomized, and possibly adaptively chosen) sequence of inputs to the optimization oracle $\cO$,  $\{{wd}_1, \ldots {wd}_m\}$, where each $wd_i$ is a weighted dataset. We denote the output of the $i^{th}$ optimization problem by $o_i$. After the $m^{th}$ optimization problem, $\cA_\cO$ outputs a deterministic outcome $a = h(o_1, o_2, \ldots o_m)$. Given access to a perfect optimization oracle ${\cO}^*$, $\cA_{{\cO}^*}$ is simply $\cA$ -- this is the definition of oracle equivalence. We construct a coupling between an algorithm $\cM$ and $\cA_\cO$, and then argue running $\cM$ is the same as running $\cA$:

\begin{algorithm}
\label{alg:coupling}
\textbf{Input}: A dataset $S$, random seed $\eta$, heuristic oracle $\cO$, perfect oracle ${\cO^*}$.\newline
\textbf{Output}: values $M(S, \eta), \cA(S, \eta)$
\begin{algorithmic}
\STATE Run $\cA_{\cO}(\eta, S)$ - generating the first optimization problem ${wd}_1$.
\FOR{$i = 1 \ldots m$}
	\STATE Compute $o_i = \cO(wd_i)$
	\IF{$o_i = \bot$}
		\STATE Output $\cA_\cO(S, w) = \bot$
		\STATE Set $o_i = {\cO}^*(wd_i)$
	\ENDIF
\STATE Generate $wd_{i+1}$ adaptively as a function of previous outputs $(o_i, o_{i-1}, \ldots o_1)$
\ENDFOR
\STATE Output $M(S, \eta) = h(o_1, \ldots o_m) = a$
\IF{$\cA_\cO(S, \eta) = \bot$ has not been output}
	\STATE Output $\cA_\cO(S, w) = a$
\ENDIF
\end{algorithmic}
\end{algorithm}

The procedure starts by generating a random seed $\eta$ and initializing a run of $\cA_\cO(\eta, S)$ - generating the first optimization problem ${wd}_1$. If the oracle $\cO$ fails on input $wd_1, \cA_\cO$ outputs $\bot$. In this case the next optimization input $wd_2$ is generated as a function of the output of the perfect oracle ${\cO}^*(wd_1)$. If it succeeds, we simply generate the next output as a function of $\cO(wd_1)$ (which is the same as ${\cO}^*(wd_1)$ by definition of certifiability). This process continues until we solve the $m^{th}$ optimization problem, and output $M(\eta, S), \cA_\cO(\eta,S)$ as described above. Now it is clear that if the oracle doesn't fail, we generate the same output $a$ for $\cA_\cO$ and $\cM$. No matter whether or not $\cO$ fails, $\cM$ has output that corresponds to perfectly solving the optimization problems generated with input $S, \eta$, and so it is equivalent to running $\cA$. Moreover, whenever the oracle does not fail, $\cA$ and $\cM$ have the same output, by construction. This completes the proof.
\end{proof}

\label{app:prsma}
\begin{definition}[\cite{negative}]
\label{scdef}
Let $(X_1, \ldots , X_n) \subset \{0,1\}^n$ be an ensemble of $n$ $\{0,1\}$-valued random variables. We say that $(X_1, \ldots X_n)$ satisfy the stochastic covering property, if for any $I \subset [n]$, $J = [n]$\textbackslash $I$, and $ a \geq a' \in \{0, 1\}^{|I|},$ where $\geq$ denotes coordinate-wise dominance, such that
$||a'-a||_1 = 1$, there is a coupling $v$ of the distributions $\mu, \mu'$ on $(X_k)_{k \in J}$ conditioned on $(X_k)_{k \in I} = a$ or $(X_k)_{k \in I} = a'$ respectively, such that $v(x,y) = 0$ unless $x \leq y$ and $||x-y||_1 \leq 1$.
\end{definition}

\begin{lemma}
\label{covering}
Given a set $|S| = n$, subsample $k \leq n$ elements without replacement. Let $X_i \in \{0, 1\}$ be $1$ if element $x_i \in S$ is subsampled, else $0$.
Then $(X_1, \ldots X_n)$ satisfy the stochastic covering property.
\end{lemma}
\begin{proof}
For $a \in \{0, 1\}^{|I|}$ let $|a|$ be the number of $1$'s in $a$. Then the distribution of $x = X_J|a$ corresponds to subsampling $k-|a|$ elements from $(x_k)_{k \in J}$, and the distribution of $y = X_J|a'$ corresponds to subsampling $k-|a|+1$ elements from $(x_k)_{k \in J}$ without replacement. To establish the stochastic covering property, we exhibit a coupling $v$ of $x, y$: \\
To generate $y$ subsample $k-|a| + 1$ elements from $(x_k)_{k \in J}$ without replacement. Let $x$ be the first $k-|a|$ such elements subsampled. Both $x, y$ constructed as such have the correct marginal distributions, and by construction $x \leq y, ||x-y||_1 = 1$ always.
 \end{proof}

 \begin{definition}
 \label{homog}
 $(X_1, \ldots , X_n) \subset \{0,1\}^n$ are $k$-homogenous if $\Pr[\sum_{i = 1}^{n}X_i = k] = 1$.
 \end{definition}

 \begin{theorem}[Theorem $3.1$ in \cite{negative}]
 Let $(X_1, \ldots X_n) \in \{0,1\}$ be $k$-homogenous random variables satisfying the stochastic covering property. Let $f: \{0, 1\}^n \to \mathbb{R}$ be an $\epsilon-$Lipschitz function, and let $\mu = \Ex{}{f(X_1, \ldots X_n)}$. Then for any $t > 0$:
 $$\prob{|f(X_1, \ldots X_n)-\mu| \geq t } \leq 2e^{\frac{-t^2}{8\epsilon^2k}}$$

 \end{theorem}

\mcdiarmid*
\begin{proof}
Fix any index $k$ of the partition. Let $\{X_i\}_{i = 1}^{n}$ be the indicator random random variables indicating that element $i$ in $S$, is included in $S_{k}$. Since $S_k$ is entirely determined by $\{X_i\}$, we can write $q_\Omega(S_k)$ as a function of $\{X_i\}$, e.g. $q_\Omega(X_1, \ldots X_n)$. Moreover, by definition of $\epsilon$-differential privacy, $q_\Omega$ is $\epsilon$-Lipschitz, i.e. for any $X_i, X_i'$: $$|q_\Omega(X_1, \ldots X_i, \ldots X_n)- q_\Omega(X_1,\ldots X_i', \ldots X_n)| \leq \epsilon$$
By Lemma~\ref{covering} proven in the Appendix, $(X_1, \ldots X_n)$ satisfy what is called the \textit{stochastic covering property}, a type of negative dependence. Since $|S_k| = n/K$, $(X_1, \ldots X_n)$ are $n/K$-homogenous. Thus by Theorem $1$ of \cite{negative}, with probability $1-\delta/K$:
$$
|q_\Omega(S_k)-\Ex{S_k \sim \cP_{split}^{S}}{q_\Omega(S_k)}| \leq \epsilon' \sqrt{8 \frac{n}{K}\log(2K/\delta)} = 1
$$
So by a union bound this holds for all $k = 1 \ldots K$ with probability at least $1-\delta$. Since for all $i,j, \Ex{S_i \sim \cP_{split}^{S}}{q_\Omega(S_i)} =
\Ex{S_j \sim \cP_{split}^{S}}{q_\Omega(S_j)}$, by the triangle inequality for all $i,j, |q_\Omega(S_i)- q_\Omega(S_j)| \leq 2$ with probability $1-\delta$, as desired.

\end{proof}

\lemexpand*
\begin{proof}
Conditioning on $\cL$ and using $\prob{\cL} \geq 1-2\delta$ we have:
 $$
  \frac{\Pr[\cps =a]}{\Pr[\cpsp \in \Omega]} \leq \frac{2\delta + \prob{\cps \in \Omega | \cL}}{\Pr[\cpsp \in \Omega]}
 $$
Expanding $\prob{\cps \in \Omega | \cL}$ by conditioning on $\cF, \cP(S)$ and using the law of total probability we have:
\begin{equation}
\label{maineq}
  \frac{\Pr[\cps =a| \cL]}{\Pr[\cpsp \in \Omega]}  = \frac{\displaystyle\sum_{\cP(S), \co \in \cF}\Pr[\cps \in \Omega | P(S), \co, \cL]\Pr[\co, P(S) | \cL]}{{\Pr[\cpsp \in \Omega]}} =
\end{equation}
Separating the summation in the numerator over $\cP(S)$ into $S_Q, S_Q^{c}$ we have:

$$
\displaystyle\sum_{\cP(S), \co \in \cF}\Pr[\cps \in \Omega | P(S), \co, \cL]\Pr[\co, P(S) | \cL] = \displaystyle\sum_{\cP(S) \in S_Q, \co \in \cF}\Pr[\cps \in \Omega | P(S), \co, \cL]\Pr[\co, P(S) | \cL] \; + $$
$$
\displaystyle\sum_{\cP(S) \in S_Q^{c}, \co \in \cF}\Pr[\cps \in \Omega | P(S), \co, \cL]\Pr[\co, P(S) | \cL]
$$
Rewriting the second term,
$$
\displaystyle\sum_{\cP(S) \in S_Q^{c}, \co \in \cF}\Pr[\cps \in \Omega | P(S), \co, \cL]\Pr[\co, P(S) | \cL] = $$
$$\displaystyle\sum_{\cP(S) \in S_Q^{c}}(\displaystyle\sum_{\co \in \cF}\Pr[\cps \in \Omega | P(S), \co, \cL]\Pr[\co, |P(S), \cL])\Pr[P(S)|\cL] =
$$
$$
 \displaystyle\sum_{\cP(S) \in S_Q^{c}}\Pr[\cps \in \Omega | P(S), \cL])\Pr[P(S)|\cL] \leq  \displaystyle\sum_{\cP(S) \in S_Q^{c}}\Pr[P(S)|\cL] = \Pr[Q^{c}| \cL]
$$
The first equality follows from the fact that $\Pr[\co, \cP(S)| \cL] = \Pr[\co|\cP(S), \cL]\Pr[\cP(S)|\cL]$, and the second equality follows from the law of total probability. Since $\Pr[Q^{c}] \leq \delta$, and $\Pr[\cL] \geq 1-2\delta$, $\Pr[Q^{c}| \cL] \leq \delta/(1-2\delta) \leq 2\delta$, for $\delta \leq 1/4$. Thus

$$
  \frac{\Pr[\cps =a| \cL]}{\Pr[\cpsp \in \Omega]}  \leq \frac{\displaystyle\sum_{\cP(S) \in S_Q, \co \in \cF}\Pr[\cps \in \Omega | P(S), \co, \cL]\Pr[\co, P(S) | \cL] + 2\delta}{\Pr[\cpsp \in \Omega]}
$$
Combining this bound, with the bound $\frac{\Pr[\cps =a]}{\Pr[\cpsp \in \Omega]} \leq \frac{2\delta + \prob{\cps \in \Omega | \cL}}{\Pr[\cpsp \in \Omega]},$ establishes the result.

\end{proof}

\lemweight*
\begin{proof}[Proof of Lemma~\ref{lemweight}]
By Equation~\ref{decomp}, we know:
$$
\Pr[\cps \in \Omega |P(S), \co, \cL] = \frac{1}{|\co|}\displaystyle\sum_{i \in I_{pass}^{\co}}\Pr[\cA_\cO(S_i) = a| \cA_\cO(S_i) \neq \bot]
$$
We also know that since we've conditioned on $\cL$ (and hence on $E$), for each $i \in I_{pass}^{\co}$ on the RHS of the above equation, $\Pr[\cA_\cO(S_i) = \bot] \leq \delta.$ By Lemma~\ref{lemcoupling}, we know that there exists a coupling between $\cA_\cO(S_i)$ and $\cA(S_i)$ such that $\Pr[\cA_\cO(S_i) = a| \cA_\cO(S_i) \neq \bot] = \Pr[\cA(S_i) = a| \cA_\cO(S_i) \neq \bot]$. \\

By the law of total probability:
$$\Pr[\cA(S_i) = a] =  \Pr[\cA(S_i) = a|\cA_\cO(S_i) \neq \bot]\Pr[\cA_\cO(S_i) \neq \bot] + \Pr[\cA(S_i) = a|\cA_\cO(S_i) = \bot]\Pr[\cA_\cO(S_i) = \bot]$$

Then $\Pr[\cA(S_i) = a|\cA_\cO(S_i) \neq \bot]\Pr[\cA_\cO(S_i) \neq \bot] \leq \Pr[\cA(S_i) = a]  \implies \Pr[\cA(S_i) = a|\cA_\cO(S_i) \neq \bot] \leq \frac{1}{1-\delta}\Pr[\cA(S_i) = a]$, and similarly $\Pr[\cA(S_i) = a|\cA_\cO(S_i) \neq \bot] \geq \Pr[\cA(S_i) =a]-\delta$. \\

Since we've shown that each of the conditional probabilities $\Pr[\cA_\cO(S_i) = a| \cA_\cO(S_i) \neq \bot]$ is close to $\Pr[\cA_\cO(S_i) = a]$, since we assume $\cP(S) \in S_Q$, we know they are close to each other. Using the inequalities above:
$$\Pr[\cA(S_j) = a|o_j\neq \bot] \leq \frac{1}{1-\delta}\Pr[\cA(S_j) = a] \leq \frac{e^2}{1-\delta}\Pr[\cA(S_i) =a] \leq \frac{e^2}{1-\delta}(\frac{\Pr[\cA(S_i) = a|o_i \neq \bot]}{1-\delta} + \delta),$$
where the middle inequality follows from the definition of $S_Q$. Summing both sides over $i: o_i \neq \bot, i \neq j$ and rearranging gives the desired result.
\end{proof}

\lemexpandtwo*
\begin{proof}
Denote  $\displaystyle\sum_{P(S) \in S_Q}\bigg(\displaystyle\sum_{\co \in \cF: o_1 \neq \bot}\big(\frac{1}{|\co|}\displaystyle\sum_{i \in I_{pass}^{\co}}\Pr[\cA_\cO(S_i) = a| \cA_\cO(S_i) \neq \bot]\big)\Pr[\co| P(S)],$ by $(\star)$.
By Lemma~\ref{lemweight}, $(\star) \leq$
$$
 \displaystyle\sum_{P(S) \in S_Q}\bigg(\displaystyle\sum_{\co \in \cF: o_1 \neq \bot}\big(\frac{\delta e^2}{(1-\delta)|\co|} + (1 + \frac{e^2}{(|\co|-1)(1-\delta)^2}) \cdot \frac{1}{|\co|}\displaystyle\sum_{i \in I_{pass}^{\co}, i \neq 1}\Pr[\cA_\cO(S_i) = a| \cA_\cO(S_i) \neq \bot, i^* = i]\big)\Pr[\co| P(S)]\bigg)\Pr[ P(S)]
$$
Pulling out the $\frac{\delta e^2}{(1-\delta)|\co|}$, we see that $\displaystyle\sum_{\co \in \cF: o_1 \neq \bot}\frac{\delta e^2}{(1-\delta)|\co|}\Pr[\co| P(S)] \leq \displaystyle\sum_{\co \in \cF}\frac{\delta e^2}{(1-\delta)}\cdot \epsilon \Pr[\co| P(S)] = \frac{\epsilon \delta e^2}{(1-\delta)}$.  Here we've used the fact that $|\co| > \frac{1}{\epsilon}$. Similarly, $\displaystyle\sum_{P(S) \in S_Q}\frac{\epsilon \delta e^2}{(1-\delta)}\Pr[P(S)] \leq \frac{\epsilon \delta e^2}{(1-\delta)}$. Applying to $(\star)$, we get:
\begin{equation}
\label{crunk}
(\star) \leq
\displaystyle\sum_{P(S) \in S_Q}\bigg(\displaystyle\sum_{\co \in \cF: o_1 \neq \bot}\big((1 + \frac{e^2}{(1-\delta)^2(\frac{1}{\epsilon}-1)}) \cdot \frac{1}{|\co|}\displaystyle\sum_{i \in I_{pass}^{\co}, i \neq 1}\Pr[\cA_\cO(S_i) = a| \cA_\cO(S_i) \neq \bot, i^* = i]\big)\Pr[\co| P(S)]\bigg)\Pr[ P(S)] + \frac{\epsilon\delta e^2}{1-\delta}
\end{equation}
Applying this upper bound on $(\star)$ gives:
$$
\displaystyle\sum_{\cP(S) \in S_Q, \co \in \cF}\Pr[\cps \in \Omega | P(S), \co, \cL]\Pr[\co, P(S) | \cL] \leq
$$
$$
\displaystyle\sum_{P(S) \in S_Q}\bigg(\big(\displaystyle\sum_{\co \in \cF: o_1 \neq \bot}\big((1 + \frac{e^2}{(1-\delta)^2(\frac{1}{\epsilon}-1)}) \cdot \frac{1}{|\co|}\displaystyle\sum_{i \in I_{pass}^{\co}, i \neq 1}\Pr[\cA_\cO(S_i) = a| \cA_\cO(S_i) \neq \bot, i^* = i]\big)\Pr[\co| P(S)]\big) \;+\;
$$
$$
\displaystyle\sum_{\co \in \cF: o_1 = \bot}\big(\frac{1}{|\co|}\displaystyle\sum_{i \in I_{pass}^{\co}}\Pr[\cA_\cO(S_i) = a| \cA_\cO(S_i) \neq \bot, i^* = i]\Pr[\co| P(S)\big)\bigg)\Pr[ P(S)] \leq
$$
$$
(1 + \frac{e^2}{(1-\delta)^2(\frac{1}{\epsilon}-1)}) \displaystyle\sum_{P(S) \in S_Q}\bigg(\displaystyle\sum_{\co \in \cF}\big(\frac{1}{|\co|}\displaystyle\sum_{i \in I_{pass}^{\co}, i \neq 1}\Pr[\cA_\cO(S_i) = a| \cA_\cO(S_i) \neq \bot, i^* = i]\big)\Pr[\co| P(S)]\bigg)\Pr[P(S)]
$$

\end{proof}
\noindent \textbf{End of proof of Theorem~\ref{thm:prsma}.}\\
\begin{proof}
First we  rewrite the numerator:
$$
\displaystyle\sum_{\co \in \cF}\big(\frac{1}{|\co|}\displaystyle\sum_{i \in I_{pass}^{\co}, i \neq 1}\Pr[\cA_\cO(S_i)\in \Omega| \cA_\cO(S_i) \neq \bot]\Pr[\co| P(S)]\big) =
$$
$$
\displaystyle\sum_{\co_{-1} \in \cF}\displaystyle\sum_{i \in I_{pass}^{\co}, i \neq 1}\bigg(\big(\frac{1}{|\co_{-1}|}\Pr[\cA_\cO(S_i)\in \Omega| \cA_\cO(S_i) \neq \bot]\Pr[o_1 = \bot, \co_{-1}| P(S)] \;+$$
$$
 \frac{1}{|\co_{-1} + 1|}\Pr[\cA_\cO(S_i)\in \Omega| \cA_\cO(S_i) \neq \bot]\Pr[o_1 =1, \co_{-1}| P(S)]\big)\bigg)  \leq
$$
$$
\displaystyle\sum_{\co_{-1} \in \cF}\displaystyle\sum_{i \in I_{pass}^{\co}, i \neq 1}\big(\frac{1}{|\co_{-1}|-1}\Pr[\cA_\cO(S_i)\in \Omega| \cA_\cO(S_i) \neq \bot]\Pr[\co_{-1}| P(S)]\big) ,
$$
where we've used the fact that $\Pr[o_1 = \bot, \co_{-1}| P(S)] + \Pr[o_1 \neq \bot, \co_{-1}| P(S)] = \Pr[\co_{-1}| P(S)]$.
Similarly, we can use this same trick to lower bound the denominator:
$$
\sum_{\co \in \cF}\displaystyle\sum_{i \in I_{pass}^{\co}}(\frac{1}{|\co|}\Pr[\cA_\cO(S_i')\in \Omega| \cA_\cO(S_i') \neq \bot]\Pr[\co| P(S')])\geq
$$

$$
\displaystyle\sum_{\co_{-1} \in \cF}\displaystyle\sum_{i \in I_{pass}^{\co}, i \neq 1}(\frac{1}{|\co_{-1}|}\Pr[\cA_\cO(S_i')\in \Omega| \cA_\cO(S_i') \neq \bot]\Pr[o_1 \neq \bot, \co_{-1}| P(S')] \;+
$$

$$
\frac{1}{|\co_{-1}|-1}\Pr[\cA_\cO(S_i')\in \Omega| \cA_\cO(S_i') \neq \bot]\Pr[o_1 = \bot, \co_{-1}| P(S')]) \geq
$$

$$
\displaystyle\sum_{\co_{-1} \in \cF}\displaystyle\sum_{i \in I_{pass}^{\co}, i \neq 1}\big(\frac{1}{|\co_{-1}|}\Pr[\cA_\cO(S_i')\in \Omega| \cA_\cO(S_i') \neq \bot]\Pr[\co_{-1}| P(S')]\big)
$$
Substituting these inequalities into (\ref{doodoo3}), we get that (\ref{doodoo3}) $\leq$
$$
\sup_{P(S) \sim P(S')} \frac{(1 + \frac{e^2}{(1-\delta)^2(\frac{1}{\epsilon}-1)})\displaystyle\sum_{\co_{-1} \in \cF}\displaystyle\sum_{i \in I_{pass}^{\co}, i \neq 1}\big(\frac{1}{|\co_{-1}|-1}\Pr[\cA_\cO(S_i)\in \Omega| \cA_\cO(S_i) \neq \bot]\Pr[\co_{-1}| P(S)]\big)}{\displaystyle\sum_{\co_{-1} \in \cF}\displaystyle\sum_{i \in I_{pass}^{\co}, i \neq 1}\big(\frac{1}{|\co_{-1}|}\Pr[\cA_\cO(S_i')\in \Omega| \cA_\cO(S_i') \neq \bot]\Pr[\co_{-1}| P(S')]\big) } + \frac{\frac{\epsilon\delta e^2}{1-\delta}}{\Pr[\cpsp \in \Omega]} \leq
$$

$$
(1 + \frac{e^2}{(1-\delta)^2(\frac{1}{\epsilon}-1)}) \cdot \sup_{P(S) \sim P(S'), \co_{-1}, i \neq 1} \frac{\frac{1}{|\co_{-1}|-1}\Pr[\cA_\cO(S_i)\in \Omega| \cA_\cO(S_i) \neq \bot]\Pr[\co_{-1}| P(S)]}{\frac{1}{|\co_{-1}|}\Pr[\cA_\cO(S_i')\in \Omega| \cA_\cO(S_i') \neq \bot]\Pr[\co_{-1}| P(S')]} + \frac{\frac{\epsilon\delta e^2}{1-\delta}}{\Pr[\cpsp \in \Omega]}
$$
Since $S_i = S_i'$ for all $i \neq 1$, this reduces to: $$(1 + \frac{e^2}{(1-\delta)^2(\frac{1}{\epsilon}-1)}) \cdot \sup_{\co_{-1}} \frac{|\co_{-1}|}{|\co_{-1}|-1} + \frac{\frac{\epsilon\delta e^2}{1-\delta}}{\Pr[\cpsp \in \Omega]} \leq (1 + \frac{e^2}{(1-\delta)^2(\frac{1}{\epsilon}-1)})\frac{1}{1-\epsilon} + \frac{\frac{\epsilon\delta e^2}{1-\delta}}{\Pr[\cpsp \in \Omega]}$$ since $|\co_{-1}| \geq \frac{1}{\epsilon}$ by definition.

Following the chain of inequalities back to their genesis, we finally obtain:
$$
\frac{\Pr[\cps \in \Omega | P(S), \co, \cL]\Pr[\co, P(S) | \cL]}{\Pr[\cpsp \in \Omega]} \leq (1 + \frac{e^2}{(1-\delta)^2(\frac{1}{\epsilon}-1)})\frac{1}{1-\epsilon} + \frac{\frac{\epsilon\delta e^2}{1-\delta}}{\Pr[\cpsp \in \Omega]},
$$
which substituting into Lemma~\ref{lemexpand} gives:
$$
   {\Pr[\cps \in \Omega]} \leq (1 + \frac{e^2}{(1-\delta)^2(\frac{1}{\epsilon}-1)})\frac{1}{1-\epsilon}{\Pr[\cpsp \in \Omega]} + 4\delta + \frac{\epsilon\delta e^2}{1-\delta}
$$

For $\epsilon, \delta  \leq 1/2$, $(1 + \frac{e^2}{(1-\delta)^2(\frac{1}{\epsilon}-1)})\frac{1}{1-\epsilon} \leq  e^{8e^2\epsilon + \epsilon + \epsilon^2}$,
which establishes that  \textbf{PRSMA} is $(8e^2\epsilon + \epsilon + \epsilon^2, 4\delta + \frac{\epsilon\delta e^2}{1-\delta})$ differentially private. Setting $\epsilon = \epsilon^*, \delta = \delta^*$ completes the proof.
\end{proof}

\section{Proofs from Section \ref{sec:barrier}}
\label{app:barrier}
\tokyo*
\begin{proof} 
This theorem is folklore, and this proof is adapted from the lecture notes of \cite{RS18}. 
We introduce some notation. First, write $\ell^t \in \mathbb{R}^{|\cQ|}$ to denote the ``loss vector'' faced by the algorithm at round $t$, with value $q_i(x^t)$ in coordinate $i$.  Write $\ell^{1:t}$ to denote the summed vector $\ell^{1:t}=\sum_{j=1}^t\ell^j$. Write $M:\mathbb{R}^d\rightarrow \mathbb{R}^d$ to denote the function such that $M(v)_{i^*} = 1$ where $i^* = \arg\min_{i} v_i$ and $M(v)_i = 0$ otherwise. In this notation, at each round $t$, ``Follow the Leader'' obtains loss $M(\ell^{1:t-1})\cdot \ell^t$ and ``Follow the Private Leader'' obtains loss $M(\ell^{1:t-1}+Z^t)\cdot \ell^t$. At the end of play, at time $T$, the best query $q$ in hindsight obtains cumulative loss $M(\ell^{1:T})\cdot \ell^{1:T}$.

The proof of this theorem will go through a thought experiment. Consider an imaginary algorithm called ``be the leader'', which at round $t$ plays according to $M(\ell^{1:t})$. We will first show that this imaginary algorithm obtains loss that is only lower than that of the best action in hindsight.
\begin{lemma}[\cite{KV05}]
\label{lem:BTL}
$$\sum_{i=1}^T M(\ell^{1:t})\cdot \ell^t \leq M(\ell^{1:T})\cdot \ell^{1:T}$$
\end{lemma}
\begin{proof}
This follows by a simple induction on $T$. For $T = 1$, it holds with equality. Now assume it holds for general $T$ -- we show it holds for the next time step:
$$\sum_{i=1}^{T+1} M(\ell^{1:t})\cdot \ell^t \leq M(\ell^{1:T})\cdot \ell^{1:T} + M(\ell^{1:{T+1}})\cdot \ell^{T+1} \leq  M(\ell^{1:T+1})\cdot \ell^{1:T} + M(\ell^{1:{T+1}})\cdot \ell^{T+1} =M(\ell^{1:T+1})\cdot \ell^{1:T+1}$$
\end{proof}

Recall that private pERM algorithms operate by sampling a perturbation vector  $Z^t \sim \cD_{\epsilon,\delta}$ at each round. 
Next, we show that ``be the private leader'', which at round $t$ plays according to  $M(\ell^{1:t} + Z^t)$, doesn't do much worse.
\begin{lemma}[\cite{KV05}]
\label{lem:BTPL}
For any set of loss vectors $\ell^1,\ldots,\ell^T$ and any set of perturbation vectors $Z^0 \equiv 0, Z^1,\ldots,Z^T$:
$$\sum_{t=1}^T M(\ell^{1:t} + Z^t)\cdot \ell^t \leq  M(\ell^{1:T})\cdot \ell^{1:T} + 2\sum_{t=1}^T ||Z^t - Z^{t-1}||_\infty$$
\end{lemma}
\begin{proof}
Define $\hat{\ell}^t = \ell^t + Z^t - Z^{t-1}$. Note that $\hat{\ell}^{1:t} = \ell^{1:t} + Z^t$, since the sum telescopes. Thus, we can apply Lemma \ref{lem:BTL} on the sequence $\hat{\ell}$ to conclude:
\begin{eqnarray*}
\sum_{t=1}^T M(\ell^{1:t}+Z^t)\cdot(\ell^t+Z^t-Z^{t-1}) &\leq& M(\ell^{1:T} + Z^T)\cdot (\ell^{1:T}+Z^T) \\
&\leq&  M(\ell^{1:T})\cdot (\ell^{1:T}+Z^T) \\
&=&  M(\ell^{1:T})\cdot \ell^{1:T} + \sum_{t=1}^T M(\ell^{1:T})\cdot (Z^t - Z^{t-1})
\end{eqnarray*}
Subtracting from both sides, we have:
$$\sum_{t=1}^T M(\ell^{1:t}+Z^t)\cdot \ell^t \leq  M(\ell^{1:T})\cdot \ell^{1:T} + \sum_{t=1}^T (M(\ell^{1:T}) - M(\ell^{1:t}+Z^t))\cdot (Z^t - Z^{t-1}) \leq  M(\ell^{1:T})\cdot \ell^{1:T} + \sum_{t=1}^T 2||Z^t - Z^{t-1}||_\infty$$
\end{proof}

We will use this lemma and a trick to compute the expected regret of ``be the private leader''. Since expectations distribute over sums, we have:
$$\E[\sum_{t=1}^T M(\ell^{1:t} + Z^t)\cdot \ell^t] = \sum_{t=1}^T \E[M(\ell^{1:t} + Z^t)\cdot \ell^t]$$
Hence, the expectation remains unchanged in the thought experiment under which the perturbation is not resampled at every step, and instead $Z^1 = \ldots = Z^t \sim \cD_{\epsilon,\delta}$. Applying Lemma \ref{lem:BTPL} to this version of be the private leader, we obtain:
$$\E[\sum_{t=1}^T M(\ell^{1:t} + Z^t)\cdot \ell^t] \leq M(\ell^{1:T})\cdot \ell^{1:T} + 2\E[||Z^1||_\infty]$$

Finally, we use the fact that the algorithm is $(\epsilon,\delta)$-differentially private, and the difference between ``follow the private leader'' and ``be the private leader'' amounts to running a differentially private algorithm on one of two datasets. For any $(\epsilon,\delta)$ differentially private algorithm $\cA:\cX^*\rightarrow R$, for any function $f:R\rightarrow [0,T]$, and for any pair of neighboring datasets $S, S'$ we have that:
$$\E[f(\cA(S))] \leq e^\epsilon \E[f(\cA(S'))] + \delta T.$$
  We can therefore conclude that for each $t$:
 $$\E[M(\ell^{1:t-1} + Z^t)\cdot \ell^t] \leq e^{\epsilon}\E[M(\ell^{1:t} + Z^t)\cdot \ell^t] + \delta t$$
Combining this bound with the regret bound we have proven for ``be the private leader'' yields:
$$\sum_{t=1}^T \E[M(\ell^{1:t-1} + Z^t)\cdot \ell^t] \leq e^{\epsilon} \E[\sum_{t=1}^T M(\ell^{1:t} + Z^t)\cdot \ell^t] + \delta T \leq e^{\epsilon}\left(M(\ell^{1:T})\cdot \ell^{1:T} +  2\E[||Z^1||_\infty]\right) +\delta T\leq$$
$$(1+2\epsilon)\left(M(\ell^{1:T})\cdot \ell^{1:T} + 2\E[||Z^1||_\infty]\right) + \delta T\leq M(\ell^{1:T})\cdot \ell^{1:T} +  (2+4\epsilon)\E[||Z^1||_\infty] + 2\epsilon T + \delta T$$
Dividing by $T$ yields the theorem. 
\end{proof}

\hongkong*
\begin{proof}
We start by quoting the main ingredient proven in \cite{HK16} that goes into their lower bound:
\begin{theorem}[\cite{HK16} Theorem 4]
\label{thm:hardgame}
For every $N = 2^d$, and for every randomized algorithm for the players in the game with access to a best-response oracle, there is an $N \times N$ game with payoffs taking values in $\{0, 1/4, 3/4, 1\}$ such that with probability $2/3$ the players have not converged to a $1/4$-approximate min-max equilibrium  until at least $\Omega(\sqrt{N}/\log^3(N))$ time.
\end{theorem}

We start by using the Yao min-max principle to reverse the order of quantifiers: Theorem \ref{thm:hardgame} also implies that there is a fixed \emph{distribution} over $N \times N$ games that is hard in expectation for \emph{every} algorithm. However, because we will be interested in the support size of this distribution, we go into a bit more detail in how we apply the min-max principle. 

Consider a ``meta game'' defined by an (infinite) matrix $M$, with rows $i$ indexed by the $L = 4^{N^2}$ $N \times N$ zero-sum games $G_i$ taking values in $\{0,1/4,3/4,1\}$, and columns indexed by algorithms $A_j$ instantiated with best-response oracles designed to play zero-sum games. $M(i,j)$ will encode the expected running time before algorithm $A_j$ when used to play game $G_i$ converges to a value that is within 1/4 of the equilibrium value of game $G_i$. Let the row player (the ``lower bound'' player) be the maximization player in the zero sum game defined by $M$, and let the column player (the ``algorithm player'') be the minimization player. As stated, the entries in $M$ can take unboundedly large values --- but observe that there is a simple modification to the game that allows us to upper bound the entries in $M$ by $O(N\log N)$. This is because there exists an algorithm $A$ (the multiplicative weights algorithm --- see e.g. \cite{MWsurvey}) that can be used to play any $N \times N$ game and converge to a $1/4$-approximate equilibrium after time at most $O(N  \log N)$ (running in time $O(N)$ per iteration for $O(\log N)$ iterations). It is also possible to check whether a pair of distributions form a $1/4$ approximate equilibrium with two calls to a best response oracle. Hence, we can take any algorithm $A_i$ and modify it so that it converges to a 1/4-approximate equilibrium after at most $O(N\log N)$ time. We simply halt the algorithm after $O(N\log N)$ time if it has not yet converged, and run the multiplicative weights algorithm $A$. Theorem \ref{thm:hardgame} implies that the value of this modified game $M$ is at least  $\Omega(\sqrt{N}/\log^3(N))$.

We now observe that the ``meta-game'' $M$ has an $O(1)$-approximate max-min strategy for the lower bound player that has support size at most $O(N^4 \log^2 N)$. To see this, consider the following constructive approach to computing an approximate equilibrium: simulate play of the game in rounds. Let the lower bound player sample a game $G^t$ at each round $t$ using the multiplicative weights distribution over her $L$ actions, and let the algorithm player best respond at each round to the lower bound player's distribution. By construction, the algorithm player has 0 regret, whereas the lower bound player has regret $O(\sqrt{\log L/T}\cdot N\log N)$ after $T$ rounds (since the entries of $M$ are bounded between $0$ and $O(N\log N)$) This corresponds to $O(1)$ regret after $T = O(\log L\cdot N^2\log^2 N) = O(N^4\log^2 N)$ many rounds. By Theorem \ref{fs}, the empirical distribution over these $O(N^4\log^2 N)$ many games $G^t$ forms an $O(1)$-approximate max-min strategy. Thus we have proven:

\begin{corollary}
For every $N = 2^d$, there is a fixed set $H \subseteq \{0,1/4,3/4,1\}^{N\times N}$ of $N \times N$ games, of size $|H| = O(N^4\log^2 N)$ such that for every randomized algorithm $A$ for players in a game with access to a best response oracle, there is a game $G \in H$ such that with probability $2/3$, the players have not converged to a $1/4$-approximate min-max equilibrium until at least $\Omega(\sqrt{N}/\log^3(N))$ time.
\end{corollary}

Now consider the $N \times O(N^5\log^2 N)$ matrix $R$ that results from stacking the matrices in $H$. Identify the $N$ rows with a query class $\cQ$ of size $N$, indexed by functions $q \in \cQ$. Identify the columns with a data universe $\cX$ of size $|\cX| =  O(N^5\log^2 N)$ indexed by $x \in \cX$, and define the queries such that $q(x) = R(q,x)$ for each $q \in \cQ, x \in \cX$. Observe that any no-regret algorithm with action set $\cQ$ that can obtain $o(1)$ regret against an adversary who is constrained to play loss vectors in $\cX$ can be used to compute an $o(1)$ approximate equilibrium strategy for any game in $H$ (together with a single call to a best-response oracle per round by his oppoinent, by Theorem \ref{fs}). Thus we can conclude that no algorithm can guarantee to get $o(1)$ regret over $\cQ$ until at least  $\Omega(\sqrt{N}/\log^3(N))$ time.
\end{proof}

\barrier*

\begin{proof}
For any such pERM algorithm $\cA$, Let $B = \mathbb{E}_{Z \sim \cD_{(\eps, \delta)}}[||Z||_\infty]$. We know from Theorem \ref{tokyo} that follow the private leader instantiated with $\cA$ obtains regret $o(1)$ whenever $T = \omega(B)$, for any $\epsilon + \delta = o(1)$. Since $\cA$ is oracle efficient (i.e. runs in time $\textrm{poly}(t, \log|\cQ|)$), the total running time needed to obtain diminishing regret is $\sum_{t=1}^T \textrm{poly}(t, \log|\cQ|) = \textrm{poly}(B,\log|\cQ|)$. We know from Theorem \ref{hongkong} that to guarantee diminishing regret over $\cQ$, the total running time must be at least $\Omega(\sqrt{|\cQ|}/\log^3(|\cQ|))$. Thus we must have that $B = \textrm{poly}(|\cQ|)$ --- i.e. $B = \Omega(|Q|^c)$ for some $c > 0$.
\end{proof}

\end{document}